\documentclass{article}
\usepackage[utf8]{inputenc}

\usepackage{amsthm,amsmath,amssymb, amsfonts}
\RequirePackage[numbers]{natbib}
\RequirePackage[colorlinks,citecolor=blue,urlcolor=blue]{hyperref}

\usepackage{dsfont}
\usepackage{mathrsfs}
\usepackage{verbatim}
\usepackage{graphicx}
\usepackage{authblk}
\usepackage{url}
\usepackage{blindtext}
\usepackage{algorithm}
\usepackage{algorithmic}
\usepackage{enumitem}   
\usepackage{xargs}
\usepackage{natbib}
\usepackage{mathrsfs}
\usepackage{xargs}
\usepackage{caption}
\usepackage{subcaption}
\usepackage{comment}
\usepackage[table]{xcolor}
\usepackage{color, colortbl}
\usepackage[top=1in,bottom=1in,left=1in,right=1in]{geometry}

\usepackage{color}
\definecolor{darkred}{RGB}{100,0,0}
\definecolor{darkgreen}{RGB}{0,100,0}
\definecolor{darkblue}{RGB}{0,0,150}

\usepackage{hyperref}
\hypersetup{colorlinks=true, linkcolor=darkred, citecolor=darkgreen, urlcolor=darkblue}
\usepackage{url}

\usepackage{todonotes}

\newtheorem{thm}{Theorem}
\newtheorem{prp}{Proposition}
\newtheorem{lem}{Lemma}
\newtheorem{cor}{Corollary}

\newtheorem{remark}{Remark}
\newtheorem{ass}{Assumption}

\def\beq{\begin{equation}}
\def\eeq{\end{equation}}
\def\beqn{\begin{eqnarray*}}
\def\eeqn{\end{eqnarray*}}
\def\bitem{\begin{itemize}}
\def\eitem{\end{itemize}}
\def\benum{\begin{enumerate}}
\def\eenum{\end{enumerate}}
\def\bmult{\begin{eqnarray*}}
\def\emult{\end{eqnarray*}}
\def\bcenter{\begin{center}}
\def\ecenter{\end{center}}

\newcommand{\thmref}[1]{Theorem~\ref{thm:#1}}

\newcommand{\figref}[1]{Figure~\ref{fig:#1}}

\DeclareMathOperator*{\argmin}{arg\, min}

\DeclareMathOperator{\rank}{rank}

\newcommand{\Sprm}{\bold S}
\newcommand{\Mprm}{\bold M}
\newcommand{\Zprm}{\bold Z}
\newcommand{\Sone}{\bold S}

\newcommand{\eqdef}{\overset{\bigtriangleup}{=}}
\newcommandx{\norm}[2][2=]{\| #1 \|_{#2}}
\newcommand{\pscal}[2]{\langle #1, #2  \rangle}

\def\cB{\mathcal{B}}

\def\cF{\mathcal{F}}

\def\cI{\mathcal{I}}

\def\cO{\mathcal{O}}
\def\cP{\mathcal{P}}

\def\cS{\mathcal{S}}

\def\bA{\boldsymbol{A}}

\def\bG{\boldsymbol{G}}
\def\bH{\boldsymbol{H}}

\def\bL{\boldsymbol{L}}
\def\bM{\boldsymbol{M}}
\def\bN{\boldsymbol{N}}

\def\bQ{\boldsymbol{Q}}

\def\bS{\boldsymbol{S}}

\def\bU{\boldsymbol{U}}
\def\bV{\boldsymbol{V}}
\def\bW{\boldsymbol{W}}
\def\bX{\boldsymbol{X}}
\def\bY{\boldsymbol{Y}}

\def\bb{\mathbf{b}}

\def\bv{\mathbf{v}}

\newcommand{\bepsilon}{{\boldsymbol\epsilon}}
\newcommand{\bxi}{{\boldsymbol\xi}}

\newcommand{\bSigma}{{\boldsymbol\Sigma}}
\newcommand{\bOmega}{{\boldsymbol\Omega}}
\newcommand{\bPi}{{\boldsymbol\Pi}}
\newcommand{\bGamma}{{\boldsymbol\Gamma}}

\newcommand{\bLambda}{{\boldsymbol\Lambda}}

\def\bb\bU{\mathbb{\bU}}

\def\\bUnif{\text{\bUnif}}

\def\iff{\ \Leftrightarrow \ }

\usepackage{ulem}
\begin{document}

\providecommand{\keywords}[1]
{
  \small	
  \textbf{\textit{Keywords---}} #1
}

\title{Outliers Detection in Networks with Missing Links}
\author[1]{Solenne Gaucher \thanks{solenne.gaucher@ensae.fr}}
\author[2,3]{Olga Klopp \thanks{kloppolga@math.cnrs.fr}}
\author[4]{Geneviève Robin \thanks{genevieve.robin@inria.fr}}
\affil[1]{Laboratoire de Math\'ematiques d’Orsay, Univ. Paris-Sud, CNRS, Universit\'e Paris-Saclay.}
\affil[2]{ESSEC Business School}
\affil[3]{CREST, ENSAE}
\affil[4]{LaMME, CNRS, Université d'Évry Val d'Essonne}
\maketitle
\begin{abstract}

Outliers arise in networks due to different reasons such as fraudulent behaviour of malicious users or default in measurement instruments and can significantly impair network analyses. In addition, real-life networks are likely to be incompletely observed, with missing links due to individual non-response or machine failures. Identifying outliers in the presence of missing links is therefore a crucial problem in network analysis. In this work, we introduce a new algorithm to detect outliers in a network that simultaneously predicts the missing links. The proposed method is statistically sound: we prove that, under fairly general assumptions, our algorithm exactly detects the outliers, and achieves the best known error for the prediction of missing links with polynomial computation cost. It is also computationally efficient: we prove sub-linear convergence of our algorithm.
We provide a simulation study which demonstrates the good behaviour of the algorithm in terms of outliers detection and prediction of the missing links. We also illustrate the method with an application in epidemiology, and with the analysis of a political Twitter network. The method
is freely available as an R package on the Comprehensive R Archive Network.
\end{abstract}

\section{Introduction}
Networks are powerful tools to analyse complex systems: agents are represented as nodes, and pairwise interactions between agents are recorded as edges between these nodes. Examples of fields of applications include biology, where networks may be used to describe protein-protein interactions; ecology, where they may represent food webs \cite{Foodweb} or spatial distributions in crop diversity networks \cite{CropNet}; ethnology, where networks summarise relationships or trades between individuals or communities \cite{FoodSharing, NetworkMarginality}; sociology, where the recent development of online social networks offers unprecedented possibilities while fostering new challenges \cite{SocNet}. 
Those real-life networks are often modelled as realisations of random graphs or, equivalently, as noisy versions of more structured networks. In this setting, recovering the ``noiseless" version of the graph, i.e. estimating the underlying probabilities of interactions between agents, is a key problem that has recently gained considerable attention (see, e.g.,  \cite{kloppgraphon, GaoBiclustering, gaucher, spectralGraphon}). Most of the proposed methods are based on models describing the connectivity of the majority of nodes. However, in many examples those models fail to describe networks containing a small number of outliers nodes with abnormal behaviour. Following Hawkins \cite{Hawkins}, we define an outlier as ``an observation that deviates so much from other observations as to arouse suspicion that it was generated by a different mechanism".

 Detecting nodes with anomalous behaviour is  an important problem in applications. For example, in social networks, malicious nodes corresponding to fake accounts created to spread fake news, to distribute malware, or to spam other users may be hidden among the regular nodes \cite{Malicious}. These outliers often exhibit connection patterns that differ from that of normal nodes: the authors of \cite{SpamAttack} show, for example, that spam attackers are often connected with numerous nodes in a random fashion, thus forming characteristic hubs. By contrast, the connections between regular nodes are more sparse and more structured: they may, for example, exhibit community structures. Identifying those malicious nodes is crucial to protect users from the threat they represent. In the context of graphs obtained from survey data, anomalous behaviour may indicate that participants are providing false answers to distort the public opinion on a subject \cite{opFraud,bipartite}. In other cases, defaults of measurement instruments or fraudulent behaviours can lead to abnormal connectivity patterns. Finally, in contact networks, individuals with anomalous connection patterns may play an important role in the propagation of diseases, and their identification finds important applications in epidemiology \citep{Epidemiology}.
These examples illustrate how identifying outlier nodes can provide us with hindsight on the network. Moreover, detecting these nodes allows us to control the bias induced by their anomalous behaviour in the network analysis. For example, it has been shown that  the presence of hubs in graphs exhibiting community structure can hinder the estimation of these communities \cite{TCai, DegreeCorrected}.

In addition, many real-life networks are polluted by missing data \cite{MissingGuimer, MissingHandcock}. Indeed, complete exploration of all pairwise interactions between agents can be expensive, time consuming, and requires significant effort. In social sciences, graphs constructed from survey data are likely to be incomplete, due to non-response or drop-out of participants.  Online social network data are often obtained through crawling of users profile; however the gigantic size of these networks may drive analysts to stop prematurely this crawling, and work with a sub-sample of the network \citep{crawling}. Protein-protein interactions networks provide a blatant example of incompleteness, as the existence of each interaction must be tested experimentally, and most of these interactions have yet to be tested \cite{MissingLinkProtein}.  When dealing with a partially observed network, being able to predict the probability of existence of non-observed edges is of particular interest and finds numerous applications, for example in biology \cite{link_pred_bio}, recommender systems \cite{Recommend} and ecology \cite{fu2019link}.\newline

In this paper, we propose a new algorithm that detects the outliers in networks. In addition, this method robustly estimates the probabilities of connection of the nodes in the network, which allows to predict the missing links. The present paper is mostly related to two lines of work in network analysis: anomaly detection in networks and estimation in networks with missing values. 
Anomaly detection in networks has indeed been studied under several sets of assumptions on the behaviour of outlier nodes; we refer the interested reader to \cite{Malicious} for a review of these technics. For instance, many algorithms based on trust propagation rely on the assumption that outlier nodes are not well connected with normal nodes \cite{SybilGuard, SybilRadar}. Other algorithms, based on community structure, assume that outliers  \cite{SybilCom,SybilCom2} fail to be well connected to communities of normal nodes. However, it has been shown in \cite{Sybilsconnected} that these assumptions do not hold in many situations. In addition, most of these technics focus on outliers detection, and do not study the estimation of underlying structure. Meanwhile, robust estimation of the graph structure in the presence of outlier nodes has been less studied. In \cite{TCai}, the authors aim to recover community structures when the majority of the nodes follow an assortative stochastic block model in the presence of arbitrary outlier nodes. However, their algorithm does not allow to detect these outlier nodes. Note that our problem is different, as we would like to estimate connection probabilities between nodes rather than recover community structures, and our assumptions on the random graph are more general.

On the other hand, estimation in networks with missing observations, and its application to link prediction has known a quite recent development. In \cite{GaoBiclustering}, the authors study the least squares estimator for the stochastic block model assuming observations are missing uniformly at random, and show that the procedure is minimax optimal. In \cite{gaucher}, the authors show that the maximum likelihood estimator is minimax optimal in the same setting, while being adaptive to more general sampling schemes. These two estimators are too costly to compute to be used in practice (computationally efficient approximations exist for the maximum likelihood). In \cite{LinkPredLevina}, the authors consider the setting where non-existing edges can be erroneously recorded as observed (or existing edges recorded as not observed), both errors occurring at a fixed rate. More recently, \cite{MissingSubNets} and \cite{LevinaEgo} proposed algorithms to estimate the edge probabilities under different missing observations schemes, and \cite{li2020community} proposed a method for consistent community detection also under several missing values scenarios. Both papers present convincing numerical experiments, but lack theoretical guarantees. 

Finally, our work is also closely related to recent developments in the field of robust matrix completion. Indeed, in our general model presented in Section \ref{subsec:model}, we assume that the matrix of connection probabilities can be decomposed as the sum of a low rank component (connectivity pattern of inliers), and that of a column-wise sparse component (non-zero columns corresponding to outliers). Our problem is to estimate the low-rank matrix in order to reconstruct the connectivity of inliers, and to \textit{reconstruct the support of the column-wise sparse component}, in order to detect outliers. The problem of estimating the low-rank matrix is related to that of robust  matrix completion, in which one aims at estimating a low-rank matrix from incomplete and corrupted observations of its entries; see, for example, \cite{Candes_robust, Chandra_Parillo_sujay_Willsky, HsuKakadeZhang, Xu_Caramanis_Sanghavi, AgarwalNegahbanWainwright,  chen_jalali, li,KloppRMC}.  More recently, the problem of robust matrix completion with binary observations has been studied in \cite{MainEffect, pmlr-v89-shen19a}. However, to the best of our knowledge, existing work on sparse plus low-rank matrix decomposition in the noisy case do not provide guarantees concerning support recovery of the sparse component. In this paper, we provide such results and prove that our algorithm exactly recovers the support of the sparse matrix. Another shortcoming of existing results on binary robust matrix completion (e.g. \cite{MainEffect, pmlr-v89-shen19a}) is that applying them to the estimation of connection probabilities in networks yields sub-optimal error rates. Indeed, in our framework, the signal to noise ratio is critically low, as the variances of the variables are of the same order as their expectations. The main difficulty arising in our case, and that we tackle in the present paper, is therefore to obtain the optimal dependence on the sparsity of the network. \newline

In the present work, we present a new algorithm to detect the outliers and to estimate the connection probabilities of the remaining nodes, which is robust to missing observations. For this algorithm, we provide both statistical and computational guarantees. In particular, in Theorem \ref{thm:outliers_detection}, we prove that under fairly general assumptions our algorithm achieves exact detection of the outliers. In Theorem \ref{thm:Borne_norme_2}, we also prove an upper bound on the estimation error of connection probabilities between inliers. Importantly, the estimation error of our method matches the best known error for tractable algorithms \cite{spectralGraphon}. We also analyse the algorithm convergence complexity in Theorem \ref{thm:cvg}, and show sub-linear convergence. In Section \ref{section:simul}, we provide a simulation study with comparisons to state-of-the-art technics, indicating that the proposed method has good empirical properties in terms of outliers detection and link prediction. Finally, we illustrate the performance of our method with two applications in epidemiology and social network analysis.
\subsection{Example: ``Les Misérables" characters network}

Before introducing our general model, let us start with an example.``Les Misérables" characters network encodes interactions between characters of Victor Hugo's novel; the network was created by Donald Knuth, as part of the Stanford Graph Base \citep{Knuth:1993:SGP:164984}. It contains 77 nodes corresponding to characters of the novel, and 254 edges connecting two characters whenever they appear in the same chapter. The book itself spans around two decades of nineteenth century France and numerous characters. It is structured in five volumes, each one focused on a specific period and  featuring handful of characters. 
One expects to observe communities in this network, corresponding roughly to the plots narrated in each volume: such structures are well captured by the classical Stochastic Block Model (SBM). 
In the SBM (see, e.g., \cite{HOLLAND1983109}), nodes are  classified into $k$ communities (for example corresponding
to volumes of the book).  Denote by $\mathcal{G} = (\mathcal{V}, \mathcal{E})$ the graph, where $\mathcal{V}$ is the set of nodes, and $\mathcal{E}$ the set of edges.  For any $i\in \mathcal{V}$, denote by $c(i)$ its community assignment. Then, the probability that an edge connects two nodes only depends on their community assignments:
\begin{equation}
\label{eq:SBM}
    \mathbb{P}((i,j)\in \mathcal{E}) = \bQ_{c(i)c(j)}.
\end{equation}
In \eqref{eq:SBM}, $\bQ$ denotes a $k\times k$ symmetric matrix of connection probabilities between communities. Usually, in the Stochastic Block Model, the community assignment is unknown and learned from data.

However, some of the characters behave differently, as their stories follow the entire novel. For instance, the main character, Jean Valjean, acts as a \textit{hub} with 36 connections, well above the second most connected character Gavroche, with a degree of 22. Other characters, for instance, Cosette, do not necessarily have a large degree but are connected to characters across all the volumes, and thus also stand out from the communities structure. Nodes such as Cosette correspond to outliers with \textit{mixed membership} profile. 
In \figref{miserables_graph_sbm}, we display the communities assignment resulting from the classical SBM. Note that the node corresponding to Jean Valjean (large yellow node), is alone in its community. In addition, one of the clusters (in red) contains most of the main characters of the novel (Les Thénardier, Éponine, Javert). \newline
\begin{figure}[!tbp]
\begin{center}
  \begin{subfigure}[b]{0.4\textwidth}
    \includegraphics[width=\textwidth]{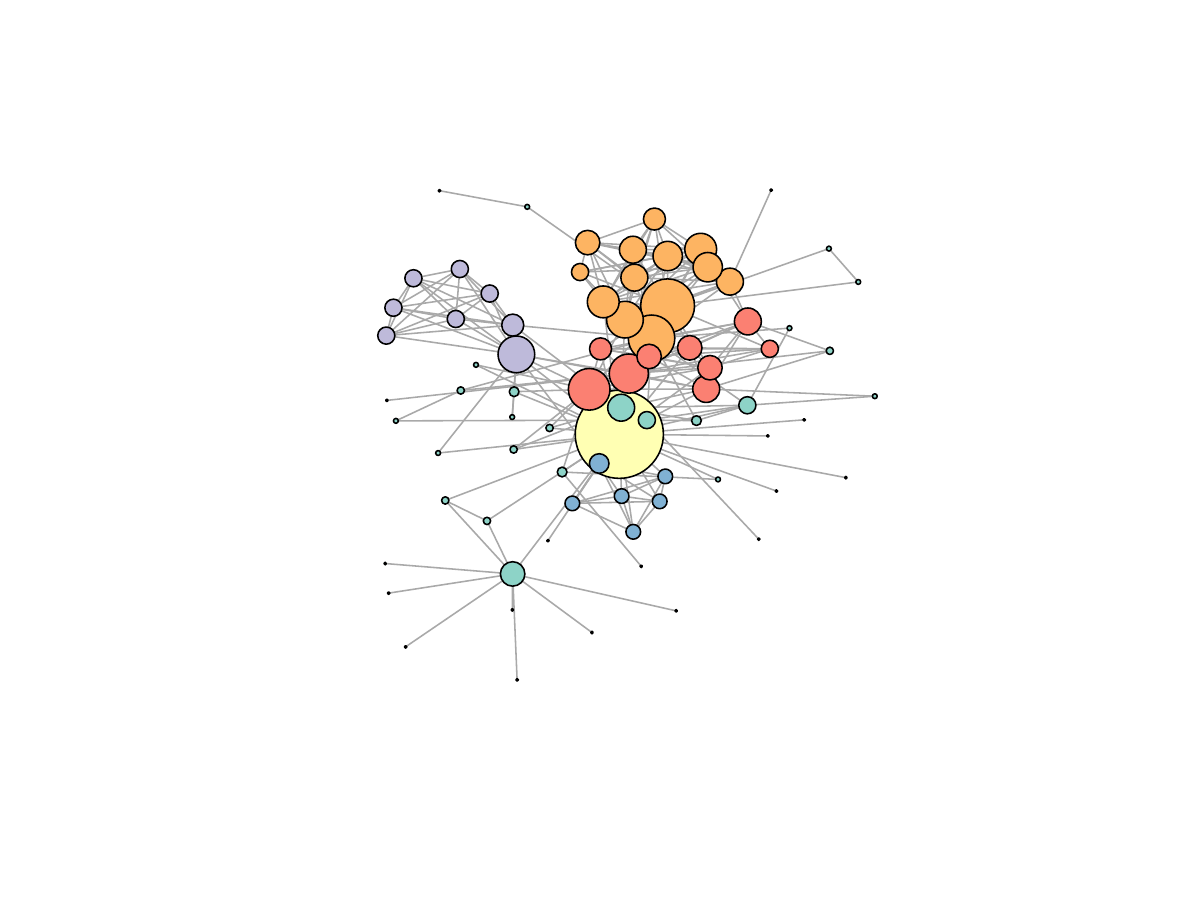}
    \caption{SBM model with 6 communities (the number of communities is chosen to minimise the Integrated Completed Likelihood criterion).}
    \label{fig:miserables_graph_sbm}
  \end{subfigure}
  \hspace{1cm}
  \begin{subfigure}[b]{0.4\textwidth}
    \includegraphics[width=\textwidth]{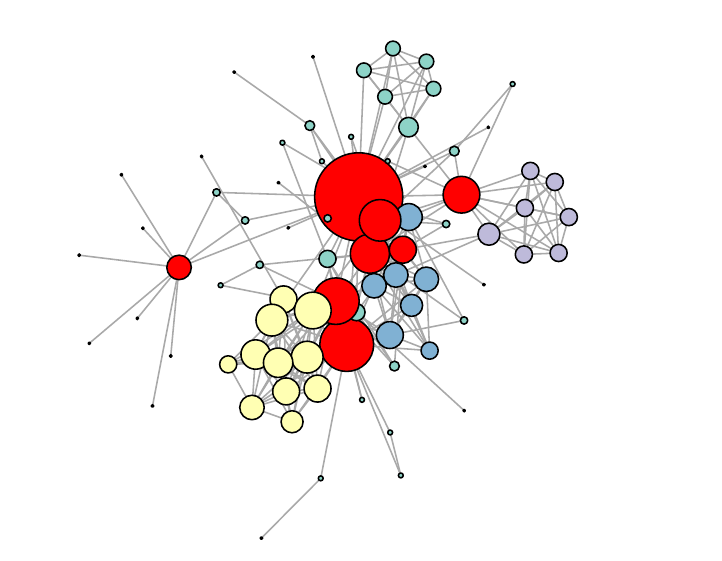}
    \caption{Proposed Stochastic Block Model with outliers. The detected outliers are coloured in red, and classification is performed on the rest of the nodes.
    }
    \label{fig:miserables_graph_robust_sbm}
  \end{subfigure}
  \caption{Les Misérables characters network. The nodes are represented with size proportional to their degree, and coloured according to their community assignment. On the left in \figref{miserables_graph_sbm}, classification is performed according to the classical SBM model. On the right in \figref{miserables_graph_robust_sbm}, the detected outliers are indicated in red, and classification is performed on the rest of the nodes (inliers).}
  \label{fig:miserables_graph}
  \end{center}
\end{figure}
To model simultaneously the community structure and the outlier profiles, we propose to decompose $\mathcal{V}$ into two sets of nodes: the inliers $\mathcal{I}$ following the classical Stochastic Block Model structure and the outliers $\mathcal{O}$ for which we do not make any assumption on their connection pattern. As a result, the probability of connection between inliers is given, for any $(i,j)\in \mathcal{I}^2$, by
$$\mathbb{P}((i,j)\in \mathcal{E}) = \bL_{ij}^*,$$
where $\bL^*$ is a symmetric matrix with entries in $[0,1]$ corresponding to a classical SBM. On the other hand, for any outlier $i\in \mathcal{O}$ and for any node $j\in \mathcal{V}$ we set $$\mathbb{P}((i,j)\in \mathcal{E}) = \left(\bS^* + {\bS^*}^{\top}\right)_{ij},$$ with $\bS^*$ an arbitrary matrix in $[0,1]^{n\times n}$. Our only assumption regarding the outliers is that their number is small compared to the size of the network, i.e., the matrix $\bS^*$ is column-wise sparse. Note that the inlier and outlier sets are unknown a priori, and learned from data. In \figref{miserables_graph_robust_sbm}, we display the communities assignment resulting from our model. The outlier nodes -- which are selected automatically by our procedure -- are indicated in red, and coincide with central characters of the novel. They correspond either to hubs (Jean Valjean, Myriel) or to nodes with mixed memberships (Cosette, Javert, Marius).

\subsection{Organisation of the paper}
The rest of the paper is organised as follows. First, in Section \ref{subsec:notation}, we summarise notation used throughout this paper and, in Section \ref{subsec:model}, we introduce our model. Then, in Section \ref{subsec:algo}, we present a computationally efficient algorithm for detecting outliers and estimating the connection probabilities between inliers. We also provide theoretical guarantees on the speed of convergence of this algorithm. In Section \ref{subsec:error_bounds}, we provide bounds on the error of the outliers detection and on the error of the estimation of the connection probabilities between inliers. In Section \ref{section:simul}, we present numerical experiments which demonstrate the good empirical behaviour of our method, both in terms of outliers detection and in terms of prediction of the missing links. The method is implemented in the R \citep{R} package \href{https://cran.r-project.org/web/packages/gsbm/index.html}{\texttt{GSBM}} available on the Comprehensive R Archive Network. The proofs are relegated to the Appendix \ref{sec:proofs}.
\subsection{Notations}
\label{subsec:notation}
The notation used in the paper is gathered in the following paragraph :
\begin{itemize}
\item We use bold notations for matrices and vectors: for any matrix $\bM$, we denote by $\bM_{ij}$ its entry on row $i$ and column $j$. The vector corresponding to its $i$-th row is denoted by $\bM_{i,\cdot}$, and the vector corresponding to its $j$-th column is denoted by $\bM_{\cdot,j}$. The notation $\mathbf{0}$ denotes either a matrix or a vector with entries all equal to $0$.

\item We write $\odot$ to denote the entry-wise product for matrices or vectors. For any vector $\bv \in \mathbb{R}^n$, we denote by $\left \Vert \bv \right \Vert_2$ its Euclidean norm. For any two matrices $\bM, \bN \in \mathbb{R}^{n \times n}$, $\left \langle\bM \big \vert \bN \right \rangle \triangleq \underset{ij}{\sum}\bM_{ij}\bN_{ij}$ is the Frobenius scalar product between $\bM$ and $\bN$. For any matrix $\bM \in \mathbb{R}^{n \times n}$, $\left \Vert \bM \right \Vert_F$ is its Frobenius norm, $\left \Vert \bM \right \Vert_*$ is its nuclear norm (the sum of its singular values), $\left \Vert \bM \right \Vert_{\text{op}}$ is its operator norm (its largest singular value), and $\left \Vert \bM \right \Vert_{\infty} \triangleq \underset{ij}{\max} \left \vert \bM_{ij} \right \vert$ is the largest absolute value of its entries. Its column-wise 2,1-norm is denoted by $\left \Vert \bM \right \Vert_{2,1} \triangleq \underset{j}{\sum} \sqrt{\underset{i}{\sum} \bM_{ij}^2}$, and the column-wise 2,$\infty$-norm is denoted by $\left \Vert \bM \right \Vert_{2,\infty} \triangleq \underset{j}{\max} \sqrt{\underset{i}{\sum} \bM_{ij}^2}$. The weighed $L_2$-norm with respect to the sampling probability $\bPi$ is written $\left \Vert \bM \right \Vert_{L_2(\bPi)} $. Finally, for any matrix $\bM$ and any vector $\bv$, we denote respectively by $\left(\bM\right)_+$ and $\left(\bv\right)_+$ the matrix and vector obtained by considering the positive part of their entries.

\item  For a matrix $\bM \in \mathbb{R}^{n\times n}$, we denote by $\cP_{\bM}$ the projection defined as follows: for any matrix $\bA \in \mathbb{R}^{n\times n}$, $\cP_{\bM}^{\perp}(\bA) =\bA - {\cP^{\perp}_{\bM}}(\bA)$, where ${\cP^{\perp}_{\bM}}(\bA) = P^{\perp}_{U(\bM)} \bA P^{\perp}_{V(\bM)}$, and $P^{\perp}_{U(\bM)}$ and $P^{\perp}_{V(\bM)}$ denote respectively the projection on the spaces orthogonal to the spaces spanned by the right and left singular vectors of $\bM$.

\item  We denote by $[n]$ the set of integers from $1$ to $n$, by $\mathcal{I}$ the set of inlier nodes, and by $\mathcal{O}$ the set of outlier nodes. For a set of indices $\mathcal{S}$ and a matrix $\bM \in \mathcal{R}^{n\times n}$, we write $\bM_{\vert \mathcal{S}} \triangleq \mathds{1}_{\cS} \odot \bM$ where $\mathds{1}_{\cS}$ is the indicator matrix of the set $\cS$. For any set $\mathcal{S}$, we denote by $\vert \mathcal{S} \vert$ its cardinality.
\end{itemize}

\section{General model}
\label{subsec:model}

We consider an undirected, unweighted graph with $n$ nodes indexed from $1$ to $n$. To encode the set of edges, we use the \textit{adjacency matrix} of the graph, which we denote by $\bA$. This matrix is defined as follows: set $\bA_{ij} = 1$ if there exists an edge linking node $i$ and node $j$, and $\bA_{ij} = 0$ otherwise. Note that since the graph is undirected we have $\bA_{ij} = \bA_{ji}$. We assume there are no loops in the graph: no edge can connect a node to itself, and thus $\bA_{ii} = 0$.
The nodes can be divided into inliers and outliers. Inliers correspond to the majority of the nodes, and their connection probabilities are given by a low-rank model. Outliers correspond to a small number of nodes with anomalous connections, and connect arbitrarily to inlier and outlier nodes.

\paragraph{Probability of connection between inliers}
$\text{For any pair of inliers } (i,j) \in \cI^2\text{, }i<j$ we assume that $\bA_{ij} \overset{ind.}{\sim} \text{Bernoulli}(\bL^*_{ij}),$ where $\bL^*$ is a $n \times n$ symmetric matrix with entries in $[0,1]$.  For inliers, we  consider a more general model than the classical Stochastic Block Model assuming that $\bL^*$ is low-rank. This assumption is enough to model some interesting properties of the SBM, such as positive and negative homophily, and stochastic equivalence. Indeed, when $\rank(\bL^*) = k$, there exist a matrix $\bU \in \mathbb{R}^{n\times k}$ and a diagonal matrix $\bLambda \in \mathbb{R}^{k \times k}$ such that $\bL^* = \bU \bLambda \bU^{\top}$. The model can then be interpreted as follows: each row $\bU_{i, \cdot}$ corresponds to a vector of $k$ latent attributes describing the node $i$. If $\bLambda_{aa} >0$, two nodes sharing attributes of the same sign along the $a$-th coordinate will have a tendency to be more connected (everything else being equal), modelling positive homophily along this coordinate. If $\bLambda_{aa} <0$, they will tend to be less connected, modelling negative homophily. Note that two nodes with similar characteristics in the latent space will have similar stochastic behaviour (i.e. their probabilities of connection to other nodes will be given by similar vectors of probabilities). On the other hand, assuming that $\bL^*$ is low-rank closely relates to the \textit{latent eigenmodel}, described, for example, in \cite{Hoff}. In this model, the probability of connection of nodes $i$ and $j$ is given by $f(\bL^*_{ij})$, where $\bL^*$ is of rank $k$ and $f$ is a link function. Note that our algorithm can be extended to the latent eigenmodel by replacing $\bL$ by $f(\bL)$ in the objective function \eqref{eq:objective}.

Finally, most graphs encountered by practitioners are \textit{sparse}, with a small average degree compared to the number of nodes. To account for the sparsity, we assume that the entries of $\bL^*$ are bounded by $\rho_n$ where $\rho_n$ is a sequence of sparsity inducing parameters such that $\rho_n \rightarrow 0$. In particular, we have that the average degree of the graph grows as $\rho_n n$. In the rest of the paper we assume  that $\rho_n \leq \frac{1}{2}$. This assumption is only intended to clarify the exposition of our results, and can be easily removed.

\paragraph{Probability of connection of outlier nodes}
In our model we have no  assumptions on the connectivity of outliers. In particular, we do not assume a block constant or a low rank structure. We set $\bL^*_{ij} = 0$ for any pair of nodes $(i,j)$ such that either $i\in \cO$ or $j\in\cO$, and we use  matrix $\bS^*$ to describe the outliers. For any inlier $j\in \cI$, the $j$-th column of $\bS^*$ is null. Therefore, the matrix $\bS^*$ has at most $s=|\mathcal{O}|$ non-zero columns,  where the number of outliers $s$ is small compared to the number of nodes $n$. For any outlier $j\in\cO$, the $j$-th column of $\bS^*$ describes the connectivity of $j$: for any $j\in \cO\text{ and }i\in\cI,\ \bA_{ij} \sim$ Bernoulli$(\bS^*_{ij})$ and for any $(i,j)\in\cO\times \cO, \ \bA_{ij} \sim$Bernoulli$(\bS^*_{ij}+ \bS^*_{ji})$.
We set $\bS_{ii}^* =0$ for any $i \in [n]$. With these notations, we have that 
\begin{equation}\label{eq:decA}
  \mathbb{E}\left[\bA\right] = \bL^{*} - \text{diag}(\bL^*) + \bS^* + \left(\bS^*\right)^{\top}.
\end{equation}
In this model, the outliers may account for different types of behaviour of the nodes, such as hubs or mixed membership profiles. In practice, while most nodes may be assigned to a community and share a similar stochastic behaviour with members of their community, a fraction of the nodes may belong to two or more communities. Our model allows for such a behaviour by considering the nodes with mixed membership as outliers. In these cases, being able to detect nodes with singular behaviour provides valuable information on the network.  Note that this setting includes as particular case the  Generalised Stochastic Block Model, introduced in \cite{TCai}. In this model, the $n$ nodes consist of $n-s$ inliers obeying the Stochastic Block Model (SBM), and $s$ outliers, which are connected with other nodes in an arbitrary way.

\paragraph{Missing data pattern}  We say that we sample the pair $(i,j)$ if we observe the presence or absence of the corresponding edge. We denote by $\bOmega$ the sampling matrix such that $\bOmega_{ij} = 1$ if the pair  $(i,j)$ is sampled, $\bOmega_{ij} = 0$ otherwise. The graph is unoriented and the sampling matrix is therefore symmetric; moreover we set $\text{diag}(\bOmega) = \bold{0}$ since an observation of a entry on the diagonal of $\bA$ does not carry any information. 
We  assume that the entries $\left \{ \bOmega_{ij}\right\}_{i<j}$ are independent random variables and that $\bOmega$ and $\bA$ are independent. We denote by $\bPi \in \mathbb{R}^{n\times n}$ the expectation of the random matrix $\bOmega$. Then, for any pair $(i,j)$, $\bOmega_{ij}\sim$ Bernoulli$(\bPi_{ij})$. For any matrix $\bM \in \mathbb{R}^{n \times n}$, we define
\begin{equation*}
\left \Vert \bM \right \Vert_{L_2(\bPi)}^2 \triangleq \mathbb{E}\left[ \left \Vert \Omega \odot \bM \right \Vert_F^2\right].   
\end{equation*}
This fairly general sampling scheme covers some of the settings encountered by practitioners. In particular, it covers the case of random dyad sampling (described, e.g., in \cite{2017Tabouy}), where the probability of sampling any pair depends on the matrices $\bL^*$ and $\bS^*$ (and, if we consider the  Stochastic Block Model, on the communities of the adjacent nodes). 

{\paragraph{Identifiability of the model} The matrices $\bL^*$ and $\bS^*$ appearing in the decomposition \eqref{eq:decA} may not be unique. Since we estimate $\bL^*$ and $\bS^*$ from a noisy, incomplete observation of their sum, we cannot achieve exact reconstruction of these matrices, and do not require strong identification conditions. We restrict our attention to pairs of matrices $\left(\bL^{(1)}, \bS^{(1)}\right)$ such that 
\begin{equation}\label{eq:1defLS}
       \left(\bL^{(1)}, \bS^{(1)}\right) \in \argmin \Big\{\rank(\bL) + \left \Vert\bS \right \Vert_{2,0}:
         \mathbb{E}\left[\bA\right] = \bL - \text{diag}(\bL) + \bS + \left(\bS\right)^{\top}, \left(\bL, \bS\right) \in \mathcal{M}\Big\},\nonumber
\end{equation}
where $\left \Vert\bS \right \Vert_{2,0}$ is the number of non-zero columns of the matrix $\bS$, and $\mathcal{M}$ is the set of admissible pairs of matrices :
\begin{align}
       \mathcal{M} = \Big\{& \left(\bL, \bS\right):   \bL \in [0, \rho_n]_{sym}^{n \times n},\ \ \bS \in [0, 1]^{n \times n},\  \forall j \in [n], \bS_{\cdot, j} \neq 0 \iff \bL_{\cdot, j} = 0 \Big\}.\nonumber
\end{align}
Among matrices verifying equation \eqref{eq:1defLS}, we choose to consider matrices $\bL$ with minimal rank, as they reflect our belief that inlier nodes should have a low-rank connectivity pattern. Thus, for $c = \rank(\bL^{(1)}) + \Vert\bS^{(1)} \Vert_{2,0}$, we define
\begin{equation}\label{eq:1defLSmin}
       \left(\bL^{*}, \bS^{*}\right) \in \argmin \Big\{\rank(\bL): \mathbb{E}\left[\bA\right] = \bL - \text{diag}(\bL) + \bS + \left(\bS\right)^{\top}, \left(\bL, \bS\right) \in \mathcal{M}, \ \rank(\bL) + \left \Vert\bS\right \Vert_{2,0} =  c \Big\}.
\end{equation}
Again, the solution of equation \eqref{eq:1defLSmin} may not be unique. We show in Section \ref{subsec:error_bounds} that under assumption \ref{ass:outlier_detection}, strong identifiability is guaranteed, and we can detect exactly all outliers with large probability.

When assumption \ref{ass:outlier_detection} does not hold, we can still show that all matrices $\bL^*$ solution to \eqref{eq:1defLSmin} are close to each other in Frobenius norm. By definition, all solutions $\left(\bL^{*}, \bS^{*}\right)$ of equation \eqref{eq:1defLSmin} are such that $\rank(\bL^*) = k$ and $\left \Vert\bS^*\right \Vert_{2,0} = s$. Moreover, for all solution $(\tilde{\bL}, \tilde{\bS}) \neq \left(\bL^{*}, \bS^{*}\right)$, we can show that $\bL^*$ and $\tilde{\bL}$ are close in Frobenius norm. Indeed, let $\cI = \{j: \bL^*_{\cdot, j} \neq \mathbf{0}\}$ (respectively $\tilde{\cI} = \{j: \tilde{\bL}_{\cdot, j} \neq \mathbf{0}\}$) be the support of the columns of $\bL^*$ (respectively of $\tilde{\bL}$), and $\cO = \{j: \bS^*_{\cdot, j} \neq \mathbf{0}\}$ (respectively $\tilde{\cO} = \{j: \tilde{\bS}_{\cdot, j} \neq \mathbf{0}\}$) be the support of the columns of $\bS^*$ (respectively of $\tilde{\bS}$). Then, $$\bL^* = \mathbb{E}[\bA]_{\vert \cI \times \cI}\text{ and }\tilde{\bL} = \mathbb{E}[\bA]_{\vert \tilde{\cI}\times \tilde{\cI}}.$$ Thus, $\bL^* - \tilde{\bL}$ is has support in the symmetrical difference between $\cI \times \cI$ and $\tilde{\cI} \times \tilde{\cI}$. Thus, $\bL^* - \tilde{\bL}$ has at most $$2 \vert (\cI \cap \tilde{\cO}) \times (\cI\cap \tilde{\cI})\vert + \vert (\cI \cap \tilde{\cO}) \times (\cI\cap \tilde{\cO})\vert  + 2 \vert (\tilde{\cI }\cap \cO) \times (\tilde{\cI}\cap \cI)\vert + \vert (\tilde{\cI} \cap \cO) \times (\tilde{\cI} \cap \cO)\vert$$ non zero entries, and each entry is bounded by $\rho_n$ (because it belongs either to $\cI$ or to $\tilde{\cI}$). Since $\vert \tilde{\cO}\vert = \vert \cO \vert = s$ and $\vert \tilde{\cI}\vert = \vert \cI \vert  \leq n$, the solution $\tilde{\bL}$ is therefore in a Frobenius ball of radius $\sqrt{(4ns + 2 s^2)}\rho_n \leq \sqrt{6ns}\rho_n$, centered at $\bL^*$. Now, Corollary \ref{corollary_fixed_lambda} ensures that our estimator $\widehat{\bL}$ is in a ball centered at $\bL^*$ of radius $$R = C\mu_n^{-1/2}\left(\frac{\nu_n}{\mu_n} \rho_n kn  + (\nu_n \rho_n \lor \tilde{\nu}_n \gamma_n) \rho_n s n\right)^{1/2},$$ where $\nu_n$, $\tilde{\nu}_n$ and $\mu_n$ are upper and lower bounds on the sampling probabilities defined in Section \ref{subsec:error_bounds}, $\gamma_n$ is an upper bound on the entries of $\mathbb{E}[A]$, and $C$ is an absolute constant. Since $R \geq \sqrt{6ns}\rho_n$, the distance between our estimator $\widehat{\bL}$ and any matrix $\tilde{\bL}$ solution of \eqref{eq:1defLSmin} is bounded by $2R$.

\section{Estimation procedure}
\label{subsec:algo}
In order to estimate the matrices $\bL^*$ and $\bS^*$, we consider the following objective function:
\begin{equation}
\label{eq:objective}
\mathcal{F}(\Sone, \bL) \eqdef \frac{1}{2}\norm{\Omega\odot(\bA - \bL - \Sone - (\bS)^{\top})}[F]^2 
+ \lambda_1\norm{\bL}[*] + \lambda_2\norm{\bS}[2,1],
\end{equation}
defined by a least squares data-fitting term penalised by a hybrid regularisation term. On the one hand, the nuclear norm penalty  $\norm{\bL}[*]$ is a convex relaxation of the rank constraint, meant to induce low-rank solutions for $\bL$. On the other hand, the term $\norm{\bS}[2,1]$ is a relaxation of the constraint on the number of non-zero columns in $\bS$, meant to induce column-wise sparse solutions for $\bS$. Our estimators are defined as
  \begin{equation}
  \label{objectif}
  \left( \widehat{\bS}, \widehat{\bL}\right) \in \underset{\bS \in [0,1]^{n \times n},  \bL\in [0,\rho_n]^{n \times n}_{sym}}{\argmin} \mathcal{F}\left( \bS, \bL\right).
  \end{equation}
When information on the presence or absence of some edges is missing, the objective function may not have a unique minimiser. We propose to approximate our target parameters $(\widehat{\bS}, \widehat{\bL})$ by minimising the objective \eqref{eq:objective} with an additional ridge penalisation term, $\frac{\epsilon}{2}(\norm{\bL}[F]^2 + \norm{\bS}[F]^2)$, which ensures strong convexity of the objective function. This additional penalty is not necessary to obtain convergence in terms of the objective value, and setting $\epsilon=0$ does not impact the convergence of the algorithm. However, it is required to obtain convergence of the parameters themselves: this additional penalty allows also to ensure approximate matching of the estimation and approximation errors, as detailed in our theoretical results. Note that, by choosing $\epsilon$ sufficiently small, $\cF_{\epsilon}$ can be arbitrarily close to $\cF$, but the choice of $\epsilon$ will impact the speed of convergence of our algorithm. 

Furthermore, we assume for simplicity that the box constraints on $\bS$ and $\bL$ are always inactive. We make a final simplification by dropping the symmetry constraint on $\bL$. Indeed, we will see later on that the low-rank matrix $\bL$ remains symmetric throughout the algorithm, provided that it is initialised by a symmetric matrix. Thus, in the end, we (approximately) solve the following optimisation problem:
\begin{equation}
\label{eq:estimation}
\begin{array}{ll}
\text{minimize} & \mathcal{F}_{\epsilon}(\bS,\bL) \triangleq \mathcal{F}(\bS,\bL) + \frac{\epsilon}{2}(\norm{\bL}[F]^2 + \norm{\bS}[F]^2).
\end{array}
\end{equation}
Let us now describe the optimisation procedure. First, we consider the augmented objective function:
\begin{equation*}
\label{eq:aug-objective}
\Phi_{\epsilon}(\bS, \bL, R) \eqdef \frac{1}{2}\norm{\Omega\odot(\bA - \bL - \bS - (\bS)^{\top})}[F]^2 
+ \lambda_1R + \lambda_2\norm{\bS}[2,1] + \frac{\epsilon}{2}(\norm{\bL}[F]^2 + \norm{\bS}[F]^2),
\end{equation*}
with $R\in\mathbb{R}_+$. Note that, if an optimal solution to \eqref{eq:estimation} $(\hat \bS_{\epsilon}, \hat \bL_{\epsilon})$ satisfies $\norm{\hat \bL_{\epsilon}}[*]\leq \bar R$ for some $\bar R \geq 0$, then any optimal solution to the augmented problem
\begin{equation}
\label{eq:aug-estimation}
\begin{array}{ll}
\text{minimise} & \Phi_{\epsilon}(\bS,\bL, R)\\
\text{such that} & \norm{\bL}[*]\leq R\leq \bar R
\end{array}
\end{equation}
will also be optimal to \eqref{eq:estimation} (we will show in appendix \ref{details_algo}  how the upper bound $\bar R$ can be chosen and tightened adaptively inside the algorithm). Thus,  solving \eqref{eq:aug-estimation} we directly obtain the solution to our initial problem \eqref{eq:estimation}. Finally, our estimators are defined as the minimisers of the following augmented objective function:
\begin{equation*}
\label{eq:aug-estimators}
\begin{array}{ll}
(\hat\bS_{\epsilon}, \hat\bL_{\epsilon}, \tilde R) & \in \text{argmin} \  \Phi_{\epsilon}(\bS,\bL, R)\\
\text{such that} & \norm{L}[*]\leq R \leq \bar R.
\end{array}
\end{equation*}
A natural option to solve problem \eqref{eq:aug-estimation} is the coordinate descent algorithm, where the parameters $(\bS, \bL, R)$ are updated alternatively along descent directions. To update $\bS$, we apply the proximal gradient method. We use the conjugate gradient method (or Frank-Wolfe method \cite{jaggi13}, which relies on linear approximations of the objective function) to update $(\bL, R)$. Similar Mixed Coordinate Gradient Descent (MCGD) algorithms were considered in \cite{Mu2016ScalableRM, Robin:2018:LIS, Garber2018} to estimate sparse plus low-rank decomposition with hybrid penalty terms combining an $\ell_1$ and a nuclear norm penalties. 
Here, we extend the procedure to handle the $\ell_{2,1}$ penalty as well. 
The details of the algorithm are described in Appendix \ref{details_algo}.
The entire procedure is sketched in Algorithm \ref{algo}, where we also define our final estimators $\left(\bL^{(T)}, \bS^{(T)}\right)$. 
\begin{algorithm}[H]
\caption{Mixed coordinate gradient descent (MCGD)}
\label{algo}
\begin{algorithmic}[1]
\STATE \textbf{Initialisation: } $(\bL^{(0)}, \bS^{(0)}, R^{(0)}, t) \leftarrow (\bold{0}, \bold{0}, 0, 0)$
\FOR{$t=1,\ldots,T$}
\STATE $t \leftarrow t+1$

\STATE \label{line:S}
Compute the proximal update \eqref{eq:Supdate1} to obtain $\Sprm^{(t)}$.
\STATE \label{line:up-bound}
Compute the upper bound $\bar{R}^{(t)}=\lambda_1^{-1}\Phi_{\epsilon}(\bS^{(t-1)}, \bL^{(t-1)}, R^{(t-1)})$.\\
\STATE \label{line:top-svd} Compute the direction $(\tilde\bL^{(t)}, \tilde{R}^{(t)})$ using \eqref{eq:Lprm1}.
\STATE \label{line:CG} Compute the Conjugate Gradient update \eqref{eq-Lupdate1}, with step size $\beta_t$ defined in \eqref{eq:step-size1} to obtain $(\bL^{(t)}, {R}^{(t)})$.
\ENDFOR
\RETURN {$\left(\bL^{(T)}, \Sprm^{(T)}\right)$}
\end{algorithmic}
\end{algorithm}
Denote by $\bG^{(t-1)}_L = -\Omega\odot(\bA-\bL^{(t-1)}-\bS^{(t)}-(\bS^{(t)})^{\top})+\epsilon\bL^{(t-1)}$ the gradient with respect to $\bL$ of the quadratic part of the objective function, evaluated at $(\bS^{(t)}, \bL^{(t-1)})$ and by $\bG^{(t-1)}_S = -2\Omega\odot(\bA-\bL^{(t-1)}-\bS^{(t-1)}-(\bS^{(t-1)})^{\top})+\epsilon\bS^{(t-1)}$ the gradient with respect to $\bS$ of the quadratic part of the objective function, evaluated at $(\bS^{(t-1)}, \bL^{(t-1)})$.
In Algorithm \ref{algo}, the column-wise sparse component $\bS$ is updated with a proximal gradient step:
\begin{equation}
\label{eq:Supdate1}
\begin{array}{ll}
\bS^{(t)} & \in \operatorname{argmin}\left(\eta\lambda_2 \left \Vert\bS \right \Vert_{2,1} + \frac{1}{2}\left \Vert \bS - \bS^{(t-1)} + \eta\bG_S^{(t-1)} \right \Vert_{F}^2 \right),\\
& = \mathsf{Tc}_{\eta\lambda_2}\left(\bS^{(t-1)} -\eta \bG_S^{(t-1)}\right),
\end{array}
\end{equation}
where $\mathsf{Tc}_{\eta\lambda_2}$ is the column-wise soft-thresholding operator such that for any $\Mprm \in \mathbb{R}^{n\times n}$ and for any $\lambda>0$, the $j$-th column of $\mathsf{Tc}_{\lambda}(\Mprm) $ is given by $(1-\lambda/\norm{\Mprm_{.,j}}[2])_+\Mprm_{.,j}$. The step size $\eta$ is constant and fixed in advance, and satisfies $\eta \leq 1/(2+\epsilon)$. The low-rank component given by $(\bL,R)$ is updated using a conjugate gradient step as follows:
\begin{equation}
\label{eq-Lupdate1}
\left(\bL^{(t)}, R^{(t)} \right) = \left(\bL^{(t-1)}, R^{(t-1)} \right) + \beta_t\left(\tilde\bL^{(t)}-\bL^{(t-1)}, \tilde R^{(t)}-R^{(t-1)} \right),
\end{equation}
where $\beta_t \in [0,1]$ is a step size set to:
\begin{equation}
\label{eq:step-size1}
\beta_t = \min\left\{1, 
\frac{\pscal{\bL^{(t-1)}-\tilde\bL^{(t)}}{\bG_L^{(t-1)}}+\lambda_1(R^{(t-1)}-\tilde R^{(t)})}{(1+\epsilon)\norm{\tilde{\bL}^{(t)}-\bL^{(t-1)}}[F]^2}\right\}.
\end{equation}
 The direction $(\tilde\bL^{(t)}, \tilde R^{(t)})$ is defined by:
\begin{equation}
\label{eq:Lprm1}
\begin{array}{rl}
\left(\tilde\bL^{(t)}, \tilde R^{(t)} \right) \in& \operatorname{argmin}_{\Zprm, R}\quad \pscal{\Zprm}{\bG_L^{(t-1)}} + \lambda_1R\\
\text{such that} & \norm{\Zprm}[*] \leq R \leq \bar R^{(t)}.
\end{array}
\end{equation}
Note that, if the matrix $\bL^{(t)}$ is symmetric, then the matrix $\bL^{(t+1)}$ remains symmetric at iteration $t+1$. Indeed, the gradient $\bG_{L}^{(t)}$ is defined in terms of the matrices $\bA$, $\bOmega$, and $\bS^{(t)} + (\bS^{(t)})^{\top}$, all three symmetric matrices. Therefore, to obtain a symmetric estimator of $\bL$, it suffices to initialise the algorithm with symmetric $\bL^{(0)}$. 

The Mixed Coordinate Gradient Descent algorithm described in Algorithm~\ref{algo} converges sublinearly to the optimal solution of \eqref{eq:aug-estimation}, as shown by the following result:
\begin{thm}\label{thm:cvg} Let $\delta>0$. After $T_{\delta} = \mathcal{O}(1/\delta)$ iterations, the iterate satisfies:
\begin{equation}
    \label{eq:cvg:obj}
    \mathcal{F}_{\epsilon}(\bS^{(T_{\delta})}, \bL^{(T_{\delta})})- \mathcal{F}_{\epsilon}(\hat\bS_{\epsilon}, \hat\bL_{\epsilon})\leq \delta.
\end{equation}
In addition, by strong convexity of $\mathcal{F}_{\epsilon}$,
\begin{equation}
    \label{eq:cvg:prm}
    \norm{\bS^{(T_{\delta})}-\hat\bS_{\epsilon}}[F]^2+\norm{ \bL^{(T_{\delta})}-  \hat\bL_{\epsilon}}[F]^2\leq \frac{2 \delta}{\epsilon}.
\end{equation}
\end{thm}
\noindent In Appendix \ref{proof:cvg} we provide a more detailed result, with an estimation of the constant in $\mathcal{O}(1/\delta)$.

\section{Theoretical analysis of the estimator}

\label{subsec:error_bounds}

In this section we provide theoretical analysis of our algorithm. First, we provide guarantees on the support recovery of the outliers. Next, we prove a non asymptotic bound on the risk of our estimator. We start by introducing assumptions on the missing values  mechanism.

\subsection{Assumption on the sampling scheme}

Our first assumption on the sampling scheme requires that all the edges between the inliers  are observed with a non-vanishing probability. Recall that $I = \cI \times \cI$ denote the pairs of inlier nodes.
\begin{ass}\label{ass:quasi_uniforme}
There exist a strictly positive sequence $\mu_n$ such that for any $(i,j) \in I$, $\mu_n \leq \bPi_{ij}$.
\end{ass}
\noindent Bounding the probabilities of observing any entry away from $0$ is a usual assumption in the literature dealing with missing observations (different patterns for missing observations are discussed, e.g., in \cite{klopp2014, koltchinskii2011, Negahban:2012}). We denote by $\nu_n$ and $\tilde{\nu}_n$ two sequences such that for any $i \in I$, $\sum_{j \in \cI} \bPi_{ij}\leq \nu_n n$ and for any $i \in [n]$, $\sum_{j \in \cO}\bPi_{ij} \leq \tilde{\nu}_ns$. We always have $\nu_n \leq 1$ and $\tilde{\nu}_n \leq 1$, but when $\nu_n$ and $\tilde{\nu}_n$ are decreasing sequences, we obtain better error rates by taking advantage of the fact that observations are distributed over different nodes in the network. Note that our estimators do not require the knowledge of the sequences $\mu_n$ and $\tilde{\nu}_n$. On the other hand, for the theoretical analysis we need an upper bound on  $\nu_n\rho_n n$ (the average observed connectivity of inlier nodes), which can be estimated robustly (for example by using Median of Means \cite{Mom}).

Recall that we do not observe any entry on the diagonal of $\bA$. Combined with Assumption \ref{ass:quasi_uniforme}, this implies that for any matrix $\bM \in \mathbb{R}^{n \times n}$
\begin{equation}\label{eq:eqnorms}
       \left \Vert \bM_{\vert I} \right\Vert^2_F \leq \frac{1}{\mu_n}\left \Vert \bM \right\Vert^2_{L_2(\bPi)} + n \left \Vert\bM \right\Vert^2_{\infty}.
\end{equation}
Moreover, since $\left \vert O \right \vert = 2ns + (s-1)(s-2)/2 \leq 3ns$, we find that
\begin{equation}\label{eq:eqnormsoutliers}
\left \Vert \bM_{\vert O} \right\Vert^2_F \leq 3ns\left \Vert\bM \right\Vert_{\infty}^2.
\end{equation}
Before stating the second assumption, recall that $\rho_n$ is a sparsity inducing sequence such that $\left \Vert \bL^{*} \right \Vert_{\infty} \leq \rho_n$. Similarly, we define $\gamma_n = \left \Vert \mathbb{E}[\bA] \right \Vert_{\infty}$. Since $\left \Vert \bS^{*} \right \Vert_{\infty} \leq \gamma_n$, $\gamma_n$ characterises the sparsity of connections of the outlier nodes. Note that outliers and inliers may have different sparsity levels, i.e., $\gamma_n$ and $\rho_n$ may be of different orders of magnitude.
\begin{ass}\label{ass:moderately_sparse}
$\nu_n \rho_n \geq \log(n)/n$ and $\tilde{\nu}_n \gamma_n \geq \log(n)/n$.
\end{ass}
\noindent Assumption \ref{ass:moderately_sparse} implies that the \textit{observed} average node degree is not too small. Note that considering very sparse graphs, where the expectation of the probability of observing an edge is of order $\frac{1}{n}$, is of lesser interest since it has been shown in \cite{GaoBiclustering} that the trivial null estimator is minimax optimal in this setting. On the other hand, the sparsity threshold $\log(n)/n$ is known to correspond to phase transition phenomenons for recovering structural properties in the SBM \cite{Abbe18}. We also need the following assumption on the ``signal to noise ratio".
\begin{ass}\label{ass:out_con}
$\nu_n\rho_n n \geq\tilde{\nu}_n\gamma_ns$
\end{ass}
\noindent Here, edges connecting inliers to inliers can be seen as a ``signal term" in the estimation of connection probabilities, while edges connecting outliers to any other nodes can be seen as a ``noise term". Now, recall that $\rho_n$ bounds the probability of any inlier to be connected to any inlier, while $\gamma_n$ bounds the probability of any inlier to be connected to any outlier. Then, Assumption \ref{ass:out_con} requires that we observe more connection between inliers than between inliers and outliers, or equivalently that the ``signal" induced by the connections of the inliers be stronger than the ``noise". For example, under a uniform sampling, all entries are observed with the same probability, so $\mu_n = \nu_n = \tilde{\nu}_n = p$. Then, Assumption \ref{ass:out_con} becomes $\rho_n n \geq\gamma_ns$, and requires that inlier nodes be more connected with other inlier nodes than with outliers. As the number of outliers $s$ is typically much smaller than the number of inlier nodes, $n-s$,  this assumption is not restrictive.
\subsection{Outlier detection} 
The $\Vert \cdot \Vert_{2,1}$-norm penalisation induces the column-wise sparsity of the estimator $\widehat{\bS}$ (when appropriately calibrated, it allows only a small number of columns of $\widehat{\bS}$ to be non-zero). Using this  sparsity, we define the set of estimated outliers as 
\begin{equation}
\label{eq:estimated_outliers}
\widehat{\cO} \triangleq \left \{ j\in [n]: \widehat{\bS}_{\cdot,j} \neq \mathbf{0}\right \}.
\end{equation}
The following lemma, proven in  Appendix  \ref{subsubsec:syst_S}, provides a characterisation of this set:
\begin{lem}\label{lem:syst_S}
For any $j \in [n]$, $\widehat{\bS}_{\cdot, j} \neq \bold{0} \iff \left \Vert \bOmega_{\cdot, j} \odot \left(\bA_{\cdot,j} - \widehat{\bL}_{\cdot,j} - \widehat{\bS}_{j, \cdot} \right)_+\right \Vert_2 > \frac{\lambda_2}{2}.$
\end{lem}
 \noindent Lemma \ref{lem:syst_S} provides a lower bound on $\lambda_2$ that will prevent from erroneously reporting inliers as outliers by choosing  $\lambda_2$ larger than the expected norm of columns corresponding to inliers. Note that for any inlier $j$, $\mathbb{E}[\Vert (\bOmega \odot (\bA_{\cdot,j} - \bL^*_{\cdot, j})_{\vert I})_+ \Vert_2]$ is of the order $\sqrt{\nu_n\rho_n (n-s) + \tilde{\nu}_n \gamma_ns}$. If $\lambda_2$ falls below this threshold, some inliers are likely to be erroneously reported as outliers.
 Therefore, we choose $\lambda_2 \gtrsim \sqrt{\nu_n\rho_n (n-s) + \tilde{\nu}_n\gamma_ns}$. Under Assumption \ref{ass:out_con}, this condition becomes $\lambda_2 \gtrsim \sqrt{\nu_n\rho_n n}$. With this choice of $\lambda_2$ we have the following results proven in Appendix \ref{subsec:Proof_inliers_detection}:
\begin{thm}\label{thm:inliers_detection} Let $\lambda_2 = 19 \sqrt{\nu_n\rho_n n}$. Then, under Assumptions \ref{ass:quasi_uniforme}-\ref{ass:out_con}, there exists an absolute constant $c>0$ such that with  probability at least $1 - c/n$
\begin{equation}
    \widehat{\cO} \cap \cI = \emptyset.\label{eq:detectedOutliers}
\end{equation}
\end{thm}
\noindent One cannot hope to further separate outliers from inliers without additional assumptions on how the first group differs from the second one. Here, we provide an intuition about our condition on the connectivity of outliers that is sufficient for outliers detection. According to Lemma \ref{lem:syst_S}, any outlier $j$ will be reported as such if  $\Vert (\bOmega_{\cdot,j} \odot (\bA_{\cdot,j} - \widehat{\bL}_{\cdot,j} - \widehat{\bS}_{j, \cdot} ) )_+\Vert_2 > \lambda_2/2$. So, in order to detect an outlier $j$, the threshold $\lambda_2$ must be at least smaller than $\mathbb{E}[\Vert (\bOmega_{\cdot,j} \odot (\bA_{\cdot,j} - \widehat{\bL}_{\cdot,j} - \widehat{\bS}_{j, \cdot} ) )_+ \Vert_2 ]$. Recalling that $\widehat{\bL}$ and $\widehat{\bS}$ have non-negative entries, we see that $$ \mathbb{E}\left[\left \Vert \left(\bOmega_{\cdot,j} \odot \left(\bA_{\cdot,j} - \widehat{\bL}_{\cdot,j} - \widehat{\bS}_{j, \cdot} \right) \right)_+\right \Vert_2 \right] \leq \mathbb{E}\left[\left \Vert \left(\bOmega_{\cdot,j} \odot \left(\bA_{\cdot,j}\right)\right)_+\right \Vert_F\right]=\sqrt{\underset{i \in \cI}{\sum} \bPi_{ij}\bS^*_{ij} + \underset{i \in \cO}{\sum}\bPi_{ij}(\bS^*_{ij} + \bS^*_{ij}}).$$ Thus, the condition $\sqrt{\nu_n\rho_n n} \lesssim \lambda_2 \lesssim \min_{j \in \cO}\sqrt{\sum_{i \in \cI}\bPi_{ij}\bS^*_{ij}}$ appears naturally when separating the inliers from the outliers. This condition is formalised in the following assumption:
\begin{ass}\label{ass:outlier_detection}
$\min_{j \in \cO}\sum_{i \in \cI}\bPi_{ij}\bS^*_{ij} > C\rho_n\nu_n n$ where $C$ is an absolute constant  defined in Section \ref{proof:outliers_detection}.
\end{ass}

\noindent When the outliers represent only a small fraction of the nodes, we have that $\vert\cI\vert \simeq n$. Then, Assumption \ref{ass:outlier_detection} is met when outlier nodes have higher expected observed degree than inlier nodes. When the sampling probabilities are uniform, this assumption essentially reads $\gamma_n \geq C\rho_n$. This assumption is compatible with assumption \ref{ass:out_con}, as the number of outliers $s$ is typically much smaller than the number of nodes $n$. The following Lemma shows that assumption \ref{ass:outlier_detection} ensures strong identifiability of the set of outliers and inliers.

\begin{lem}\label{lem:unicite}
Under assumption \ref{ass:outlier_detection}, the solution $(\bL^*, \bS^*)$ to equation \eqref{eq:1defLSmin} is unique up to diagonal terms.
\end{lem}
\noindent Lemma \ref{lem:unicite} ensures that under Assumption \ref{ass:outlier_detection}, the set of outliers is well defined. Moreover, all outliers are detected with large probability, as indicated by} the following result proven in Appendix \ref{proof:outliers_detection}:
 \begin{thm}\label{thm:outliers_detection} Let $\lambda_2 = 19 \sqrt{\nu_n\rho_n n}$. Under Assumptions \ref{ass:quasi_uniforme}-\ref{ass:outlier_detection}, there exists an absolute constant $c>0$ such that $\cO = \widehat{\cO}$ with probability at least $1 - cs/n$.
\end{thm}
\noindent Theorem \ref{thm:outliers_detection} provides guarantees on the recovery of the support of the column-sparse component of the decomposition \eqref{eq:decA}. To the best of our knowledge, this is the first result of this sort in the noisy setting, where the exact reconstruction of both components, the low-rank and the sparse one, is impossible. For both Theorem \ref{thm:outliers_detection} and Theorem \ref{thm:inliers_detection}, we actually show that the results hold with probabilities at least $1 - 8se^{-c_nn}$ and $1 - 6e^{-c_nn}$ respectively, where $c_n$ is a sequence depending on $\nu_n$ and $\rho_n$  such that $c_n \geq \log(n)/n$. 

\subsection{Estimation of the connections probabilities}
In this section, we establish the non-asymptotic upper bound on the risk of our estimator. We denote the noise matrix $\bSigma \triangleq \bA - \mathbb{E}[\bA]$. Let $\bGamma$ be the random matrix defined as follows: for any $(i,j)$, $\bGamma_{ij} \triangleq \epsilon_{ij}\bOmega_{ij}$, where $\left\{\bepsilon_{ij}\right\}_{1 \leq i < j \leq n}$ is a Rademacher sequence. To clarify the exposition of our results, we introduce the following error terms $$\Phi \triangleq n\rho_n^2\left( \frac{\nu_n k}{\mu_n} + \nu_n s\right),\ \ \Psi \triangleq 16\tilde{\nu}_n\gamma_n\rho_n s n\  \text{ and } \ \Xi \triangleq \frac{\sqrt{\nu_n n  } \rho_n}{\lambda_1}+ 1.$$
The following theorem, proven in  Appendix \ref{subsec:ProofTh4}, provides  the error bound for the risk of the  estimator $\widehat{\bL}$ that depends on the choice of the regularisation parameter $\lambda_1$:
\begin{thm}\label{thm:Borne_norme_2}
Assume that $\lambda_1 \geq 3\left \Vert\bOmega \odot \bSigma_{\vert I}\right \Vert_{op}$, and that $\lambda_2 = 19 \sqrt{\nu_n\rho_nn}$. Then, under Assumptions \ref{ass:quasi_uniforme}-\ref{ass:out_con}, there exists  absolute constants $C>0$ and $c>0$ such that with probability at least $1 - c/n$,
\begin{eqnarray}
\left \Vert \left( \widehat{\bL} - \bL^*\right)_{\vert I}\right \Vert_{L_2(\bPi)}^2 &\leq& C\left(\frac{\lambda_1^2k}{\mu_n} +  \Phi + \Xi\Psi \right).\label{eq:thm_bound_L}
\end{eqnarray}
\end{thm}
\noindent 
Next, we provide a choice for $\lambda_1$ such that the condition $\lambda_1 \geq 3\left \Vert\bOmega \odot \bSigma_{\vert I}\right \Vert_{op}$ holds with high probability. To do so, we must first obtain a high-probability bound on $\left \Vert \bOmega\odot\bSigma_{\vert I} \right \Vert_{op}$. This is done in the following Lemma:
\begin{lem}\label{lem:bound_lambda}
$
\mathbb{P}\left(\left \Vert \bOmega\odot\bSigma_{\vert I} \right \Vert_{op} \geq 28\sqrt{\nu_n\rho_n n} \right) \leq e^{-\nu_n\rho_n n}. 
$
\end{lem}
\noindent Using Lemma \ref{lem:bound_lambda}, we obtain the following corollary proven in  Appendix \ref{subsec:ProofCor1}:
\begin{cor}\label{corollary_fixed_lambda}
Choose $\lambda_1 = 84\sqrt{\nu_n \rho_n n  }$ and $\lambda_2 = 19\sqrt{\nu_n \rho_n n }$. Then, under the conditions of Theorem \ref{thm:Borne_norme_2}, there exists absolute constants $C>0$ and $c>0$ such that with probability at least $1 - c/n$,
\begin{eqnarray}
  \left \Vert \left(\widehat{\bL} - \bL^*\right)_{\vert I}\right \Vert_{L_2(\bPi)}^2 &\leq& C\left(\frac{\nu_n}{\mu_n} \rho_n kn  + (\nu_n \rho_n \lor \tilde{\nu}_n \gamma_n) \rho_n s n \right)\label{eq:thm_bound_L_cor1}
\end{eqnarray}
and 
\begin{eqnarray}
  \left \Vert \left(\widehat{\bL} - \bL^*\right)_{\vert I}\right \Vert_{F}^2 &\leq& \frac{C}{\mu_n}\left(\frac{\nu_n}{\mu_n} \rho_n kn  + (\nu_n \rho_n \lor \tilde{\nu}_n \gamma_n) \rho_n s n \right).\label{eq:thm_bound_L_cor2}
\end{eqnarray}
\end{cor}

\begin{remark}
The estimator  $(\widehat{\bL}, \widehat{\bS})$ returned by the MCGD Algorithm does not have the property $\widehat{\bL}_{\cdot, j} \neq \mathbf{0} \iff \widehat{\bS}_{\cdot, j} = \mathbf{0}$ (non-overlapping support). To obtain estimators verifying this property, we may define a new estimator $\widehat{\bL}'$ for $\bL^*$ such that 
\[
    \widehat{\bL}'_{ij}=\left\{
                \begin{array}{ll}
                  \widehat{\bL}_{ij} \text{ if } j \notin \widehat{\cO}\\
                  0 \text{ if } j \in \widehat{\cO}
                \end{array}
              \right.
\]
Note that $\widehat{\bL}' = \mathcal{P}_{\widehat{\cI} \times \widehat{\cI}}(\widehat{\bL})$, where $\widehat{\cI} = [n] \backslash \widehat{O}$ is the set of estimated inliers, and $\mathcal{P}_{\widehat{\cI} \times \widehat{\cI}}$ is the orthogonal projection onto the set of matrices with support in $\widehat{\cI} \times \widehat{\cI}$. Using Theorem \ref{thm:inliers_detection}, we find that with high probability, $\bL^*$ has a support in $\widehat{\cI} \times \widehat{\cI}$. Then, classical properties of orthogonal projections ensure that $\Vert \bL^* - \widehat{\bL}' \Vert_F \leq \Vert \bL^* - \widehat{\bL} \Vert_F$. Thus, the new estimator $(\widehat{\bL}', \widehat{\bS})$ achieves the same error rate as the estimator $(\widehat{\bL}', \widehat{\bS})$ and detects the same outliers, while having non-overlapping support.
\end{remark}

To get a better understanding of the results of Corollary \ref{corollary_fixed_lambda}, we consider the following simple example. We consider a missing data scheme where all entries of $\bA$ are observed with the same probability $p$ (that is, $\nu_n = \tilde{\nu}_n = \mu_n = p$). Then, the error of our estimator $\hat{\bL}$ in Frobenius norm  is at most $O(\rho_n k n/p + \rho_n\gamma_n s n)$. Assume now that the number of outliers $s$ is bounded by $k/(p\gamma_n)$ (note that when the network is sparse, $\gamma_n \rightarrow 0$ and thus the number of outliers may grow). Then, the error rate is of the order $O(\rho_n k n/p)$, which corresponds to the minimax optimal rate for the low-rank matrix estimation problem without outliers. By comparison, applying methods from the low-rank matrix completion literature, we obtain an error rate of the order $O(kn/p)$, which is sub-optimal since $\rho_n$ is typically of the order of $\log(n)/n$.\newline

To the best of our knowledge, no results on robust estimation of the connection probabilities in the presence of outliers and missing observations have been established before. Previous rates of convergence for the problem of estimating the connection probabilities under the Stochastic Block Model with missing links have been established, for the uniform sampling scheme, in \cite{GaoBiclustering}, and, for more general sampling schemes, in \cite{gaucher}. To compare our bound with these previous results, we consider the case of the uniform sampling and  assume that the condition $\left(\tilde{\nu}_n\mu_n \lor \nu_n \rho_n \right) s \leq\nu_nk/\mu_n$ is met. In \cite{GaoBiclustering} and \cite{gaucher}, the authors show that the risk of their estimators in $\left\Vert \cdot \right \Vert_{L_2(\bPi)}$-norm is of the order $\rho_n\left(\log(k)n + k^2\right)$, and that it is minimax optimal.  The rate provided by Corollary \ref{corollary_fixed_lambda}  is of the order $ \rho_n kn$. So, for the relevant case  $k\leq \sqrt{n}$, our method falls short of the minimax optimal rate for this problem by a factor $k/\log(k)$. Note that, estimators proposed in \cite{GaoBiclustering} and \cite{gaucher} have non-polynomial computational cost while our estimator can be used in practice. On the other hand, the authors of \cite{spectralGraphon} propose a polynomial-time algorithm for estimating the probabilities of connections in the Stochastic Block Model under complete observation of the network. They show that the risk of their estimator for the connection probabilities is bounded by $C\rho_nkn$. Thus,  our method matches the best known rate established for a polynomial time algorithm for the Stochastic Block Model while being robust to missing observations and outliers.
\section{Numerical experiments}
\label{section:simul}

\subsection{Outliers detection}
\label{section:simul:detection}

In this section, we illustrate the performance of our method in terms of outliers detection on two different types of outliers: hubs and mixed membership profiles. We start by generating a graph containing $n=1000$ inlier nodes according to the Stochastic Block Model with three communities of approximately the same size. In each community, the probability of connection between nodes is equal to $p=0.05$. The probability of connection between  communities is equal to $q=0.01$. With this choice of parameters, the average node degree is of the order of $\log(n)$. Then, we generate $s=20$ outlier nodes using the following two methods:
\begin{enumerate}
    \item \textbf{Hub}:  outlier $j$ connects to any other node $i$ with probability $\pi_{\textsf{hub}}$.
    \item \textbf{Mixed membership}:  for any outlier $j$, we select at random two communities. For any other node $i$, if it belongs to one of the two communities, outlier $j$ connects to $i$ with probability $\pi_{\textsf{mix}}$. Otherwise, it connects to $i$ with probability $q=0.01$.
\end{enumerate}
Finally, we introduce 20\% of missing values in the adjacency matrix uniformly at random. For each of the two types of outliers, we consider increasing values of the ratio $$\rho = \frac{\min_{j\in\mathcal{O}}\sum_{i\in\mathcal{I}}\mathbf{\Pi}_{ij}\bS^*_{ij}}{\tilde\nu_n\rho_n n},$$ highlighted in Theorem \ref{thm:outliers_detection} as the crucial quantity to guarantee strong identification of the outliers (see Assumption \ref{ass:outlier_detection}). In our case, it is of the order of $\rho_{\textsf{hub}} = \frac{3\pi_{\textsf{hub}}}{p+2q}$ for hubs, and  $\rho_{\textsf{mix}} = \frac{2\pi_{\textsf{mix}}+q}{p+2q}$ for mixed membership nodes. We fix the size of the network $n=1000$, the number of outliers $s=20$ and the connection probability intra and inter communities $p=0.05$ and $q=0.02$. Then, we generate outliers with increasing values of $\pi_{\textsf{hub}}$ and $\pi_{\textsf{mix}}$ so that the ratios $\rho_{\textsf{hub}}$ and $\rho_{\textsf{mix}}$ spans the range $(0.6, 2)$. For each value of $\rho_{\textsf{hub}}$ and $\rho_{\textsf{mix}}$, we apply our algorithm to detect outliers, fixing the parameters $\lambda_1$ and $\lambda_2$ to their theoretical values. The results are presented in Figures \ref{subfig:outlier_detection_hub} and \ref{subfig:outlier_detection_mix}, where we display the power ($\frac{|\hat{\mathcal{O}}\cap \mathcal{O}|}{|\mathcal{O}|}$) and the False Discovery Rate (FDR, $\frac{|\hat{\mathcal{O}}\cap \mathcal{I}|}{|\mathcal{O}|}$) for hubs and mixed membership nodes, respectively. In both cases, the limit $\rho = 1$ is indicated with a dashed black line. Note that, the theoretical detection limit given in Assumption
~\ref{ass:outlier_detection} yields $\rho\geq 152 \gg 1$ (see \ref{proof:outliers_detection}). Thus, our empirical results show that our algorithm is in fact able to detect outliers at much lower signal-to-noise ratio than predicted by theory. In addition we emphasize that, for $\rho = 1$, outliers have approximately the same degree as the inliers, and thus cannot be detected by inspecting the histogram of degrees.

\begin{figure}
\centering
\begin{subfigure}[b]{0.47\textwidth}
\centering
\includegraphics[width=\textwidth]{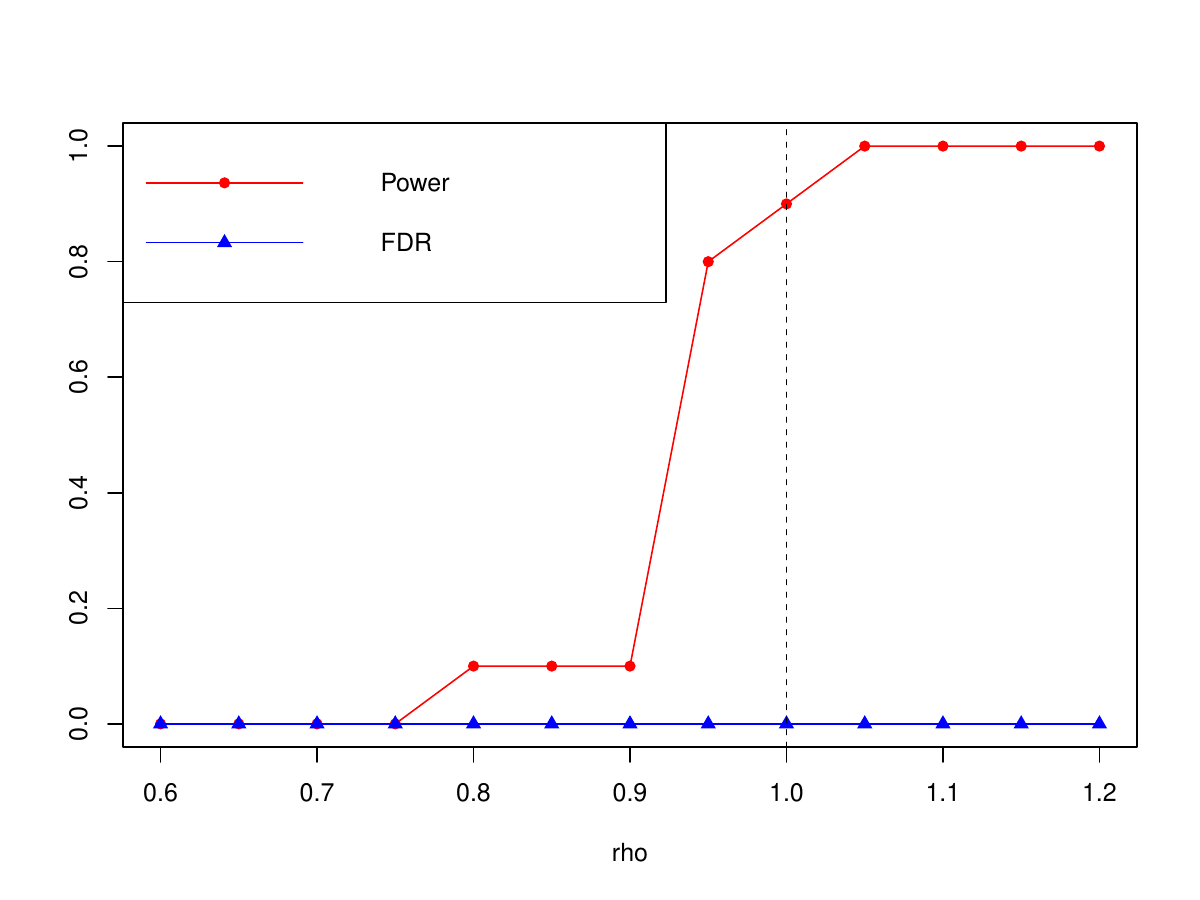}
\caption{\textbf{Hubs} detection: \textbf{Power} (red points) and \textbf{FDR} (blue triangles) for increasing $\rho_{\textsf{hub}} \sim \pi_{\textsf{hub}}/p$, averaged across 10 replications. $\rho_{\textsf{hub}} = 1$ indicated with dashed black line.}
\label{subfig:outlier_detection_hub}
\end{subfigure}
\hspace{0.2cm}
\begin{subfigure}[b]{0.47\textwidth}
\centering
\includegraphics[width=\textwidth]{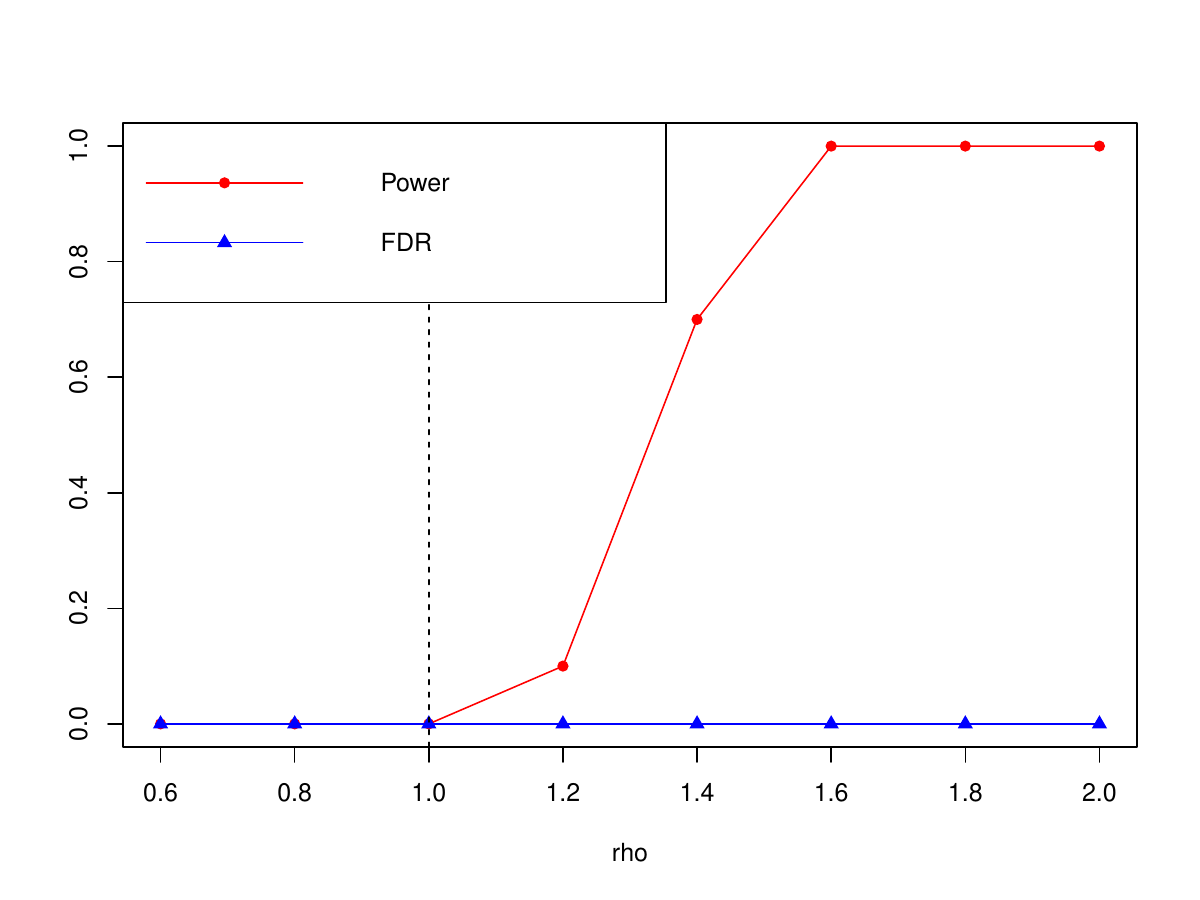}
\caption{\textbf{Mixed membership} detection: \textbf{Power} (red points) and \textbf{FDR} (blue triangles) for increasing $\rho_{\textsf{mix}} \sim \pi_{\textsf{mix}}/p$, averaged across 10 replications. $\rho_{\textsf{mix}} = 1$ indicated with dashed black line.}
\label{subfig:outlier_detection_mix}
\end{subfigure}
\end{figure} 

Our numerical results show that for outliers with hubs profiles (Figure \ref{subfig:outlier_detection_hub}), our algorithm successfully detects the outliers, including in ``hard'' settings where their average degree is the same as inliers. Our simulations also confirm the relevance of our theoretical findings, which highlight the importance of the ratio $\frac{ \min_{j\in\mathcal{O}}\sum_{i\in\mathcal{I}}\mathbf{\Pi}_{ij}\bS^*_{ij}}{\tilde\nu_n\rho_n n}$ for outliers detection, even though our theoretical constants may not be optimal. Finally, note that, using the theoretical values of $\lambda_1$ and $\lambda_2$, our algorithm almost never falsely labels inliers as outliers (FDR is consistently $0$). In the case of outliers with mixed membership profiles, we observe a similar behaviour. However, the empirical value of $\rho_{\textsf{mix}}$ required for exact outliers selection is in this case of the order of $\rho_{\textsf{mix}} \simeq 1.6$, slightly above the observed limit for hubs $\rho_{\textsf{hub}} 
\simeq 1$. This seems to indicate that, in practice, mixed membership nodes are ``harder'' to detect than hubs.  

\subsection{Estimation of connection probabilities}
\label{section:simul:link-prediction}

We now evaluate the performance of our method in terms of estimation of the connection probabilities of inliers. As before, we start by generating a network of size $n=1000$ using the Stochastic Block Model with three balanced communities. We keep the same parameters for the SBM, with $p=0.05$ and $q=0.01$, and introduce 20\% of missing values. Then, we study two settings where we introduce $s$ outliers corresponding to hubs and mixed membership nodes, respectively. For each of the two types of outliers, we consider increasing values of the ratio $$\tau = \frac{\rho_n\nu_n n}{\tilde\nu_n\gamma_n s},$$ highlighted in Corollary \ref{corollary_fixed_lambda} as the signal to noise ratio for the problem of estimation of the connection probabilities of inliers (see Assumption \ref{ass:out_con}). In our case, it is of the order of $\tau_{\textsf{hub}} = \frac{n(p+2q)}{3s\pi_{\textsf{hub}}}$ for hubs,  and $\tau_{\textsf{mix}} = \frac{n(p+2q)}{s(2\pi_{\textsf{mix}}+q)}$ for mixed membership nodes.

We fix the size of the network $n=1000$, the intra- and inter-communities connection probabilities $p=0.05$ and $q=0.02$, and the values $\pi_{\textsf{hub}} = 0.2$ and $\pi_{\textsf{mix}}=0.3$. Note that, these values of $\pi_{\textsf{hub}}$ and $\pi_{\textsf{mix}}$ produce outliers which are much more connected than inliers. This corresponds to a setting where the detection of outliers is ``easy'' because they have large degrees, but the estimation of the connection probabilities of inliers (parameter $\bL^*$) is ``hard'' because outliers have many connections polluting the network.  Then, we generate an increasing number of outliers ($s=20$, $s=50$, $s=100$), so that the signal to noise ratios $\tau_{\textsf{hub}}$ and $\tau_{\textsf{mix}}$ take different values ($5$, $2$, and $1$). For each value of $\tau_{\textsf{hub}}$ and $\tau_{\textsf{mix}}$, we estimate the connection probabilities of inliers, fixing the parameters $\lambda_1$ and $\lambda_2$ to their theoretical values. In each case, we compare the estimation results with two competitors: the method implemented in the R \citep{R} package \href{https://CRAN.R-project.org/package=missSBM}{\texttt{missSBM}} \citep{2017Tabouy,tabouy2019misssbm} which fits a Stochastic Block Model in the presence of missing links, and matrix completion as implemented in the R package \href{form https://CRAN.R-project.org/package=softImpute }{\texttt{softImpute}} \citep{Hastie:2015:MCL:2789272.2912106}; the methods are compared in terms of the Mean Squared Error (MSE) of estimation, normalized by size of the set of pairs inliers $I = \mathcal{I}\times \mathcal{I}$. The MSE is thus defined for some estimator $\hat\bL$ by:
$$\text{MSE}(\hat{\bL}) = \frac{\left\| \hat{\bL}_{\vert I} - \bL^*_{\vert I}\right\|_F^2}{|I|}.$$ 
The results are presented in Figure \ref{fig:link_hub} for hubs and Figure \ref{fig:link_mix} for mixed membership nodes, which display (on the same scale) boxplots of the MSE of each method obtained by 10 replications of the experiment, for different values of the signal to noise ratio (from left to right: $\tau = 1$, $\tau = 2$, $\tau = 5$).

\begin{figure}
\centering
\begin{subfigure}[b]{0.31\textwidth}
\centering
\includegraphics[width=\textwidth]{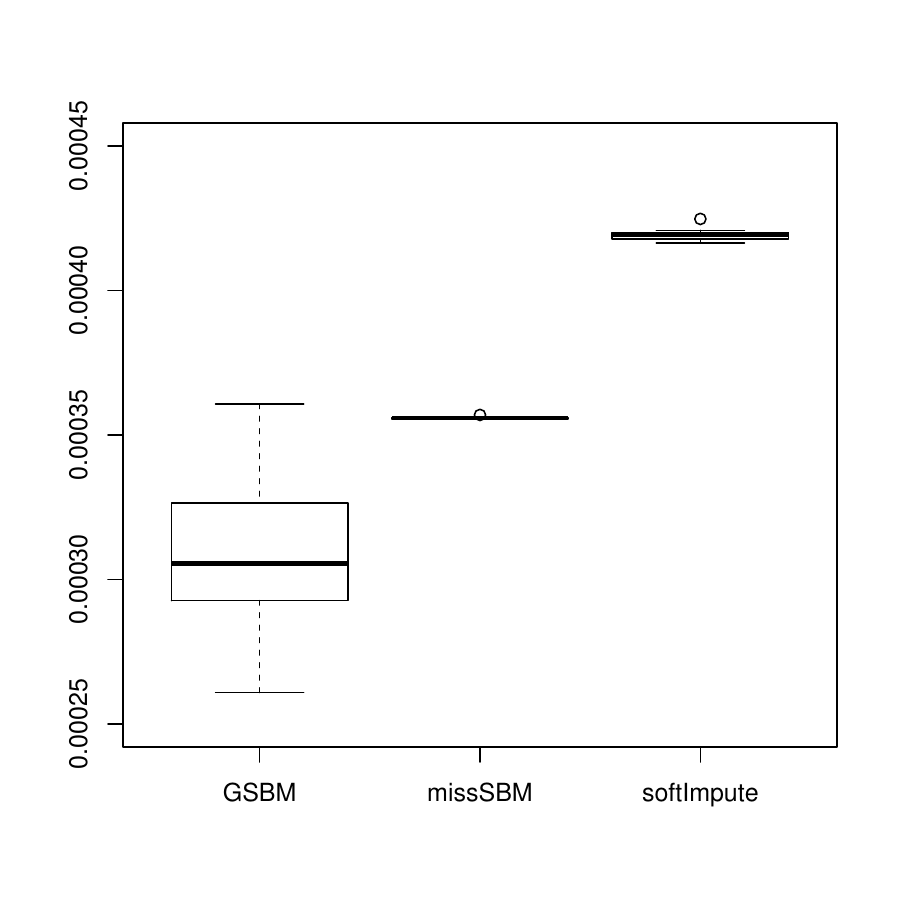}
\caption{$\tau_{\textsf{hub}} = 1$ (100 outliers)}
\label{subfig:link_hub_1}
\end{subfigure}
\hspace{0.2cm}
\begin{subfigure}[b]{0.31\textwidth}
\centering
\includegraphics[width=\textwidth]{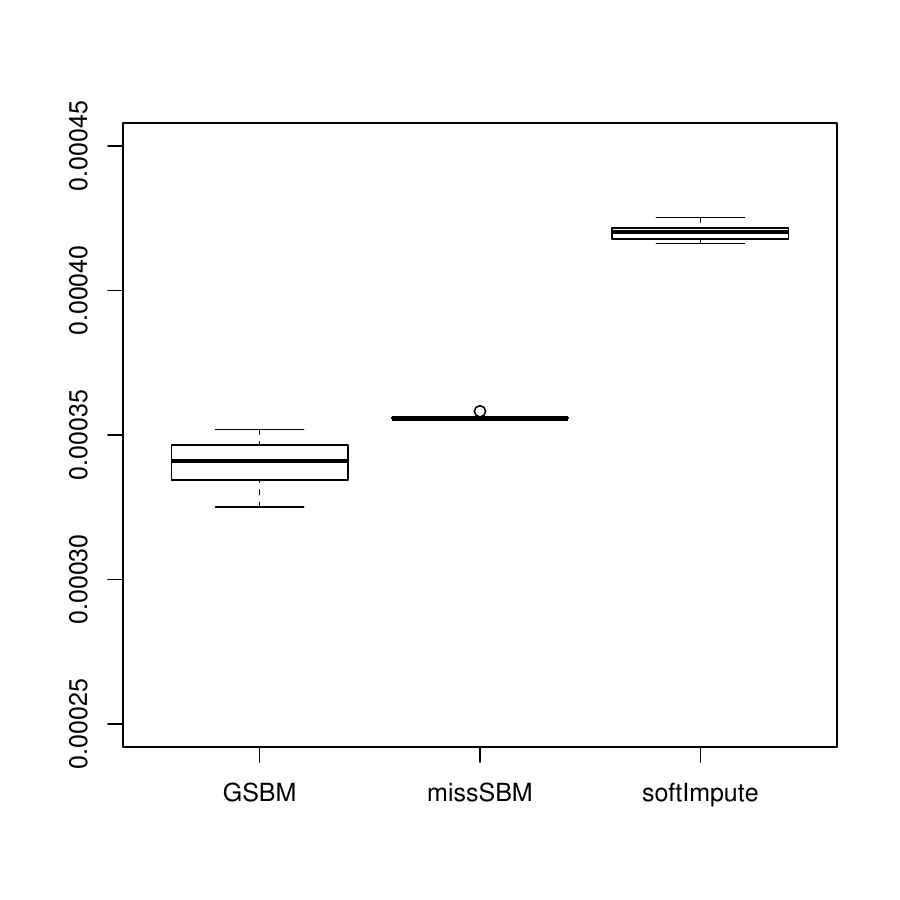}
\caption{$\tau_{\textsf{hub}} = 2$ (50 outliers).}
\label{subfig:link_hub_2}
\end{subfigure}
\hspace{0.2cm}
\begin{subfigure}[b]{0.31\textwidth}
\centering
\includegraphics[width=\textwidth]{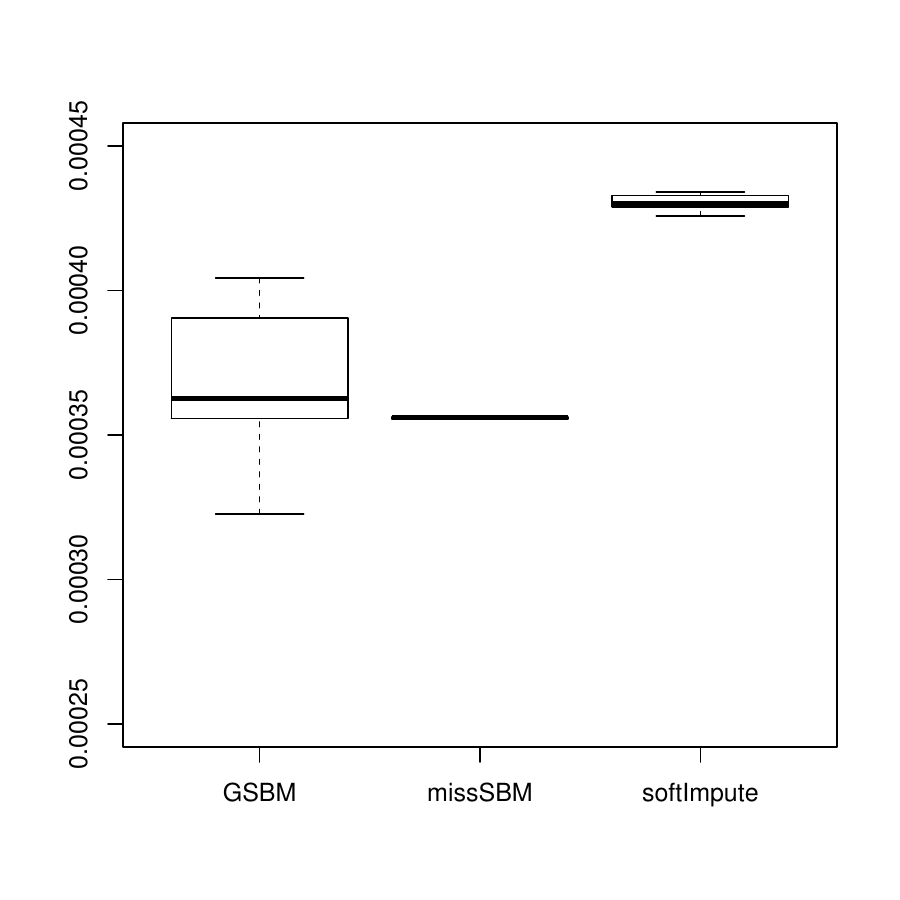}
\caption{$\tau_{\textsf{hub}} = 5$ (20 outliers).}
\label{subfig:link_hub_5}
\end{subfigure}
\caption{\textbf{Hubs}: Estimation of connection probabilities of inliers, for  different numbers of outliers (left: $s = 100$, middle: $s = 50$, right: $s=20$) corresponding to three signal to noise ratios (left: $\tau_{\textsf{hub}} = 1$, middle: $\tau_{\textsf{hub}} = 2$, right: $\tau_{\textsf{hub}} = 5$). For each of the three plots, we compare our package \texttt{gsbm} to two \href{https://CRAN.R-project.org/package=missSBM}{\texttt{missSBM}} \citep{2017Tabouy,tabouy2019misssbm} and \href{form https://CRAN.R-project.org/package=softImpute }{\texttt{softImpute}} in terms of the standardized MSE of estimation $\| \hat{\bL} - \bL^{*}\|_F^2/(n-s)^2$ (10 replications).}
\label{fig:link_hub}
\end{figure}

\begin{figure}
\centering
\begin{subfigure}[b]{0.31\textwidth}
\centering
\includegraphics[width=\textwidth]{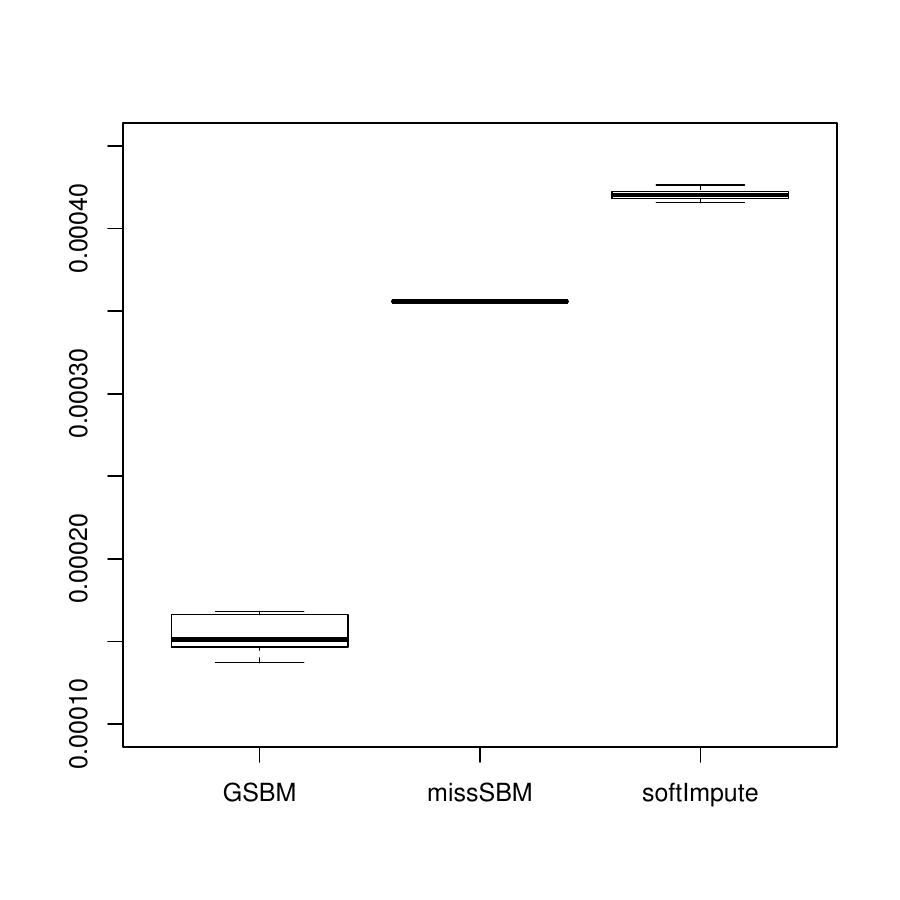}
\caption{$\tau_{\textsf{mix}} = 1$ (100 outliers)}
\label{subfig:link_mix_1}
\end{subfigure}
\hspace{0.2cm}
\begin{subfigure}[b]{0.31\textwidth}
\centering
\includegraphics[width=\textwidth]{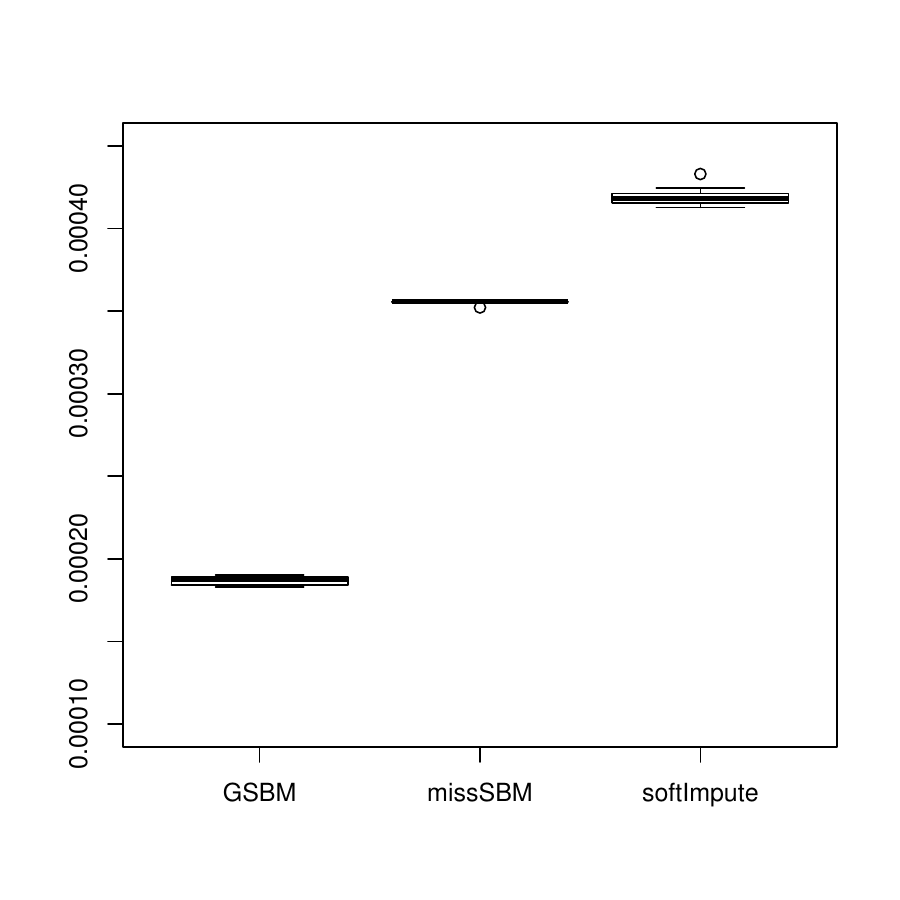}
\caption{$\tau_{\textsf{mix}} = 2$ (50 outliers).}
\label{subfig:link_mix_2}
\end{subfigure}
\hspace{0.2cm}
\begin{subfigure}[b]{0.31\textwidth}
\centering
\includegraphics[width=\textwidth]{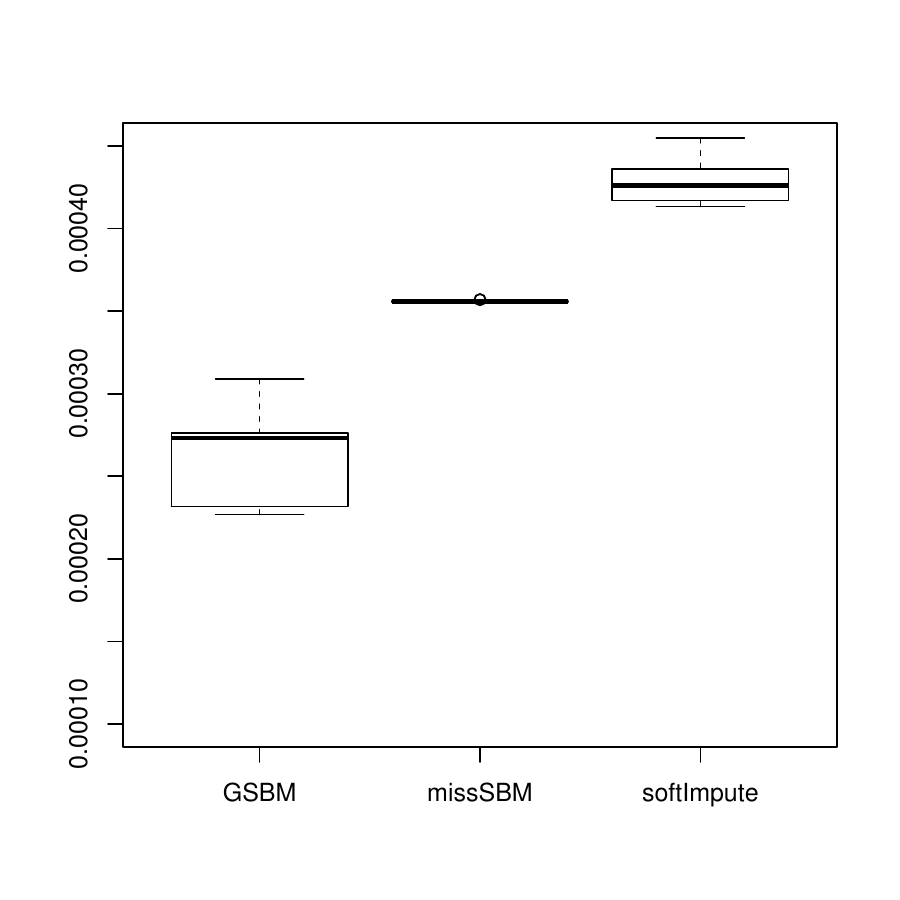}
\caption{$\tau_{\textsf{mix}} = 5$ (20 outliers).}
\label{subfig:link_mix_5}
\end{subfigure}
\caption{\textbf{Mixed membership}: Estimation of connection probabilities of inliers, for  different numbers of outliers (left: $s = 100$, middle: $s = 50$, right: $s=20$) corresponding to three signal to noise ratios (left: $\tau_{\textsf{mix}} = 1$, middle: $\tau_{\textsf{mix}} = 2$, right: $\tau_{\textsf{mix}} = 5$). For each of the three plots, we compare our package \texttt{gsbm} to two \href{https://CRAN.R-project.org/package=missSBM}{\texttt{missSBM}} \citep{2017Tabouy,tabouy2019misssbm} and \href{form https://CRAN.R-project.org/package=softImpute }{\texttt{softImpute}} in terms of the standardized MSE of estimation $\| \hat{\bL} - \bL^{*}\|_F^2/(n-s)^2$ (10 replications).}
\label{fig:link_mix}
\end{figure}

For the hubs (Figure \ref{fig:link_hub}), we observe that \href{https://CRAN.R-project.org/package=missSBM}{\texttt{missSBM}} and \href{form https://CRAN.R-project.org/package=softImpute }{\texttt{softImpute}} have similar estimation errors across all settings; overall \href{https://CRAN.R-project.org/package=missSBM}{\texttt{missSBM}} gives an MSE 20\% smaller than \href{form https://CRAN.R-project.org/package=softImpute }{\texttt{softImpute}}. For the large signal to noise ratio $\tau_{\textsf{hub}} = 5$ where outliers do not impair too much the estimation, our method \texttt{gsbm} gives similar results as \href{https://CRAN.R-project.org/package=missSBM}{\texttt{missSBM}}, but displays a larger variance. However, as the signal to noise ratio $\tau_{\textsf{hub}}$ decreases, i.e., in the settings where outliers severely challenge the estimation problem, our method \texttt{gsbm} improves over \href{https://CRAN.R-project.org/package=missSBM}{\texttt{missSBM}} by about 15\%. For the mixed membership outliers (Figure \ref{fig:link_mix}), we observe that our method \texttt{gsbm} consistently improves other methods by 30 to 50\%. As in the previous experiment with hubs, we observe that the improvement of \texttt{gsbm} over existing methods increases when the signal to noise ratio $\tau_{\textsf{mix}}$ decreases i.e., in the settings where outliers are the most challenging for the estimation of connection probabilities.

\subsection{Analysis of a contact network in a primary school}
\label{section:simul:polblogs}

Next, we apply our algorithm to analyse a network of contacts within a french elementary school, collected and analysed by the authors of \cite{PrimaryNetwork}, with the objective of better understanding the propagation of respiratory infections. The network records physical interactions occurring within a primary school between $226$ children divided into $10$ classes and their $10$ teachers over the course of a day; it was collected using a system of sensors worn by the participants. This system records the duration of interactions between two individuals facing each other at a maximum distance of one and a half metres. The duration of these interactions varies between $20$ seconds and two and a half hours. We consider that a physical interaction has been observed if the corresponding interaction duration is greater than one minute. If an interaction of less than one minute is observed, we consider that this observation may be erroneous, and treat the corresponding data as missing. We thus obtain a $236 \times 236$ adjacency matrix with $7054$ missing entries (including $236$ diagonal entries), and $4980$ entries equal to $1$ (corresponding to $2490$ observed undirected edges). The corresponding network is represented in Figure \ref{fig:primary}. 
\begin{figure}
\centering
\begin{subfigure}{.5\textwidth}
  \centering
    \includegraphics[scale=0.25]{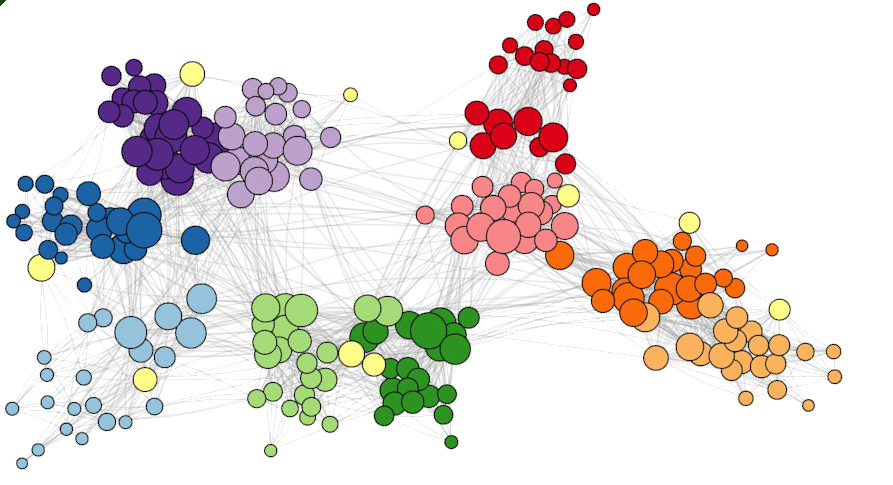}
\end{subfigure}%
\begin{subfigure}{.5\textwidth}
  \centering
    \includegraphics[scale=0.25]{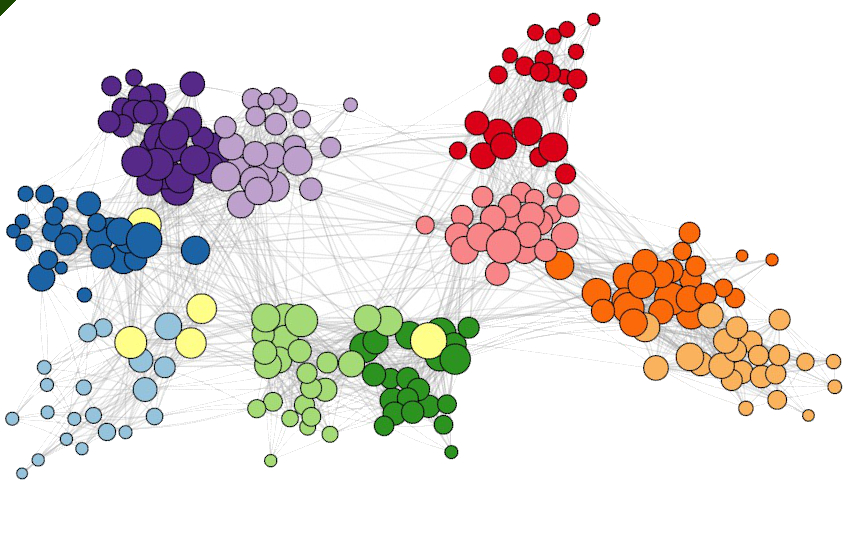}
\end{subfigure}
\caption{Network of interactions between children and teachers in a primary school over the course of a day. Only interactions lasting at least one minutes are represented. On the left, the colour of each node correspond to its class (if the node represent a child); teachers are indicated by yellow dots. On the right, the nodes are coloured according to estimated labels (the procedure for obtaining these estimated labels is described in Section \ref{section:simul:polblogs}). Yellow dots indicate nodes identified as outliers by MCGD. \label{fig:primary}}
\end{figure}
The analysis of the interactions network provides crucial information from an epidemiological point of view, as it can be used to model the transmission of respiratory-spread pathogens, and design strategies to mitigate the propagation of diseases \cite{close_class}. Interestingly, the interactions recorded in \cite{PrimaryNetwork} are strongly structured into communities, as pupils interact mostly with pupils from their class. They also interact with other pupils from the same level, although less frequently. They are scarcely connected with pupils from other age groups. Finally, we observe that pupils are on average connected to one teacher: each one of the ten classes interacts mostly with its teacher. By contrast, the teachers form a smaller group (there are $10$ teachers, while there are around $22.6$ pupils in each class); yet they do not form a cluster, as they are mostly connected to pupils from their class.\newline

The MCGD algorithm allows us to detect individuals with abnormal connectivities. We run this algorithm for a grid of values of $\lambda_1$ and $\lambda_2$. We notice that the set of nodes detected as outliers is stable when the parameters $(\lambda_1, \lambda_2)$ are chosen around $(8.5,8)$, and that it contains the nodes $15, 30, 25, 94$ and $180$. We also note that those nodes are detected as outliers significantly more frequently than the other nodes when the parameters vary. In the following, we consider estimators $(\widehat{\bL}, \widehat{\bS})$ obtained by running MCGD for this choice of parameters. In order to gain insights on the connectivity of these nodes, we compute the frequencies of their interactions with individuals from each class. Table \ref{tab:outlier_detection_primary} presents our findings. Nodes $30$, $35$ and $180$ belong to the class ``CE1 A", node $180$ belongs to the class ``CE1 B", and node $15$ belongs to the class ``CE2 B".
\begin{table}[ht]
\centering
\begin{tabular}{|c|c|c|c|c|c|c|c|c|c|c|c|}
\hline
& \textsc{ce1 a}& \textsc{ce1 b} & \textsc{ce2 a} & \textsc{ce2 b} & \textsc{cm1 a} & \textsc{cm1 b} & \textsc{cm2 a} & \textsc{cm2 b} & \textsc{cp a} & \textsc{cp b} & Teachers\\
\hline
\hline
\textsc{ce1 a} & 0.68 & 0.07 & 0.07 & 0.01 & 0.00 & 0.00 & 0.00 & 0.00 & 0.03 & 0.03 & 0.10\\
\hline
node $30$ & 0.93 & 0.30 & 0.23 & 0.08 & 0.00 & 0.00 & 0.05 & 0.00 & 0.12 & 0.15 &0.11\\
\hline
node $35$ & 0.62 & 0.24 & 0.35 & 0.00 & 0.06 & 0.00 & 0.00 & 0.05 & 0.13 & 0.22 & 0.11\\
\hline
node $180$ & 1.00 & 0.43 & 0.28 & 0.00 & 0.00 & 0.00 & 0.00 & 0.00 & 0.17 & 0.12 & 0.12\\
\hline
\hline
\textsc{ce1 b} & 0.07 & 0.83 & 0.02 & 0.00 & 0.00 & 0.01 & 0.00 & 0.00 & 0.05 & 0.08 & 0.11\\
\hline
node $94$ & 0.11 & 0.95 & 0.06 & 0.00 & 0.10 & 0.00 & 0.00 & 0.00 & 0.22 & 0.38 & 0.10\\
\hline
\hline
\textsc{ce2 b} & 0.01 & 0.00 & 0.21 & 0.94 & 0.05 & 0.03 & 0.00 & 0.03 & 0.02 & 0.00 & 0.13\\
\hline
node $15$ & 0.00 & 0.00 & 0.56 & 1.00 & 0.33 & 0.09 & 0.00 & 0.12 & 0.00 & 0.00 & 0.10\\
\hline
\end{tabular}
\caption{Frequency of contacts between either a node or an individual from a given class and other individuals from a given class. On average, a pupil from class ``CE1 A" has been in contact with a fraction $0.68$ of the remaining pupils from his class. By contrast, the node $30$, who is in class ``CE1 A", is connected with a fraction $0.93$ of the remaining pupils from his class.}
\label{tab:outlier_detection_primary}
\end{table}
We notice from Table \ref{tab:outlier_detection_primary} the pupil $30$ is $4$ times more connected to students from class ``CE1 B" and $5$ times more in contact with students from class ``CP B" than the other students from his class. Similarly, we observe that the children identified as outliers are strongly connected with different groups. Thus, the outliers detected by the method correspond to nodes with mixed membership. From a public health perspective, these children can potentially act as super-propagators, and contribute to spreading a virus from one group to the others.

Finally, we demonstrate that our estimator for the matrix of connection probabilities $\widehat{\bL}$ contains significant information on the structure of the network. More precisely, we show that the communities corresponding to the different classes can be recovered from this estimator. To do so, we consider the matrix whose columns contain the $10$ left singular vectors of $\widehat{\mathbf{L}}$, and we estimate the classes of the different nodes by running a $10$-means algorithm on its rows. This method recovers perfectly the classes of the children considered as ``inliers" (up to a permutation of the labels of the classes). While this method is not able to identify teachers, we note that teachers are mapped to the classes in a one-to-one fashion, which indicates that this method succeeds in assigning each class to its teacher. We represent the classes estimated by this method and the nodes identified as outliers in Figure \ref{fig:primary}.

\subsection{Analysis of a political Twitter network}

The ``\#Élysée2017fr'' data set, originally introduced in \citep{Elysee2017} provides data about 22,853 Twitter profiles active during the campaign of the French 2017 presidential election, from November 2016 to May 2017. Among other data, it contains a mentions network, where each node corresponds to a Twitter profile, and a directed edge (from mentioning profile to mentioned one) connects two profiles whenever one of them mentions the other in a Tweet. In total, this amounts to 1,896,262 edges. In the original study, the authors of \citep{Elysee2017} highlighted a community structure, where communities roughly correspond to affiliations to the 5 main political parties in France: France Insoumise (FI), Parti Socialiste (PS), Les Républicain (LR), La République en Marche (LREM), Rassemblement National (RN), with preferential attachment between nodes of the same political party. Detecting outliers in this network is of interest in order to detect, for example, influential figures.
We apply our algorithm to detect potential outliers to a subsample of this network containing the 10,000 most connected nodes; we also make the network undirected by drawing an edge between two nodes whenever one of them mentions the other. After subsampling and symmetrization of the adjacency matrix, the number of edges in the network is 1,562,419. Using the estimated theoretical values of the regularization parameters $\lambda_1$ and $\lambda_2$, we detect around $600$ outliers in the network. Inspecting the corresponding $600$ profile annotations, and node degrees we observe that the detected outliers correspond mainly to densely connected hubs or to mixed membership profiles (i.e. profiles affiliated to at least two political parties).
\paragraph{Hubs} First of all, we detect large hubs corresponding to main political figures and large media. The first detected outliers are the Twitter profiles of candidates to the election: Emmanuel Macron, Marine Le Pen, François Fillon, Jean-Luc Mélenchon, Benoît Hamon, Nicolas Dupont-Aignan. Other detected private personalities include journalists, deputies and senators (Jean-Jacques Bourdin, Alexis Corbière, Benjamin Griveaux, Yannick Jadot, Richard Ferrand, Éric Ciotti, etc.). Secondly, we detect the Twitter profiles of high-circulation media: BFM TV, Le Figaro, Le Monde, Libération, Mediapart, France Info, Europe 1, France Inter, etc. We also note that some hubs correspond to online, unofficial political groups (\symbol{64}TeamProgressist, \symbol{64}ForceRep\_fr, \symbol{64}Presse2Droite, \symbol{64}nomacron, etc.).
\paragraph{Mixed membership} We also detect Twitter profiles corresponding to nodes of mixed membership affiliated to multiple parties. Some of these nodes also correspond to smaller hubs, such as Christine Boutin (LR/RN) and La Manif Pour Tous (LR/RN); they have smaller degrees than the main political figures and media (degree around 1000 rather than $>$5000 for the main hubs). We also find mixed membership profiles corresponding to individual profiles with no public exposition (e.g. \symbol{64}mrericmas: LREM/LR, \symbol{64}erayeye: LR/RN, \symbol{64}Apostillier1: LREM/PS, etc.). After inspecting the Twitter profiles, these seem to be individuals sharing their own political opinions on Twitter, which would not necessarily be detected by checking only the histogram of degrees.

\section{Conclusion}

In this paper, we have proposed a new, computationally efficient algorithm for detecting nodes with anomalous connection patterns. This algorithm, which is robust against missing observations, allows for simultaneous estimation of the probabilities of connections of the remaining, normal nodes. A convergence analysis of this algorithm is provided, which proves that this algorithm converges at a sub-linear rate. Moreover our simulations studies indicate that its running time remains moderate, even for networks containing a few thousands of nodes. 
Our theoretical results show that our method detects exactly the outliers under fairly general assumptions. Moreover, our estimator for the probabilities of connections achieves the best known error rate among estimator with polynomial running time. These results are supported by simulation studies which demonstrate the good properties of our estimators in terms of both outliers detection and link prediction. Finally, we have exemplified how our method can be used to detect outlier nodes  and recover structural information on the remaining nodes in real world networks, by applying this method to "Les Misérables" characters network, as well as a network of interactions taking place in a primary school, and on a political Twitter network.
The results of the present paper pave the way to several extensions, which are of interest in applications. In particular, an important one would be the generalisation of the model to dynamic networks, where the adjacency matrix is observed at multiple time points, and connections, outliers, and possibly underlying communities are allowed to vary across time. This is interesting in applied problems where the outliers have characteristic dynamic behaviour. For instance, to detect fake news in social networks where, contrary to regular users, malicious users tend to have very unstable connectivity patterns across time.


\bibliographystyle{apalike}
\bibliography{ref_2019}

\appendix

\section{Proofs}
\label{sec:proofs}
The proofs are presented as follows. First, we recall in Section \ref{subsubsec:Tools} some results that will be used in our proofs. In Section \ref{details_algo} we provide the details of the Algorithm \ref{algo}.  Section \ref{proof:cvg} is devoted to the study of the convergence of our algorithm. Theorem \ref{thm:inliers_detection} is proved in Section \ref{subsec:Proof_inliers_detection}, Theorem \ref{thm:outliers_detection} is proved in Section \ref{proof:outliers_detection}, while in Section \ref{subsec:ProofTh4} we prove Theorem \ref{thm:Borne_norme_2}. Corollary \ref{corollary_fixed_lambda} is proved in Sections \ref{subsec:ProofCor1}. Auxiliary Lemmas used throughout these sections are proved in Section \ref{subsec:ProofAuxRes}.

To ease notations, we denote henceforth by $\Delta \bS = \bS^* - \widehat{\bS}$ and $\Delta \bL = \bL^* - \widehat{\bL}$ the estimation errors of $\bS^*$ and $\bL^*$.

\subsection{Tools}\label{subsubsec:Tools}
In our proofs, we will use Bernstein's inequality on different occasions. We state it here for the reader's convenience.

\begin{thm}[Bernstein's inequality]\label{thm:Bernstein}
Let $X_1, ..., X_n$ be independent centered random variables. Assume that for any $i \in [n]$,  $\vert X_i\vert \leq M$ almost surely, then
\begin{equation}
\mathbb{P}\left(\left \vert \sum_{1 \leq i \leq n}X_i \right \vert \geq  \sqrt{2t\sum_{1 \leq i \leq n}\mathbb{E}[X_i^2]} + \frac{2M}{3}t\right) \leq 2e^{-t} \label{eq:Bernstein}
\end{equation}
\end{thm}
\noindent We will also use Bousquet's theorem, as stated in  \cite{GineNickl}, Theorem 3.3.16.
\begin{thm}[Bousquet] \label{thm:Bousquet}
Let $X_i$, $i \in \mathbb{N}$ be independent $\mathcal{S}$-valued random variables, and let $\mathcal{F}$ be a countable class of functions $f = (f_1,..., f_n): \mathcal{S} \rightarrow [-1,1]^n$ such that $\mathbb{E}[f_i(X_i)] = 0$ for any $f \in \mathcal{F}$ and $i \in [n]$. Set
$Z = \underset{f \in \mathcal{F}}{\sup} \left \vert \underset{1 \leq i \leq n}{\sum} f_i(X_i) \right \vert$ and $v = \underset{f \in \mathcal{F}}{\sup} \underset{1\leq i \leq n}{\sum} \mathbb{E} \left[f_i(X_i)^2\right]$. Then, for any $x>0$,
$$\mathbb{P} \left(Z > \mathbb{E}[Z] + \frac{x}{3} + \sqrt{2x(2\mathbb{E}[Z] + v)} \right) \leq \exp(-x).$$
\end{thm}
\noindent To bound the operator norm of random matrices with high probability, we use Corollary 3.6 in \cite{BandeiraHandel}.
\begin{prp}[Bandeira, Van Handel, 2016]\label{thm:OperatorNorm}
Let $\bX$ be a $n \times n$ symmetric random matrix with $\bX_{ij} = \xi_{ij}b_{ij}$, where $\{\xi_{ij}\}_{i\leq j}$ are independant symmetric random variables with unit variance, and $\{b_{ij}\}_{i \leq j}$ are fixed scalars. Let $\sigma \triangleq \underset{i}{\max}\sqrt{\sum_j b_{ij}^2},$ then for any $\alpha \geq 3$
$$ \mathbb{E} \left[ \left \Vert \bX \right \Vert_{op} \right] \leq e^{\frac{2}{3}}\left(2 \sigma + 14 \alpha \underset{ij}{\max} \left(\mathbb{E}\left[\left(\xi_{ij}b_{ij} \right)^{2 \alpha  }\right]\right)^{\frac{1}{2 \alpha  }}\sqrt{\log(n) }\right).$$
\end{prp}
\noindent The following  high-probability bound on the spectral norm of a random matrix is based on Remark 3.13 in \cite{BandeiraHandel}. This remark provides a bound up to an unspecified absolute constant. In order to make this constant explicit, we follow the lines of the proof of this remark, and we combine Theorem 6.10 in \cite{boucheron2013concentration}, Proposition \ref{thm:OperatorNorm}, and a symetrization argument (see, e.g., Corollary 3.3 in \cite{BandeiraHandel}) to  obtain the following proposition.

\begin{prp}\label{thm:Bandeira}
Let $\bX$ be an $n\times n$ symmetric matrix with 
$\bX_{ij} = \xi_{ij}b_{ij}$, where $\{\xi_{ij}\}_{i\leq j}$ are independent centered random variables with unit variance, and $\{b_{ij}\}_{i \leq j}$ are fixed scalars. Then for every $t \geq 0$ and every $\alpha \geq 3$, 
\begin{equation*}
    \mathbb{P}\left(\left \Vert\bX \right \Vert_{op} \geq 2e^{\frac{2}{3}}\left(2 \sigma + 14 \alpha \underset{ij}{\max} \left(\mathbb{E}\left[\left(\xi_{ij}b_{ij} \right)^{2 \alpha  }\right]\right)^{\frac{1}{2 \alpha  }}\sqrt{\log(n) }\right) +t \right)\leq e^{-t^2/2\tilde{\sigma^*}^2}
\end{equation*}
where we have defined $\tilde{\sigma^*} \triangleq \underset{ij}{\max}\left \vert \bX_{ij}\right \vert$ and $\sigma \triangleq \underset{i}{\max}\sqrt{\sum_j b_{ij}^2}$.
\end{prp}
\begin{proof}

To prove the desired high-probability bound, we first bound the expectation of the spectral norm, using the same symmetrization trick as in Corollary 3.3 in \cite{BandeiraHandel}. Let $\bX^{'}$ be an independent copy of the random matrix $\bX$, and let $\bY$ be the symmetric matrix with random entries defined as $\bY_{ij} \triangleq \bX_{ij} - \bX_{ij}^{'}$ for any $(i,j)\in [n]\times [n]$. Note that, for any $(i,j) \in [n] \times [n]$, $i<j$, $\bY_{ij} = \sqrt{2}b_{ij}\times \left(\xi_{ij} - \xi_{ij}^{'} \right)/\sqrt{2}$, where $\bxi_{ij}$ are independent copies of $\xi_{ij}$, and $\left(\xi_{ij} - \xi_{ij}^{'} \right)/\sqrt{2}$ are symmetric random variable with unit variance. Applying Proposition \ref{thm:OperatorNorm}, we find that 
\begin{equation}
\mathbb{E} \left[ {\left\Vert \bY \right\Vert}_{op} \right] \leq e^{\frac{2}{3}}\left(2 \sigma_Y + 14 \alpha \underset{ij}{\max} \left(\mathbb{E}\left[\left(\left(\xi_{ij}- \xi_{ij}^{'}\right)b_{ij} \right)^{2 \alpha  }\right]\right)^{\frac{1}{2\alpha}}\sqrt{\log(n) }\right)\nonumber
\end{equation}
with $\sigma_Y \triangleq \underset{i}{\max} \sqrt{\sum_j 2b_{ij}^2} = \sqrt{2} \sigma$. Moreover for any $(i,j) \in [n] \times [n]$, $\left(\mathbb{E}\left[\left(\left(\xi_{ij}- \xi_{ij}^{'}\right)b_{ij} \right)^{2 \alpha  }\right]\right)^{\frac{1}{2 \alpha}}\leq 2\left(\mathbb{E}\left[\left(\xi_{ij}b_{ij}\right)^{2 \alpha}\right]\right)^{\frac{1}{2\alpha}}$. Recall that $\bX$ is centered.
Then, by Jensen inequality, $\mathbb{E}\left[\left \Vert \bX\right \Vert_{op} \right] =\mathbb{E}\left[\left \Vert \bX - \mathbb{E}\left[ \bX\right]\right \Vert_{op} \right] \leq \mathbb{E}\left[\left \Vert \bX - \bX^{'}\right \Vert_{op} \right] = \mathbb{E}\left[\left \Vert \bY\right \Vert_{op} \right]$. Hence,
\begin{equation}
\mathbb{E} \left[ \left \Vert \bX \right \Vert_{op} \right] \leq 2e^{\frac{2}{3}}\left(2\sigma + 14 \alpha \underset{ij}{\max} \left(\mathbb{E}\left[\left(\xi_{ij}b_{ij}\right) ^{2 \alpha  }\right]\right)^{\frac{1}{2 \alpha  }}\sqrt{\log(n) }\right).\label{eq:opNormSym}
\end{equation}
Then, we use Talagrand's concentration inequality (see \cite{boucheron2013concentration}, Theorem 6.10) and find that for any $t >0$,
\begin{equation}
    \mathbb{P}\left[\left\Vert \bX\right\Vert_{op} \geq \mathbb{E}\left\Vert \bX\right\Vert_{op} + t \right] \leq e^{\frac{-t^2}{2\tilde{\sigma}^*}}\label{eq:Talagrand}
\end{equation}
Combining equations \eqref{eq:opNormSym} and \eqref{eq:Talagrand} yields the desired result.
\end{proof}
\subsection{Mixed coordinate gradient descent algorithm}\label{details_algo}
Below, we describe the details of our algorithm.
At iteration $t=0$, we initialize the parameters $(\bS^{(0)}, \bL^{(0)}, R^{(0)})$; then, at iteration $t\geq 1$, we start by updating $\bS$. Denote by $\bG^{(t-1)}_S = -2\Omega\odot(\bA-\bL^{(t-1)}-\bS^{(t-1)}-(\bS^{(t-1)})^{\top})+\epsilon\bS^{(t-1)}$ the gradient with respect to $\bS$ of the quadratic part of the objective function, evaluated at $(\bS^{(t-1)}, \bL^{(t-1)})$. The column-wise sparse component $\bS$ is updated with a proximal gradient step:
\begin{equation}
\label{eq:Supdate}
\begin{array}{ll}
\bS^{(t)} & \in \operatorname{argmin}\left(\eta\lambda_2 \left \Vert\bS \right \Vert_{2,1} + \frac{1}{2}\left \Vert \bS - \bS^{(t-1)} + \eta\bG_S^{(t-1)} \right \Vert_{F}^2 \right),\\
& = \mathsf{Tc}_{\eta\lambda_2}\left(\bS^{(t-1)} -\eta \bG_S^{(t-1)}\right),
\end{array}
\end{equation}
where $\mathsf{Tc}_{\eta\lambda_2}$ is the column-wise soft-thresholding operator such that for any $\Mprm \in \mathbb{R}^{n\times n}$ and for any $\lambda>0$, the $j$-th column of $\mathsf{Tc}_{\lambda}(\Mprm) $ is given by $(1-\lambda/\norm{\Mprm_{.,j}}[2])_+\Mprm_{.,j}$. The step size $\eta$ is constant and fixed in advance, and satisfies $\eta \leq 1/(2+\epsilon)$. Then, we compute the adaptive upper bound $\bar R^{(t)}$ as follows:
\begin{equation}
\label{eq:Rupperbound}
\bar R^{(t)} = \lambda_1^{-1}\Phi_{\epsilon}(\bS^{(t)}, \bL^{(t-1)}, R^{(t-1)}).
\end{equation}
Note that, by definition:
\begin{eqnarray*}
\label{eq:upper-bound}
\Phi_{\epsilon}(\bS^{(t-1)}, \bL^{(t-1)}, R^{(t-1)}) &\geq& \Phi_{\epsilon}(\hat{\bS}_{\epsilon}, \hat{\bL}_{\epsilon}, \hat R) \nonumber\\
&=& \frac{1}{2}\norm{\Omega\odot(\bA - \hat{\bL}_{\epsilon} - \hat{\bS}_{\epsilon} - (\hat{\bS}_{\epsilon})^{\top})}[F]^2
+ \lambda_1\norm{\hat{\bL}_{\epsilon}}[*] + \lambda_2\norm{\hat{\bS}_{\epsilon}}[2,1]\\
& & \quad+ \frac{\epsilon}{2}(\norm{\hat{\bL}_{\epsilon}}[F]^2 + \norm{\hat{\bS}_{\epsilon}}[F]^2)\nonumber\\
&\geq& \lambda_1\norm{\hat{\bL}_{\epsilon}}[*],
\end{eqnarray*} 
since every term in the objective function is non-negative. As a result, we obtain that $$\norm{\hat{\bL}_{\epsilon}}[*]\leq \lambda_1^{-1}\Phi_{\epsilon}(\bS^{(t-1)}, \bL^{(t-1)}, R^{(t-1)}),$$ and we get the upper bound \eqref{eq:Rupperbound}.
Finally, the low-rank component given by $(\bL,R)$ is updated using a conjugate gradient step as follows:
\begin{equation}
\label{eq-Lupdate}
\left(\bL^{(t)}, R^{(t)} \right) = \left(\bL^{(t-1)}, R^{(t-1)} \right) + \beta_t\left(\tilde\bL^{(t)}-\bL^{(t-1)}, \tilde R^{(t)}-R^{(t-1)} \right),
\end{equation}
where $\beta_t \in [0,1]$ is a step size defined later on. Denote by $\bG^{(t-1)}_L = -\Omega\odot(\bA-\bL^{(t-1)}-\bS^{(t)}-(\bS^{(t)})^{\top})+\epsilon\bL^{(t-1)}$ the gradient with respect to $\bL$ of the quadratic part of the objective function, evaluated at $(\bS^{(t)}, \bL^{(t-1)})$. The direction $(\tilde\bL^{(t)}, \tilde R^{(t)})$ is defined by:
\begin{equation}
\label{eq:Lprm}
\begin{array}{rl}
\left(\tilde\bL^{(t)}, \tilde R^{(t)} \right) \in& \operatorname{argmin}_{\Zprm, R}\quad \pscal{\Zprm}{\bG_L^{(t-1)}} + \lambda_1R\\
\text{such that} & \norm{\Zprm}[*] \leq R \leq \bar R^{(t)}.
\end{array}
\end{equation}
Let $\sigma_1$ be the largest singular value of the gradient matrix $\bG_L^{(t-1)}$, and let $u_1$ and $v_1$ be the corresponding left and right singular vectors. Then, \eqref{eq:Lprm} admits the following closed-form solution:
\begin{equation}
\label{eq:L-update}
\left(\tilde\bL^{(t)}, \tilde R^{(t)} \right) =\left\{\begin{matrix}
(\mathbf{0},0) & \text{if } \lambda_1\geq \sigma_1 \\ 
(-\bar R^{(t)}u_1v_1^{\top}, \bar R^{(t)}) & \text{if } \lambda_1 < \sigma_1.
\end{matrix}\right.
\end{equation}

The step size $\beta_t$ is set to:
\begin{equation}
\label{eq:step-size}
\beta_t = \min\left\{1, 
\frac{\pscal{\bL^{(t-1)}-\tilde\bL^{(t)}}{\bG_L^{(t-1)}}+\lambda_1(R^{(t-1)}-\tilde R^{(t)})}{(1+\epsilon)\norm{\tilde{\bL}^{(t)}-\bL^{(t-1)}}[F]^2}\right\}.
\end{equation}
We show in appendix \ref{proof:cvg} that this choice of step size ensures that the objective function decreases at every iteration.
The above steps are repeated iteratively until convergence, or for a predefined number of iterations. In practice, we stop the algorithm when the relative decrease of the objective falls below a predefined threshold (e.g., 10e-6).
\subsection{Proof of \thmref{cvg}}
\label{proof:cvg}
To prove \thmref{cvg}, we proceed in three steps. First, we demonstrate that the objective function decreases after every update of $\bS$ or $\bL$. In a second step, we compute a lower bound on the amount by which the objective function decreases at each iteration. In a third step, we use this lower bound to demonstrate that the distance to the optimal solution at iteration $t\geq 1$, $\Delta^t = \Phi_{\epsilon}(\bS^{(t)}, \bL^{(t-1)}, R^{(t-1)}) - \Phi_{\epsilon}(\hat{\bS}, \hat\bL, \hat{R})$, decreases at a rate of the order of $1/t$.

\paragraph{Decrease of the objective between successive iterations:}
We start by showing that the proximal update for the $\Sprm$ block yields a decrease of the objective. For $t\geq 1$, denote $Q^{(t-1)} = \lambda_2^{-1}\Phi_{\epsilon}(\Sprm^{(t-1)}, \bL^{(t-1)}, R^{(t-1)})$, and 
\begin{eqnarray}
\label{eq:gS}
g_S(\Sprm^{(t-1)}, \bL^{(t-1)}) = \pscal{\bG_S(\bS^{(t-1)},\bL^{(t-1)})}{\Sprm^{(t-1)}-\tilde{\Sprm}^{(t-1)}}
+ \lambda_2(\norm{\Sprm^{(t-1)}}[2,1]-\norm{\tilde{\Sprm}^{(t-1)}}[2,1]).
\end{eqnarray}
In \eqref{eq:gS}, $\bG_S(\bS^{(t-1)},\bL^{(t-1)}) = -2\Omega\odot(\bA-\bL^{(t-1)}-\Sprm^{(t-1)}-(\Sprm^{(t-1)})^{\top})+\epsilon\Sprm^{(t-1)}$ is the gradient matrix with respect to $\Sprm$ of the quadratic part of the objective function, evaluated at $(\Sprm^{(t-1)}, \bL^{(t-1)})$, and
\begin{equation*}
\label{eq:Stilde}
\tilde{\Sprm}^{(t-1)} = \argmin_{\Sprm}\quad \pscal{\bG_S(\bS^{(t-1)},\bL^{(t-1)})}{\Sprm} + \lambda_2\norm{\Sprm}[2,1]\quad \text{s.t. } \norm{\Sprm}[2,1]\leq Q^{(t-1)}.
\end{equation*}
\begin{lem}
\label{lem:S-decrease}
For $t\geq 1$, the proximal update for the $\Sprm$ block defined in \eqref{eq:Supdate} satisfies:
\begin{eqnarray*}
\label{eq:S-decrease}
\Phi_{\epsilon}(\Sprm^{(t)}, \bL^{(t-1)}, R^{(t-1)}) \leq
\Phi_{\epsilon}(\Sprm^{(t-1)}, \bL^{(t-1)}, R^{(t-1)}) - \frac{\eta}{2}\frac{g_S^2(\Sprm^{(t-1)}, \bL^{(t-1)})^2}{(2Q^{(t-1)})}.
\end{eqnarray*}
\end{lem}
\begin{proof}
See Section~\ref{proof:lem:S-decrease}.
\end{proof}

We now prove a similar result, this time concerning the $(\bL, R)$ block update. Recall that, for $t\geq 1$, $\bar{R}^{(t)} = \lambda_1^{-1}\Phi_{\epsilon}(\Sprm^{(t)}, \bL^{(t-1)}, R^{(t-1)})$. 
\begin{eqnarray}
\label{eq:gL}
g_L(\Sprm^{(t)}, \bL^{(t-1)}, R^{(t-1)}) = \pscal{\bG_L(\bS^{(t)},\bL^{(t-1)})}{\bL^{(t-1)}-\tilde{\bL}^{(t-1)}}
+ \lambda_1(R^{(t-1)} - \tilde{R}^{(t-1)}).
\end{eqnarray}
In \eqref{eq:gL}, $\bG_L(\bS^{(t)},\bL^{(t-1)}) = -\Omega\odot(\bA-\bL^{(t-1)}-\Sprm^{(t)}-(\Sprm^{(t)})^{\top})+\epsilon\bL^{(t-1)}$ is the gradient matrix with respect to $\bL$ of the quadratic part of the objective function, evaluated at $(\Sprm^{(t)}, \bL^{(t-1)})$. Recall that $M^{(t)} = \norm{\bG_L(\bS^{(t)},\bL^{(t-1)})}[F]$. We prove the following result, which ensures a decrease of the objective function after the conditional gradient update.
\begin{lem}
\label{lem:L-decrease}
For $t\geq 1$, the conditional gradient update for the $(\bL, R)$ block defined in \eqref{eq:L-update} satisfies:
\begin{eqnarray*}
\label{eq:L-decrease}
\Phi_{\epsilon}(\Sprm^{(t)}, \bL^{(t)}, R^{(t)})-\Phi_{\epsilon}(\Sprm^{(t)}, \bL^{(t-1)}, R^{(t-1)}) \leq
 -\frac{g_L^2(\Sprm^{(t)}, \bL^{(t-1)}, R^{(t-1)})}{\max\{2\bar{R}^{(t)}(\lambda_1+M^{(t)}), 8(1+\epsilon)(\bar{R}^{(t)})^2\}}.
\end{eqnarray*}
Moreover,
\begin{eqnarray}
\label{eq:L-decrease-v2}
\Phi_{\epsilon}(\Sprm^{(t)}, \bL^{(t)}, R^{(t)}) - \Phi_{\epsilon}(\Sprm^{(t)}, \bL^{(t-1)}, R^{(t-1)})\leq 
 -\frac{(1+\epsilon)}{2}\norm{\bL^{(t)} - \bL^{(t-1)}}[F]^2.
\end{eqnarray}
\end{lem}
\begin{proof}
See Section~\ref{proof:lem:L-decrease}.
\end{proof}

\paragraph{Lower bound on the decrement $\Phi_{\epsilon}(\bS^{(t)}, \bL^{(t-1)}, R^{(t-1)})-\Phi_{\epsilon}(\bS^{(t+1)}, \bL^{(t)}, R^{(t)})$:}
Consider the function
\begin{equation*}
\label{eq:lyapunov}
g^t(Q^{(t)}, \bar{R}^{(t)}) \eqdef g_S(\bS^{(t)}, \bL^{(t-1)}) + g_L(\bS^{(t)}, \bL^{(t-1)}, \bar{R}^{(t-1)}).
\end{equation*}
In what follows, we compute upper and lower bounds on $g^t(Q^{(t)}, \bar{R}^{(t)})$.
Note that $g^t(Q^{(t)}, \bar{R}^{(t)})$ depends on $(Q^{(t)}, \bar{R}^{(t)})$, because computing $g_S$ and $g_L$ involve solving constrained optimization problems, which depend on $Q^{(t)}$ and $\bar{R}^{(t)}$, respectively. By convexity of the quadratic term $\norm{\Omega\odot(\bA-\bL-\bS-\bS^{\top})}[F]^2/2+\epsilon/2(\norm{\bL}[F]^2+\norm{\bS}[F]^2)$, we obtain that:
$$g^t(Q^{(t)}, \bar{R}^{(t)})\geq \Phi_{\epsilon}(\bS^{(t)}, \bL^{(t-1)}, R^{(t-1)}) -  \Phi_{\epsilon}(\tilde{\bS}^{(t)}, \tilde{\bL}^{(t-1)}, \tilde{R}^{(t-1)}).$$
Then, by definition of the minimizer $(\hat{\bS}_{\epsilon}, \hat{\bL}_{\epsilon}, \hat{R})$:
\begin{equation}
\label{eq:lyap-lwbd}
g^t(Q^{(t)}, \bar{R}^{(t)})\geq \Phi_{\epsilon}(\bS^{(t)}, \bL^{(t-1)}, R^{(t-1)}) -  \Phi_{\epsilon}(\hat{\bS}_{\epsilon}, \hat{\bL}_{\epsilon}, \hat{R}),
\end{equation}
which gives the lower bound on $g^t(Q^{(t)}, \bar{R}^{(t)})$. \newline

Let us now compute an upper bound for $g^t(Q^{(t)}, \bar{R}^{(t)})$. To do so, we start by upper bounding $g_S(\bS^{(t)}, \bL^{(t-1)})$ defined in \eqref{eq:gS}. By definition,
\begin{eqnarray*}
\label{eq:gS-upbd}
g_S(\Sprm^{(t-1)}, \bL^{(t-1)}) &=&
 \max_{\norm{\bS}[2,1]\leq Q^{(t)}}\{\pscal{\bG_S(\bS^{(t)},\bL^{(t-1)})}{\bS^{(t)}-\bS}
+ \lambda_2(\norm{\Sprm^{(t)}}[2,1]-\norm{\bS}[2,1])\}\nonumber\\
&=& \max_{\norm{\bS}[2,1]\leq Q^{(t)}}\{\pscal{\bG_S(\bS^{(t)},\bL^{(t)})}{\bS^{(t)}-\bS}\nonumber\\ &&+ \pscal{\bG_S(\bS^{(t)},\bL^{(t-1)})-\bG_S(\bS^{(t)},\bL^{(t)})}{\bS^{(t)}-\bS}
+ \lambda_2(\norm{\Sprm^{(t)}}[2,1]-\norm{\bS}[2,1])\}\nonumber\\
&\leq& \max_{\norm{\bS}[2,1]\leq Q^{(t)}}\Big\{\pscal{\bG_S(\bS^{(t)},\bL^{(t)})}{\bS^{(t)}-\bS} + \lambda_2(\norm{\Sprm^{(t)}}[2,1]-\norm{\bS}[2,1])\nonumber\\
& &+ \norm{\bG_S(\bS^{(t)},\bL^{(t-1)})-\bG_S(\bS^{(t)},\bL^{(t)})}[F]\norm{\bS^{(t)}-\bS}[F]\Big\}\nonumber.\\
& & \leq \underbrace{\pscal{\bG_S(\bS^{(t)},\bL^{(t)})}{\bS^{(t)}} + \lambda_2\norm{\Sprm^{(t)}}[2,1] - \min_{\norm{\bS}[2,1]\leq Q^{(t)}}\left\{ \pscal{\bG_S(\bS^{(t)},\bL^{(t)})}{\bS} + \lambda_2\norm{\bS}[2,1]\right\}}_{I}\nonumber\\
& &+\underbrace{\max_{\norm{\bS}[2,1]\leq Q^{(t)}}\left\{ \norm{\bG_S(\bS^{(t)},\bL^{(t-1)})-\bG_S(\bS^{(t)},\bL^{(t)})}[F]\norm{\bS^{(t)}-\bS}[F]\right\}}_{II} 
\end{eqnarray*}
On the one hand, by definition of $\tilde\bS^{(t)}$ and $g_S(\bS^{(t)}, \bL^{(t)})$ (see \eqref{eq:gS} and \eqref{eq:Stilde}), we have:
\begin{eqnarray}
     \label{eq:gS-upbd-I}
     I\leq g_S(\bS^{(t)}, \bL^{(t)}).
\end{eqnarray}
On the other hand, by definition of $Q^{(t)}$, $\norm{\bS^{(t)}}[2,1]\leq Q^{(t)}$, which implies $\norm{\bS^{(t)}}[F]\leq Q^{(t)}$; combined with $\norm{\bS}[F]\leq Q^{(t)}$, we obtain that that  $\norm{\bS^{(t)}-\bS}[F]\leq 2Q^{(t)}$. Note also that, as the gradient $G_S$ is $(1+\epsilon)$-Lipschitz, we have $\norm{\bG_S(\bS^{(t)},\bL^{(t-1)})-\bG_S(\bS^{(t)},\bL^{(t)})}[F]\leq (1+\epsilon)\norm{\bL^{(t-1)}-\bL^{(t)}}[F]$. Finally we obtain:
\begin{eqnarray}
     \label{eq:gS-upbd-II}
     II\leq 2Q^{(t)}(1+\epsilon)\norm{\bL^{(t-1)}-\bL^{(t)}}[F].
\end{eqnarray}
Combining \eqref{eq:gS-upbd-I} and \eqref{eq:gS-upbd-II}, we finally obtain:
\begin{eqnarray}
     \label{eq:gS-upbd-final}
     g_S(\Sprm^{(t-1)}, \bL^{(t-1)})\leq g_S(\bS^{(t)}, \bL^{(t)})+2Q^{(t)}(1+\epsilon)\norm{\bL^{(t-1)}-\bL^{(t)}}[F].
\end{eqnarray}
We now use \eqref{eq:gS-upbd-final} to derive our upper bound on $g^t(Q^{(t)}, \bar{R}^{(t)})$ as follows. Using Lemma \ref{lem:S-decrease} and Lemma \ref{lem:L-decrease}, we obtain that:
\begin{eqnarray}
(g^{(t)}(Q^{(t)}, \bar{R}^{(t)}))^2
&\leq& 2\Big\{g_L^2(\bS^{(t)}, \bL^{(t-1)}, R^{(t-1)})+ g_S^2(\bS^{(t)}, \bL^{(t)})
+4(Q^{(t)})^2(1+\epsilon)^2\norm{\bL^{(t-1)}-\bL^{(t)}}[F]^2\Big\}\nonumber\\
&\leq& 2\Big\{(C_1^{(t)}+C_3^{(t)})(\Phi_{\epsilon}(\bS^{(t)}, \bL^{(t-1)}, R^{(t-1)}) - \Phi_{\epsilon}(\bS^{(t)}, \bL^{(t)}, R^{(t)}))\nonumber\\
&&+ C_2^{(t)} (\Phi_{\epsilon}(\bS^{(t)}, \bL^{(t)}, R^{(t)}) - \Phi_{\epsilon}(\bS^{(t+1)}, \bL^{(t)}, R^{(t)}))\Big\},\nonumber
\end{eqnarray}
where
\begin{equation}
C_1^{(t)}  = \max\{2\bar{R}^{(t)}(\lambda_1+M^{(t)}), 8(1+\epsilon)(\bar{R}^{(t)})^2\},\quad
C_2^{(t)}  = \frac{8(Q^{(t)})^2}{\eta},\quad
C_3^{(t)}  = 8(1+\epsilon)(Q^{(t)})^2.\nonumber
\end{equation}
Define:
\begin{equation}
\label{eq:Ct}
C^{(t)} = 2\max\{C_1^{(t)}+C_3^{(t)}, C_2^{(t)}\}.
\end{equation}
We finally have the following lower bound:
\begin{equation}
(g^{(t)}(Q^{(t)}, \bar{R}^{(t)}))^2 \leq C^{(t)}(\Phi_{\epsilon}(\bS^{(t)}, \bL^{(t-1)}, R^{(t-1)})-\Phi_{\epsilon}(\bS^{(t+1)}, \bL^{(t)}, R^{(t)})).\nonumber
\end{equation}

\paragraph{Convergence rate of order $1/t$:}
Recall that $\Delta^t := \Phi_{\epsilon}(\bS^{(t)}, \bL^{(t-1)}, R^{(t-1)}) - \Phi_{\epsilon}(\hat{\bS}_{\epsilon}, \hat\bL_{\epsilon}, \hat{R})$. Using the fact that $$(g^{(t)}(Q^{(t)}, \bar{R}^{(t)}))^2\geq (\Delta^t)^2,$$ proven in \eqref{eq:lyap-lwbd}, we obtain that
\begin{equation*}
\label{eq:suite}
\Delta^{t+1}\leq \Delta^t - \frac{1}{C^{(t)}}(\Delta^t)^2.
\end{equation*}
We use the following Lemma (see, e.g. \cite[Lemma 3.5]{Beck2013}, \cite[Lemma 8]{Robin:2018:LIS}).
\begin{lem}
\label{lem:suite}
Let $\{A_k\}_{k\geq 1}$ be a non-negative sequence satisfying:
$$A_{k+1} \leq A_k - \gamma_kA_k^2, k\geq 1,$$
where $\gamma_k>0$ for any $k\geq 1$. Then,
$$A_{k+1}\leq \frac{1}{\frac{1}{A_1} +\sum_{i=1}^k\gamma_i}.$$
\end{lem}
\begin{proof}
See Section \ref{subsec:lem:suite:prf}
\end{proof}
Lemma \ref{lem:suite} yields that:
\begin{equation*}
\label{eq:suite-v2}
\Delta^{t+1}\leq \frac{1}{(\Delta^1)^{-1} + \sum_{i=1}^t\frac{1}{C^{(i)}}}.
\end{equation*}
noting that $\Delta^{1}\leq \tilde{\Delta}^0:=\Phi_{\epsilon}(\bS^{(0)}, \bL^{(0)}, R^{(0)}) - \Phi_{\epsilon}(\hat{\bS}_{\epsilon}, \hat\bL, \hat{R})$, we have:
\begin{equation}
\label{eq:suite-v3}
\Delta^{t+1}\leq \frac{1}{(\tilde{\Delta}^0)^{-1} + \sum_{i=1}^t\frac{1}{C^{(i)}}}.
\end{equation}
Let us derive an upper bound on the time-varying constants $C^{(t)}$ defined in \eqref{eq:Ct}. We only need to bound  $\bar{R}^{(t)}$, $M^{(t)}$ and $Q^{(t)}$.
First note that, by Lemmas \ref{lem:S-decrease} and \ref{lem:L-decrease}, $\bar{R}^{(t)}\leq \lambda_1^{-1}\Phi_{\epsilon}(\bS^{(0)}, \bL^{(0)}, R^{(0)})$, and $Q^{(t)}\leq \lambda_2^{-1}\Phi_{\epsilon}(\bS^{(0)}, \bL^{(0)}, R^{(0)})$. To bound $M^{(t)} = \norm{\bG_L(\bS^{(t)},\bL^{(t-1)})}[F]$, we start by noticing that the gradient $\bG_L(\bS^{(t)},\bL^{t-1})$ of the quadratic part of the objective with respect to $\bL$ is bounded whenever $\bS^{(t)}$ and $\bL^{(t-1)}$ are bounded themselves. Since $\lambda_1\norm{\bL^{(t-1)}}[*] + \lambda_2\norm{\bS^{(t)}}[2,1]\leq \Phi_{\epsilon}(\bS^{(t)}, \bL^{(t-1)}, R^{(t-1)})\leq\Phi_{\epsilon}(\bS^{(0)}, \bL^{(0)}, R^{(0)})$, the parameters $\bS$ and $\bL$ are indeed bounded, and we obtain that there exists $\bar{M}\geq 0$ such that $M^{(t)}\leq \bar{M}$ for any $t$. Define $\mathcal{F}_0\eqdef\Phi_{\epsilon}(\bS^{(0)}, \bL^{(0)}, R^{(0)})$,
\begin{equation*}
\label{eq:Cbar123}
\bar{C}_1  = \max\{8\lambda_1^{-1}(1+\epsilon)\mathcal{F}_0^2, 2\lambda_1^{-1}\mathcal{F}_0(\lambda_1+\bar{M}) \},\quad
\bar{C}_2  = \frac{8\mathcal{F}_0^2}{\eta\lambda_2^2},\quad
\bar{C}_3  = 8\lambda_2^{-1}(1+\epsilon)\mathcal{F}_0^2,
\end{equation*}
and
\begin{equation*}
\label{eq:Cbar}
\bar{C}\eqdef \max\left\{\bar{C}_1+\bar{C}_3, \bar{C}_2\right\}.
\end{equation*}
Then, we obtain the following rate of convergence:
\begin{equation}
\Delta^{t+1}\leq \frac{1}{(\tilde{\Delta}^0)^{-1} + \sum_{i=1}^t\frac{1}{C^{(i)}}}\leq \frac{1}{(\tilde{\Delta}^0)^{-1} + t\bar{C}}. 
\end{equation}
Recall that $\Phi_{\epsilon}(\hat{\bS}_{\epsilon}, \hat\bL_{\epsilon}, \hat{R}) = \mathcal{F}(\hat{\bS}_{\epsilon}, \hat\bL_{\epsilon})$ by equivalence of the two optimization problems \eqref{eq:estimation} and \eqref{eq:aug-estimation}. In addition, by definition, $\norm{L^{(t-1)}}[*]\leq R^{(t-1)}$, which gives $\mathcal{F}_{\epsilon}(\bS^{(t)}, \bL^{(t-1)})\leq \Phi_{\epsilon}(\bS^{(t)}, \bL^{(t-1)}, R^{(t-1)})$. Thus, we obtain that $\mathcal{F}_{\epsilon}(\bS^{(t)}, \bL^{(t-1)}) - \mathcal{F}_{\epsilon}(\hat{\bS}_{\epsilon}, \hat\bL_{\epsilon})\leq \Phi_{\epsilon}(\bS^{(t)}, \bL^{(t-1)}, R^{(t-1)}) - \Phi_{\epsilon}(\hat{\bS}_{\epsilon}, \hat\bL_{\epsilon}, \hat{R})\leq \Delta^{t+1}$. \newline
For $\delta> 0$, let $T_{\delta}$ be the integer number defined by:
$$T_{\delta}\eqdef \left \lfloor\bar{C}\left(\frac{1}{\delta} - \frac{1}{\mathcal{F}_0-\mathcal{F}_{\epsilon}(\hat\bS_{\epsilon}, \hat\bL_{\epsilon})} \right)  \right \rfloor+1.$$
Then, the $T_{\delta}$-th iterate of the MCGD sequence satisfies:
$$\mathcal{F}_{\epsilon}(\bS^{(T_{\delta})}, \bL^{(T_{\delta})})- \mathcal{F}_{\epsilon}(\hat\bS_{\epsilon}, \hat\bL_{\epsilon})\leq \delta,$$
which proves sub-linear convergence of the MCGD iterates. Note that, by definition, $\mathcal{F}_0-\mathcal{F}_{\epsilon}(\hat\bS_{\epsilon}, \hat\bL_{\epsilon})\geq 0$, which implies that $T_{\delta}\leq\left\lfloor\bar{C}/\delta  \right \rfloor + 1$. In addition, in the particular case where the initial point is set to $(\bS^{(0)}, \bL^{(0)}, R^{(0)}) = (\mathbf{0}, \mathbf{0}, 0)$, we can compute an upper bound on the constant $\bar{C}$, dependent on the dimensions of the problem. First, note that in this case, $\mathcal{F}_0 = \frac{1}{2}\norm{\Omega\odot \bA}[F]^2$ is equal to the number of observed edges in the graph, denoted by $E$. Furthermore, by definition, $$M^{(t)} = \norm{\bG_L(\bS^{(t)}, \bL^{(t-1)})}[F] \leq \norm{\Omega\odot(\bA-\bL^{(t-1)}-\bS^{(t)}-(\bS^{(t)})^{\top})}[F]+\norm{\epsilon \bL^{(t-1)}}[F].$$
Since, by Lemmas \ref{lem:S-decrease} and \ref{lem:L-decrease}, the objective value decreases at every update of $\bL$ and $\bS$. As all the terms of the objective are positive, we have that $\norm{\Omega\odot(\bA-\bL^{(t-1)}-\bS^{(t)}-(\bS^{(t)})^{\top})}[F]^2\leq \mathcal{F}_0 = E$, and $\norm{\epsilon \bL^{(t-1)}}[F]^2\leq E$ as well. Thus, we obtain that, for any $t$, $M^{(t)}\leq 2\sqrt{E}$, which yields $\bar{M}\leq 2\sqrt{E}$. We then obtain that the constant $\bar{C}$ satisfies
\begin{equation}
\label{eq:barC}
  \bar{C}\leq \bar{C}_0\eqdef\max\left\{\frac{2E^2}{\eta\lambda_2^2},8(1+\epsilon)E^2\left(\frac{1}{\lambda_1}+\frac{1}{\lambda_2}\right)+\frac{2E^{3/2}}{\lambda_1} + 2E\right\},  
\end{equation}
meaning that the number of iterations increases at most quadratically with the density of the graph.
Note that, in practice, the convergence is much faster, and we observe that the algorithm converges after a few iterations.


\subsection{Proof of Theorem \ref{thm:inliers_detection}}
\label{subsec:Proof_inliers_detection}
Recall that, by Lemma \ref{lem:syst_S},  $$j \in \widehat{\cO} \iff \left \Vert\bOmega_{\cdot,j} \odot \left(\bA_{\cdot,j} - \widehat{\bL}_{\cdot,j} - \widehat{\bS}_{j, \cdot}\right)_+\right \Vert_2 > \frac{\lambda_2}{4}.$$
In a first time, we show that with high probability, no inlier belongs to the set of estimated outliers. Consider $j \in \cI$, then 
\begin{eqnarray}
   \left \Vert\bOmega_{\cdot,j} \odot \left(\bA_{\cdot,j} - \widehat{\bL}_{\cdot,j} - \widehat{\bS}_{j, \cdot}\right)_+\right \Vert_2  &\leq& \sqrt{\underset{i \in \cI}{\sum} \left(\bOmega_{ij}\left(\bA_{ij} - \widehat{\bL}_{ij} - \widehat{\bS}_{ji} \right)_+ \right)^2} + \sqrt{\underset{i \in \cO}{\sum} \left(\bOmega_{ij}\left(\bA_{ij} - \widehat{\bL}_{ij} - \widehat{\bS}_{ji} \right)_+ \right)^2} \nonumber \\
   &\leq & \sqrt{\underset{i \in \cI}{\sum} \left(\bOmega_{ij}\left(\bSigma_{ij} + \Delta \bL_{ij} - \widehat{\bS}_{ji} \right)_+ \right)^2} + \sqrt{\underset{i \in \cO}{\sum} \left(\bOmega_{ij}\bA_{ij}\right)^2} \nonumber
   \end{eqnarray}
   where we have used that  for $(i,j) \in \cI \times \cI$, $\bA_{ij} = \bSigma_{ij} + \bL^*_{ij}$ and that $\widehat{\bL}_{ij}\geq 0$ and $\widehat{\bS}_{ij}\geq 0$. Therefore, we find that 
   \begin{eqnarray}
  \left \Vert \bOmega_{\cdot,j} \odot \left(\bA_{\cdot,j} - \widehat{\bL}_{\cdot,j} - \widehat{\bS}_{j, \cdot}\right)_+\right \Vert_2  &\leq& \sqrt{\underset{i \in \cI}{\sum} \left(\bOmega_{ij}\bSigma_{ij} \right)_+^2} + \sqrt{\underset{i \in \cI}{\sum} \left(\bOmega_{ij}\Delta \bL_{ij} \right)_+^2} + \sqrt{\underset{i \in \cO}{\sum} \left(\bOmega_{ij}\bA_{ij}\right)^2}  \nonumber.
\end{eqnarray}
Recalling that $\left \Vert \Delta \bL \right \Vert_{\infty} \leq \rho_n$, we obtain 
\begin{eqnarray}
  \underset{j \in \cI}{\max}\left \{\left \Vert \bOmega_{\cdot,j} \odot \left(\bA_{\cdot,j} - \widehat{\bL}_{\cdot,j} - \widehat{\bS}_{j, \cdot} \right)_+\right \Vert_2\right\}   &\leq & \left \Vert \bOmega \odot \bSigma_{\vert I} \right \Vert_{2,\infty} + \rho_n\left \Vert \bOmega_{\vert I} \right \Vert_{2,\infty} + \left \Vert \bOmega\odot\bA_{\vert \cO \times \cI} \right \Vert_{2,\infty} \label{eq:bound_colnom_inliers}.
\end{eqnarray}
 We bound $\left \Vert \bOmega \odot \bSigma_{\vert I} \right \Vert_{2,\infty}$, $\rho_n\left \Vert \bOmega_{\vert I} \right \Vert_{2,\infty}$ and $\left \Vert \bOmega\odot\bA_{\vert \cO \times \cI} \right \Vert_{2,\infty}$ using the following Lemma.

\begin{lem}\label{lem:bound_Omega_A_out}
Under assumptions \ref{ass:quasi_uniforme}-\ref{ass:out_con}, 
\begin{eqnarray}
\mathbb{P}\left(\left \Vert \bOmega\odot\bSigma_{\vert I} \right \Vert_{2,\infty} \geq \sqrt{6\nu_n\rho_n n } \right)&\leq& 2e^{ - \nu_n\rho_n n }\label{eq:lem_noise_i}\\
\mathbb{P} \left( \left \Vert \bOmega_{\vert I} \right \Vert_{2, \infty} \geq 4\sqrt{\nu_n n }\right)&\leq& 2e^{ - \nu_n n} \label{eq:lem_noise_ii}\\
\mathbb{P}\left(\left \Vert \bOmega\odot\bA_{\vert \cO \times \cI} \right \Vert_{2,\infty} \geq \sqrt{6\nu_n\rho_nn}\right)&\leq& 2e^{-\nu_n \rho_n n}. \label{eq:lem_noise_iii}
\end{eqnarray}
\end{lem}
\begin{proof}
See Section \ref{subsubsect:proofLemOmega}
\end{proof}
Recall that $\lambda_2 = 19 \sqrt{\nu_n \rho_n n}$. Combining Lemma \ref{lem:bound_Omega_A_out}, Lemma \ref{lem:bound_lambda} and equation \eqref{eq:bound_colnom_inliers} yields that with probability larger than $1 - 6e^{-\nu_n \rho_n n}$,
\begin{eqnarray}
  \underset{j \in \cI}{\max}\left \{\left \Vert \bOmega_{\cdot,j} \odot \left(\bA_{\cdot,j} - \widehat{\bL}_{\cdot,j} - \widehat{\bS}_{j, \cdot} \right)_+\right \Vert_2\right\}   &\leq & 9\sqrt{\nu_n \rho_n n} < \frac{\lambda_2}{2}.\nonumber
\end{eqnarray}
Using Lemma \ref{lem:syst_S}, we conclude that with probability at least $1- 6e^{-\nu_n \rho_n n}$, $\widehat{\cO} \cap \cI = \emptyset$.

\subsection{Proof of Theorem \ref{thm:outliers_detection}}
\label{proof:outliers_detection}

Here, we prove that with high probability, all outliers are detected when $\underset{j \in \cO}{\min}\underset{i \in \cI}{\sum} \bPi_{ij}\bS^*_{ij} > C\rho_n\nu_n n$ for some absolute constant $C>0$. For any $j \in [n]$, note that
\begin{eqnarray}
   \left \Vert\bOmega_{\cdot,j} \odot \left(\bA_{\cdot,j} - \widehat{\bL}_{ \cdot,j} - \widehat{\bS}_{j, \cdot}\right)_+\right \Vert_2 &\geq& \sqrt{\underset{i \in \cI}{\sum} \left(\bOmega_{ij} \left(\bA_{ij} - \widehat{\bL}_{ij} - \widehat{\bS}_{ji} \right)_+ \right)^2}\nonumber.
\end{eqnarray}

We have shown in Theorem \ref{thm:inliers_detection} that with probability at least $1- 6e^{-\nu_n \rho_n n}$, $\widehat{\bS}_{ji} = 0$ for any $i\in \cI$ and any $j \in [n]$ . When this equation holds, using the bound $\left \Vert \widehat{\bL} \right \Vert_{\infty} \leq \rho_n$, we find that
\begin{eqnarray}
   \left \Vert \bOmega_{\cdot,j} \odot \left(\bA_{\cdot,j} - \widehat{\bL}_{ \cdot,j} - \widehat{\bS}_{j, \cdot}\right)_+\right \Vert_2 &\geq& \sqrt{\underset{i \in \cI}{\sum} \left(\bOmega_{ij} \left(\bA_{ij} - \rho_n \right)_+ \right)^2}.\label{eq:lower_bound_outliers}
\end{eqnarray}

We use the following Lemma to obtain a lower bound on the right hand side of equation \eqref{eq:lower_bound_outliers} when $j \in \cO$.

\begin{lem}\label{lem:lower_bound_outlier}   Assume that $\underset{j \in \cO}{\min} \underset{i \in \cI}{\sum} \bPi_{ij}\bS^*_{ij}\geq \nu_n \rho_n n$, then $$\mathbb{P}\left(\underset{j \in \cO}{\min} \sqrt{\underset{i \in \cI}{\sum} \left(\bOmega_{ij} \left(\bA_{ij} - \rho_n \right)_+ \right)^2}\leq \frac{1}{4}\underset{j \in \cO}{\min} \sqrt{\underset{i \in \cI}{\sum} \bPi_{ij}\bS^*_{ij}} \right) \leq 2se^{- \frac{-\nu_n \rho_n n}{80}}.$$
\end{lem}
\begin{proof}
See Section \ref{subsubsect:proof_LBO}.
\end{proof}
Combining this Lemma with equation \eqref{eq:lower_bound_outliers}, we see that with probability at least $1 - 2se^{- \frac{-\nu_n \rho_n n}{80}} - 6e^{-\nu_n \rho_n n}$,
\begin{eqnarray}
   \left \Vert \bOmega_{\cdot,j} \odot \left(\bA_{\cdot,j} - \widehat{\bL}_{ \cdot,j} - \widehat{\bS}_{j, \cdot}\right)_+\right \Vert_2 &\geq&  \frac{1}{4}\underset{j \in \cO}{\min} \sqrt{\underset{i \in \cI}{\sum} \bPi_{ij}\bS^*_{ij}}.\label{eq:lower_bound_outliers_2}
\end{eqnarray}

Recall that $\lambda_2 = 19\sqrt{\nu_n\rho_n n}$. When $\underset{j \in \cO}{\min}\underset{i \in \cI}{\sum} \bPi_{ij}\bS^*_{ij} > 8\times19\nu_n\rho_n n$, Lemma \ref{lem:lower_bound_outlier} and equation \eqref{eq:lower_bound_outliers_2} imply that with probability larger than  $1 - 2se^{- \frac{-\nu_n \rho_n n}{80}} - 6 e^{-\nu_n\rho_n n}$, 
\begin{eqnarray}
   \left \Vert\bOmega_{\cdot,j} \odot \left(\bA_{\cdot,j} - \widehat{\bL}_{ \cdot,j} - \widehat{\bS}_{j, \cdot}\right)_+\right \Vert_2 > \frac{\lambda_2}{2}.\nonumber
\end{eqnarray}
Combining this result with Lemma \ref{lem:syst_S}, we find that with probability at least $1 - 2se^{- \frac{-\nu_n \rho_n n}{80}} - 6e^{-\nu_n\gamma_n n} \geq 1 - 8se^{-\frac{-\nu_n \rho_n n}{80}}$, $\cO \subset \widehat{\cO}$. This concludes the proof of Theorem \ref{thm:outliers_detection}.


\subsection{Proof of Theorem \ref{thm:Borne_norme_2}}
\label{subsec:ProofTh4}
To prove Theorem \ref{thm:Borne_norme_2}, we use the definition of $\widehat{\bL}$, the separability of the $\Vert \cdot\Vert_{*}$-norm on orthogonal subspaces, and results on $\widehat{\bS}$ proved in Theorem \ref{thm:outliers_detection}.
Recall that $\Psi \triangleq 16\tilde{\nu}_n \gamma_n \rho_n s n.$
\begin{lem}\label{lem:erreur_L}
Assume that $\lambda_1 \geq 3\left \Vert\bOmega \odot \bSigma_{\vert I}\right \Vert_{op}$, and that $\lambda_2 = 19\sqrt{\nu_n \rho_n n}$. Then,
\begin{eqnarray}
  \left\Vert \bOmega\odot\Delta \bL \right \Vert_F^2 &\leq& \frac{\lambda_1}{3} \left(5\left \Vert \cP_{\bL^*}\left( \Delta \bL  \right)\right \Vert_{*} - \left \Vert \cP_{\bL^*}^{\perp}\left( \Delta \bL \right) \right \Vert_{*} \right) + \Psi \label{eq:erreur_l_i}\\
  \text{and }\ \left \Vert \Delta \bL \right \Vert_{*} &\leq& 6\sqrt{k}\left \Vert\Delta \bL _{\vert I}\right \Vert_F +6\sqrt{3ksn}\rho_n + \frac{3\Psi}{\lambda_1}.\label{eq:erreur_l_ii}
\end{eqnarray}
hold simultaneously with equation \eqref{eq:detectedOutliers} with probability at least $1 - 6 e^{-\nu_n \rho_n n} - 2e^{-\tilde{\nu}_n\gamma_nsn}$.
\end{lem}
\begin{proof}
See Section \ref{subsubsect:prooferreurL}.
\end{proof}
Bounding the $\Vert \cdot \Vert_{L_2(\bPi)}$-norm of the error $\Delta \bL$ by $\left\Vert \bOmega\odot\Delta \bL\right \Vert_F^2$ is rather involved, and we use a peeling argument, combined with the bound on $\Vert \Delta \bL \Vert_{*}$ obtained in equation \eqref{eq:erreur_l_ii}  in Lemma \ref{lem:erreur_L}. We recall that $\bGamma$ is the random matrix defined as $\bGamma_{ij} = \epsilon_{ij}\bOmega_{ij}$ for all $(i,j) \in [n] \times [n]$, where $\{\epsilon\}_{i\leq j}$ is a Rademacher sequence. Moreover, we introduce the following notation :
\begin{equation}\label{eq:def_beta}
  \beta \triangleq \mathbb{E}\left[ \left\Vert\bGamma_{\vert I} \right \Vert_{op} \right]\left(\frac{48^2\rho_n^2 k}{\mu_n}\mathbb{E}\left[ \left\Vert\bGamma_{\vert I} \right \Vert_{op} \right] + 60\rho_n^2\sqrt{ksn} +  \frac{32\Psi\rho_n}{\lambda_1} \right).  
\end{equation}

\begin{lem}\label{lem:error_X_L}
Assume that $\lambda_1 \geq 3\left \Vert\bOmega \odot \bSigma_{\vert I}\right \Vert_{op}$, and that $\lambda_2 = 19\sqrt{\nu_n \rho_n n}$. Then, there exists an absolute constant $C>0$ such that
\begin{equation}\label{eq:error_X_L}
\left \Vert \Delta \bL_{\vert I}\right \Vert_{L_2(\bPi)}^2 \leq C \left(  \frac{\lambda_1^2k}{\mu_n} +\nu_n \rho_n^2 s n + \frac{\nu_n \rho_n ^2 k n}{\mu_n} + \Psi +  \beta\right)
\end{equation}
holds simultaneously with equations \eqref{eq:detectedOutliers}, \eqref{eq:erreur_l_i} and \eqref{eq:erreur_l_ii} with probability at least $1- 7e^{-\nu_n \rho_n n} - 2e^{-\tilde{\nu}_n \gamma_n sn}$.
\end{lem}
\begin{proof}
See Section \ref{subsubsect:prooferrorX}.
\end{proof}
Finally, we bound $ \beta$ using the following lemma.
\begin{lem}\label{lem:bound_Gamma_op}
$\mathbb{E}\left[ \left\Vert\bGamma_{\vert I} \right \Vert_{op} \right] \leq 84\sqrt{\nu_n n}$.
\end{lem}
Lemma \ref{lem:bound_Gamma_op} implies that there exists some absolute constant $C>0$ such that 
\begin{eqnarray*}
 \beta &\leq& C\sqrt{\nu_n n  }\left(\frac{ \rho_n^2 k}{\mu_n}\sqrt{\nu_n n  }+ \rho_n^2\sqrt{skn}+ \frac{\Psi\rho_n}{\lambda_1} \right).
\end{eqnarray*}
\begin{proof}
See Section \ref{subsubsect:proofGammaOp}.
\end{proof}
Thus, there exists an absolute constant $C>0$ such that when equation \eqref{eq:error_X_L} holds,
\begin{eqnarray*}
 \beta &\leq& C \left(\frac{\nu_n  \rho_n^2 k n }{\mu_n} + \rho_n^2n\sqrt{\nu_n sk} + \frac{\Psi \sqrt{\nu_n n  }\rho_n}{\lambda_1}\right).
\end{eqnarray*}
Combining Lemma \ref{thm:Borne_norme_2} and Lemma \ref{lem:erreur_L}-\ref{lem:error_X_L}, and noticing that $\sqrt{\nu_n s k} \leq \nu_ns + k$ and that $\frac{\nu_n}{\mu_n} \geq 1$, we find that there exists an absolute constant $C>0$ such that with probability at least $1 - 7e^{-\nu_n \rho_nn} - 2e^{-\tilde{\nu}_n \gamma_n sn}$,
\begin{eqnarray*}
\left \Vert \Delta \bL_{\vert I}\right \Vert_{L_2(\bPi)}^2 &\leq& C\left(\frac{\lambda_1^2k}{\mu_n}+\nu_n \rho_n^2 s n + \frac{\nu_n \rho_n ^2 k n}{\mu_n}  +\Psi + \frac{\nu_n  \rho_n^2 k n }{\mu_n} + \rho_n^2n\sqrt{\nu_n sk} + \frac{\Psi \sqrt{\nu_nn  } \rho_n}{\lambda_1}\right)\\
&\leq&C\left(\frac{\lambda_1^2k}{\mu_n}+ n\rho_n^2\left( \nu_n s + \frac{\nu_nk}{\mu_n}\right) + \Psi \left(\frac{\sqrt{\nu_n n  } \rho_n}{\lambda_1}+ 1 \right)\right).
\end{eqnarray*}
Recall that $\Phi \triangleq n\rho_n^2\left( \frac{\nu_nk}{\mu_n} + \nu_n s\right)$, and that $\Xi \triangleq \frac{\sqrt{\nu_n n  } \rho_n}{\lambda_1}+ 1$. With these notations, we find that 
\begin{eqnarray*}
\left \Vert \Delta \bL_{\vert I}\right \Vert_{L_2(\bPi)}^2 \leq C\left(\frac{\lambda_1^2k}{\mu_n}+\Phi + \Psi \Xi \right)
\end{eqnarray*}
with probability at least $1- 7e^{-\nu_n \rho_n n} - 2e^{-\tilde{\nu}_n \gamma_n sn}$. We conclude the proof of Theorem \ref{thm:Borne_norme_2} by recalling that $\nu_n\rho_nn\geq \log(n)$ and $\tilde{\nu}_n \gamma_nn\geq \log(n)$.

\subsection{Proof of Corollary \ref{corollary_fixed_lambda}}
\label{subsec:ProofCor1}
Lemma \ref{lem:bound_lambda} allows us to 
choose $\lambda_1$ by bounding the noise terms $\left \Vert \bOmega\odot\bSigma _{\vert I} \right \Vert_{op}$ with high probability. For the choice $\lambda_1 = 84\sqrt{\nu_n\rho_nn}$, we find that
\begin{eqnarray*}
\Xi &=& \left(1 + \frac{\sqrt{\nu_n \rho_n^2 n}}{84\sqrt{\nu_n \rho_nn}}\right) \leq 2.
\end{eqnarray*}
Combining Lemma \ref{lem:bound_lambda} with Theorem \ref{thm:Borne_norme_2}, we find that there exists an absolute constant $C>0$ such that with probability at least $1-7e^{-\nu_n \rho_n n} - 3e^{-\tilde{\nu}_n\gamma_n n s}$,
\begin{eqnarray*}
    \left \Vert \Delta \bL_{\vert I}\right \Vert_{L_2(\bPi)}^2 &\leq& C\left(\frac{\nu_n \rho_n kn }{\mu_n} +   n\rho_n^2\left( \frac{\nu_nk}{\mu_n} + \nu_n s\right)+  \tilde{\nu}_n \rho_n\gamma_n s n\right)\\
    &\leq& C\left(\frac{\nu_n \rho_n kn }{\mu_n} +  \rho_n (\nu_n \rho_n \lor \tilde{\nu}_n \gamma_n) s n\right). 
\end{eqnarray*}

\subsection{Proof of auxiliary Lemmas}
\label{subsec:ProofAuxRes}

\subsubsection{Proof of Lemma \ref{lem:syst_S}}
\label{subsubsec:syst_S}
Recall that by definition of $\widehat{\bS}$, 
\begin{eqnarray}
   \widehat{\bS} \in \underset{\bS \in \mathbb{R}_+^{n \times n}}{\argmin}\ \left\{\ \frac{1}{2}\ \left \Vert \bOmega \odot \left( \bA - \widehat{\bL} - \bS - \bS^{\top}\right) \right\Vert_F^2 + \lambda_2 \left \Vert \bS \right \Vert_{2,1}\right\} \label{objectif_S}
\end{eqnarray}

 Now, any subgradient of the objective function \eqref{objectif_S} at $\widehat{\bS}$ is of the form 
 $$ \nabla_{\bS} \cF \left(\widehat{\bS}, \widehat{\bL} \right) = 2 \bOmega \odot \left( - \bA + \widehat{\bL} +  \widehat{\bS} + \widehat{\bS}^{\top}\right) + \lambda_2\bW$$
 where $\bW$ is a subgradient of the $\left \Vert \cdot \right \Vert_{2,1}$-norm at $\widehat{\bS}$. The matrix $\bW$ obeys the following constraints :
 \begin{itemize}
     \item for any  $j \in [n]$ such that the column $\widehat{\bS}_{\cdot, j}$ is null, $\left \Vert \bW_{\cdot, j} \right \Vert_2 \leq 1$;
     \item for any  $j \in [n]$ such that $\widehat{\bS}_{\cdot, j} \neq \mathbf{0}$, $\left \Vert \bW_{\cdot, j} \right \Vert_2 = \frac{\widehat{\bS}_{\cdot,  j}}{\left \Vert \widehat{\bS}_{\cdot, j} \right \Vert_2}$.
 \end{itemize}
The Karush-Kuhn-Tucker conditions (see, e.g., \cite{boyd_vandenberghe_2004}, Section 5.5.3) imply that there exists $\bH \in \mathbb{R}^{n \times n}$ and $\bW \in \partial\left \Vert \cdot \right \Vert_{2,1}$ such that
\begin{eqnarray}
  2\bOmega \odot \left( - \bA + \widehat{\bL} +  \widehat{\bS} + \widehat{\bS}^{\top}\right) + \lambda_2\bW - \bH = \mathbf{0} \label{eq:0_subgrad}\\
  \bH_{ij} \geq 0 \textbf{ for any } (i,j) \in [n]\times [n]\label{eq:H_positif}\\
  \bH \odot \widehat{\bS} = \mathbf{0} \label{eq:H_prod_S_0}
\end{eqnarray}

First, we prove the implication  $\widehat{\bS}_{\cdot, j} = \mathbf{0} \Rightarrow \left \Vert\bOmega \odot \left(\bA_{j, \cdot} - \widehat{\bL}_{j, \cdot} - \widehat{\bS}_{j, \cdot}\right)_+\right \Vert_2 \leq \frac{\lambda_2}{2}$. To do so, assume that $j$ is such that $\widehat{\bS}_{\cdot, j} = \mathbf{0}$. Then, equation \eqref{eq:0_subgrad} implies that \begin{eqnarray*}
    \lambda_2 \bW_{\cdot, j} = 2 \bOmega \odot \left( \bA_{\cdot, j} - \widehat{\bL}_{\cdot, j} - \widehat{\bS}_{j,\cdot}\right) + \bH_{\cdot, j} .
\end{eqnarray*}
Recall that $\left \Vert \bW_{\cdot, j}\right \Vert_2 \leq 1$, and thus 
\begin{eqnarray*}
\frac{2}{\lambda_2}\left \Vert \bOmega_{\cdot, j} \odot \left( \bA_{\cdot, j} - \widehat{\bL}_{\cdot, j} - \widehat{\bS}_{j,\cdot}\right) + \frac{1}{2}\bH_{\cdot, j} \right \Vert_2 \leq 1.
\end{eqnarray*}
Moreover, by \eqref{eq:H_positif}, $\bH_{ij} \geq 0$. Therefore,
\begin{eqnarray}
 \frac{2}{\lambda_2}\left \Vert \bOmega_{\cdot, j} \odot \left( \bA_{\cdot, j} - \widehat{\bL}_{\cdot, j} - \widehat{\bS}_{j,\cdot}\right)_+ \right \Vert_2 &\leq& \frac{2}{\lambda_2}\left \Vert\left( \bOmega_{\cdot, j} \odot \left( \bA_{\cdot, j} - \widehat{\bL}_{\cdot, j} - \widehat{\bS}_{j,\cdot}\right) + \frac{1}{2}\bH_{\cdot, j}\right)_+ \right \Vert_2 \nonumber\\
 &\leq&\frac{2}{\lambda_2}\left \Vert \bOmega_{\cdot, j} \odot \left( \bA_{\cdot, j} - \widehat{\bL}_{\cdot, j} - \widehat{\bS}_{j,\cdot}\right) + \frac{1}{2}\bH_{\cdot, j} \right \Vert_2 \leq 1.\nonumber
\end{eqnarray}
This concludes the proof of the first implication.

To prove the other implication, assume that $j$ is such that $\widehat{\bS}_{\cdot, j} \neq \mathbf{0}$. Then $\bW_{\cdot, j} = \frac{\widehat{\bS}_{\cdot,  j}}{\left \Vert \widehat{\bS}_{\cdot, j} \right \Vert_2}$, and equation \eqref{eq:0_subgrad} becomes
\begin{eqnarray*}
 \left(2 + \frac{\lambda_2}{\left \Vert \widehat{\bS}_{\cdot, j} \right \Vert_2} \right)\widehat{\bS}_{\cdot, j} =  2 \bOmega_{\cdot, j} \odot \left( \bA_{\cdot, j} - \widehat{\bL}_{\cdot, j} - \widehat{\bS}_{j, \cdot}\right)+ \bH_{\cdot, j} + 2\left(1 - \bOmega_{\cdot, j} \right)\odot \widehat{\bS}_{\cdot, j}.\label{eq:devCol}
\end{eqnarray*}
First, assume that for some $i \in [n]$, $\bH_{ij} \neq 0$. Then, equation \eqref{eq:H_prod_S_0} implies that $\widehat{\bS}_{ij}= 0$, and so $$\bOmega_{ij}\left( \bA_{ij} - \widehat{\bL}_{ij} - \widehat{\bS}_{ji}\right) = - \bH_{ij}/2 < 0.$$ On the other hand, assume that for $i \in [n]$, $\bH_{ij} = 0$. Then, $\widehat{\bS}_{ij} \geq 0$ implies that $$\bOmega_{ij}\left( \bA_{ij} - \widehat{\bL}_{i j} - \widehat{\bS}_{ji}\right) + \left(1 - \bOmega_{ij} \right) \widehat{\bS}_{ij} \geq 0$$ which implies that $\bOmega_{ij}\left( \bA_{ij} - \widehat{\bL}_{ij} - \widehat{\bS}_{ji}\right)\geq 0$. This shows that for $j \in [n]$ such that $\widehat{\bS}_{\cdot, j} \neq \mathbf{0}$,
\begin{eqnarray}\label{eq:def_S}
 \left(2 + \frac{\lambda_2}{\left \Vert \widehat{\bS}_{\cdot, j} \right \Vert_2} \right)\widehat{\bS}_{\cdot, j} =  2 \bOmega_{\cdot, j} \odot \left( \bA_{\cdot, j} - \widehat{\bL}_{\cdot, j} - \widehat{\bS}_{j, \cdot}\right)_+ + 2\left(1 - \bOmega_{\cdot, j}\right)\odot \widehat{\bS}_{\cdot, j}.
\end{eqnarray}
Now, for all $i$ such that $\bOmega_{ij} = 0$, equation \eqref{eq:def_S} becomes $\left(2 + \frac{\lambda_2}{\left \Vert \widehat{\bS}_{\cdot, j} \right \Vert_2} \right)\widehat{\bS}_{ij} =  2 \widehat{\bS}_{ij}$, and thus $\widehat{\bS}_{ij} = 0$. 
This remarks, combined with equation \eqref{eq:def_S}, implies that
\begin{eqnarray*}
 \left(2 + \frac{\lambda_2}{\left \Vert \widehat{\bS}_{\cdot, j} \right \Vert_2} \right)\widehat{\bS}_{\cdot, j} =  2 \bOmega_{\cdot, j} \odot \left( \bA_{\cdot, j} - \widehat{\bL}_{\cdot, j} - \widehat{\bS}_{j, \cdot}\right)_+.
\end{eqnarray*}
This implies in particular that 
\begin{eqnarray}
 2 \left \Vert \left(\bOmega_{\cdot, j} \odot \left( \bA_{\cdot, j} - \widehat{\bL}_{\cdot, j} - \widehat{\bS}_{j, \cdot}\right)\right)_+\right \Vert_2 = 2\left \Vert \widehat{\bS}_{\cdot, j} \right \Vert_2 + \lambda_2 > \lambda_2. \nonumber
\end{eqnarray}
This concludes the proof of Lemma \ref{lem:syst_S}.

\subsubsection{Proof of Lemma \ref{lem:unicite}}

Note that for any partition of the nodes into inliers $\cI$ and outliers $\cO$, the solution $(\bL^*, \bS^*)$ to equation \eqref{eq:1defLSmin} such that $\cO$ is the support of the columns of $\bS^*$ is unique up to diagonal terms (if it exists). Indeed, we then have $\bL^* = \mathbb{E}[\bA]_{\vert \cI \times \cI}$ and   $\bS^* = \mathbb{E}[\bA]_{\vert \cI \times \cO} + 1/2 \mathbb{E}[\bA]_{\vert \cO \times \cO}$. Thus, it is enough to prove that the partition into inliers and outliers is unique to prove Lemma \ref{lem:unicite}.

We prove Lemma \ref{lem:unicite} by contradiction. Let us assume that there exists two different sets $\cO$ and $\tilde{\cO}$ such that there exists two solutions $(\bL^*, \bS^*)$ and $(\tilde{\bL}, \tilde{\bS})$ to equation \eqref{eq:1defLSmin}, where $\cO$ is the support of the columns of $\bS^*$, and $\tilde{\cO}$ that of $\tilde{\bS}$, and such that $\nu_nn \geq (\max_{i \in \cI} \underset{j \in \cI}{\sum}\bPi_{ij})\lor(\max_{i \in \tilde{\cI}} \underset{j \in \tilde{\cI}}{\sum}\bPi_{ij})$ and $\tilde{\nu_n}s \geq (\max_{i \in \cI} \underset{j \in \cO}{\sum}\bPi_{ij}) \lor (\max_{i \in \tilde{\cI}} \underset{j \in \tilde{\cO}}{\sum}\bPi_{ij})$. Here, we have defined $\cO = \{j: \bS^*_{\cdot j} \neq \mathbf{0}\}$,  $\tilde{\cO} = \{j: \tilde{\bS}_{\cdot j} \neq \mathbf{0}\}$,  $\cI = \{j: \bL^*_{\cdot j} \neq \mathbf{0}\}$, and $\tilde{\cI} = \{j: \tilde{\bL}_{\cdot j} \neq \mathbf{0}\}$. Note that \eqref{eq:1defLSmin} implies that $\vert \cO \vert = \vert \tilde{\cO}\vert$, and thus there exists $j \in \cO \cap \tilde{\cI}$. 

We obtain a contradiction  by proving that the expected observed degree of $j$ is too large for $j$ to be an inlier. By definition of $(\bL^*, \bS^*)$, one has $\bS^* = \mathbb{E}[\bA]_{\vert \cI \times \cO} + 1/2 \mathbb{E}[\bA]_{\vert \cO \times \cO}$. Since $j \in \cO$, this yields $\underset{i \in \cI}{\sum} \bPi_{ij}\bS^*_{ij} = \underset{i \in \cI}{\sum} \bPi_{ij}\mathbb{E}[\bA]_{ij}$. Under assumption \ref{ass:outlier_detection}, we find that$\underset{i \in \cI}{\sum} \bPi_{ij}\mathbb{E}[\bA]_{ij}\geq C\rho_n\nu_nn$, where $C = 8 \times 19$. In particular, this implies that $\underset{i \in [n]}{\sum} \bPi_{ij}\mathbb{E}[\bA]_{ij}\geq 152\rho_n\nu_nn$.

Now, since $j \in \tilde{\cI}$, for all $i \in \tilde{\cI}$, we have $\mathbb{E}[\bA]_{ij} \leq \rho_n$ and $\underset{i \in \tilde{\cI}}{\sum} \bPi_{ij} \leq \nu_nn$. Thus, $\underset{i \in \tilde{\cI}}{\sum} \bPi_{ij}\mathbb{E}[\bA]_{ij} \leq \rho_n\nu_nn$. Similarly, $\underset{i \in \tilde{\cO}}{\sum} \bPi_{ij}\mathbb{E}[\bA]_{ij} \leq \gamma_n\tilde{\nu}_ns$. This implies that $\underset{i \in [n]}{\sum} \bPi_{ij}\mathbb{E}[\bA]_{ij} \leq \rho_n\nu_nn + \gamma_n \tilde{\nu}_ns$. Using assumption \ref{ass:out_con}, we find that $\underset{i \in [n]}{\sum} \bPi_{ij}\mathbb{E}[\bA]_{ij} \leq 2\rho_n\nu_nn$, and obtain a contradiction.


\subsubsection{Proof of Lemma \ref{lem:bound_lambda}}
\label{subsubsect:proofLem_bound_lam}
Note that $\bOmega\odot\bSigma_{\vert I}$ is a symmetric random matrix with independent centered entries. Moreover, for $(i,j) \in \cI \times \cI$,  $\left(\bOmega\odot\bSigma\right)_{ij} = b_{ij}\xi_{ij}$, where we define $b_{ij} \triangleq \bPi_{ij}\bL^*_{ij}\left(1 - \bL^*_{ij}\right)$ and $\xi_{ij} = \frac{\bOmega_{ij}\bSigma_{ij}}{b_{ij}}$. With these notations, we see that $\underset{ij}{\max}\mathbb{E}\left[\left(\xi_{ij}b_{ij} \right)^{2 \alpha  }\right]^{\frac{1}{2 \alpha  }} \leq 1$ and that $\underset{i}{\max}\sqrt{\sum_j b_{ij}^2} \leq \nu_n\rho_n n$. Applying Proposition \ref{thm:Bandeira} for $ t = \sqrt{2\nu_n \rho_n n}$ and $\alpha = 3$, we find that
\begin{equation*}
    \mathbb{P}\left(\left \Vert\left(\bOmega\odot\bSigma\right)_{\vert I}\right \Vert_{op} \geq \sqrt{2}e^{\frac{2}{3}}\left(2 \sqrt{\nu_n \rho_n n} + 42\sqrt{\log(n) }\right) + \sqrt{2\nu_n \rho_n n} \right)\leq e^{-\nu_n \rho_n n}.
\end{equation*}
We conclude the proof of Lemma \ref{lem:bound_lambda} by recalling that $\log(n) \leq \nu_n\rho_n n$.



\subsubsection{Proof of Lemma \ref{lem:S-decrease}}
\label{proof:lem:S-decrease}
First, using the $2$-smoothness of the least-squares data fitting term and the $\epsilon$-smoothness of the ridge regularization, we obtain that:
\begin{eqnarray}
\label{eq:smoothness}
\mathcal{F}(\Sprm^{(t)}, \bL^{(t-1)}, R^{(t-1)})&\leq& \mathcal{F}(\Sprm^{(t-1)}, \bL^{(t-1)}, R^{(t-1)}) + \pscal{\bG_S(\bS^{(t-1)},\bL^{(t-1)})}{\Sprm^{(t)}-\Sprm^{(t-1)}} \nonumber\\
&&+ \frac{2+\epsilon}{2}\norm{\Sprm^{(t)}-\Sprm^{(t-1)}}[F]^2 + \lambda_2(\norm{\Sprm^{(t)}}[2,1] - \norm{\Sprm^{(t-1)}}[2,1]).
\end{eqnarray}
Then, by definition of the proximal operator, we have that:
\begin{equation}
\label{eq:prox}
\begin{array}{rl}
\Sprm^{(t)} & \in \argmin\left(\eta\lambda_2\norm{\Sprm}[2,1] + \frac{1}{2}\norm{\Sprm-\Sprm^{(t-1)}-\eta\bG_S(\bS^{(t-1)},\bL^{(t-1)})}[F]^2 \right)\\
& \in \argmin\Big(\pscal{\bG_S(\bS^{(t-1)},\bL^{(t-1)})}{\Sprm-\Sprm^{(t-1)}} + \frac{1}{2\eta}\norm{\Sprm-\Sprm^{(t-1)}}[F]^2\\
& \quad\quad\quad\quad\quad\quad\quad\quad\quad\quad\quad\quad\quad\quad\quad + \lambda_2(\norm{\Sprm}[2,1] - \norm{\Sprm^{(t-1)}}[2,1]) \Big).
\end{array}
\end{equation}
Combining \eqref{eq:smoothness}, \eqref{eq:prox} and the fact that $\eta\leq 1/(2+\epsilon)$, we obtain that, for any $\Sprm\in\mathbb{R}^{n\times n}$:
\begin{eqnarray*}
\label{eq:subres-lem1}
\mathcal{F}(\Sprm^{(t)}, \bL^{(t-1)}, R^{(t-1)})&\leq& \mathcal{F}(\Sprm^{(t-1)}, \bL^{(t-1)}, R^{(t-1)})+ \pscal{\bG_S(\bS^{(t-1)},\bL^{(t-1)})}{\Sprm-\Sprm^{(t-1)}} \nonumber\\
&&+ \frac{1}{2\eta}\norm{\Sprm-\Sprm^{(t-1)}}[F]^2 + \lambda_2(\norm{\Sprm}[2,1] - \norm{\Sprm^{(t-1)}}[2,1]).
\end{eqnarray*}
In particular, for matrices of the form $b\tilde{\bS}^{(t-1)}+(1-b)\bS^{(t-1)}$, $b\in\mathbb{R}$, we obtain:
\begin{eqnarray*}
\mathcal{F}(\Sprm^{(t)}, \bL^{(t-1)}, R^{(t-1)})&\leq& \mathcal{F}(\Sprm^{(t-1)}, \bL^{(t-1)}, R^{(t-1)}) + b\pscal{\bG_S(\bS^{(t-1)},\bL^{(t-1)})}{\tilde\Sprm^{(t-1)}-\Sprm^{(t-1)}} \nonumber\\
&&+ \frac{b^2}{2\eta}\norm{\tilde\Sprm^{(t-1)}-\Sprm^{(t-1)}}[F]^2 + \lambda_2(\norm{b\tilde\Sprm^{(t-1)}+(1-b)\Sprm^{(t-1)}}[2,1] - \norm{\Sprm^{(t-1)}}[2,1]),
\end{eqnarray*}
and, using the triangular inequality:
\begin{eqnarray}
\label{eq:subres-lem1-2}
\mathcal{F}(\Sprm^{(t)}, \bL^{(t-1)}, R^{(t-1)})&\leq& \mathcal{F}(\Sprm^{(t-1)}, \bL^{(t-1)}, R^{(t-1)})
 + b\pscal{\bG_S(\bS^{(t-1)},\bL^{(t-1)})}{\tilde\Sprm^{(t-1)}-\Sprm^{(t-1)}} \nonumber\\
&&+ \frac{b^2}{2\eta}\norm{\tilde\Sprm^{(t-1)}-\Sprm^{(t-1)}}[F]^2 + b\lambda_2(\norm{\tilde\Sprm^{(t-1)}}[2,1] - \norm{\Sprm^{(t-1)}}[2,1]).
\end{eqnarray}
Finally, minimizing the right hand side of \eqref{eq:subres-lem1-2} with respect to $b$, we obtain the final result:
\begin{equation*}
\mathcal{F}(\Sprm^{(t)}, \bL^{(t-1)}, R^{(t-1)})-\mathcal{F}(\Sprm^{(t-1)}, \bL^{(t-1)}, R^{(t-1)})\leq \frac{-\eta g_S(\Sprm^{(t-1)}, \bL^{(t-1)})^2}{(2Q^{(t-1)})^2},
\end{equation*}
where we have used that $\norm{\tilde\Sprm^{(t-1)}-\Sprm^{(t-1)}}[F]^2\leq (2Q^{(t-1)})^2$.

\subsubsection{Proof of Lemma \ref{lem:L-decrease}}
\label{proof:lem:L-decrease}
We first observe, using a Taylor expansion of the quadratic term of the objective function (the least-squares data fitting term plus the ridge regularization term), and \eqref{eq-Lupdate} that:
\begin{eqnarray*}
\label{eq:taylor}
\mathcal{F}(\Sprm^{(t)}, \bL^{(t)}, R^{(t)}) = \mathcal{F}(\Sprm^{(t)}, \bL^{(t-1)}, R^{(t-1)}) - \beta_tg_L(\Sprm^{(t)}, \bL^{(t-1)}, R^{(t-1)})
+ \frac{\beta_t^2(1+\epsilon)}{2}\norm{\tilde{\bL}^{(t)}-\bL^{(t-1)}}[F]^2.
\end{eqnarray*}
Now, recall that
\begin{equation*}
\label{eq:step-size-v2}
\beta_t = \min\left\{1, 
\frac{\pscal{\bL^{(t-1)}-\tilde\bL^{(t)}}{\bG_L(\bS^{(t)},\bL^{(t-1)})}+\lambda_1(R^{(t-1)}-\tilde R^{(t)})}{(1+\epsilon)\norm{\tilde{\bL}^{(t)}-\bL^{(t-1)}}[F]^2}\right\},
\end{equation*}
with $(\tilde{\bL}^{(t)}, \tilde R^{(t)})$ defined in \eqref{eq:Lprm}, and $g_L$ in \eqref{eq:gL}.
\paragraph{Case 1:} $\pscal{\bG_L(\bS^{(t)},\bL^{(t-1)})}{\bL^{(t-1)}-\tilde{\bL}^{(t)}} 
+ \lambda_1(R^{(t-1)} - \tilde{R}^{(t)})\geq (1+\epsilon)\norm{\tilde{\bL}^{(t)}-\bL^{(t-1)}}[F]^2$. Then, $\beta_t = 1$, and $g_L(\Sprm^{(t)}, \bL^{(t-1)}, R^{(t-1)})\geq (1+\epsilon)\norm{\tilde{\bL}^{(t)}-\bL^{(t-1)}}[F]^2$. As a result, we observe:
\begin{eqnarray}
\label{eq:L-decrease-1}
\mathcal{F}(\Sprm^{(t)}, \bL^{(t)}, R^{(t)})-\mathcal{F}(\Sprm^{(t)}, \bL^{(t-1)}, R^{(t-1)}) &\leq&  -\frac{1}{2}g_L(\Sprm^{(t)}, \bL^{(t-1)}, R^{(t-1)})\nonumber\\
&\leq & -\frac{1}{2}\frac{(g_L(\Sprm^{(t)}, \bL^{(t-1)}, R^{(t-1)}))^2}{g_L(\Sprm^{(t)}, \bL^{(t-1)}, R^{(t-1)})}\\\nonumber
&\leq& -\frac{1}{2}\frac{(g_L(\Sprm^{(t)}, \bL^{(t-1)}, R^{(t-1)}))^2}{\bar{R}^{(t)}(\lambda_1+2M^{(t)})},
\end{eqnarray}
where, to obtain the last inequality, we have used that $M^{(t)} = \norm{\bG_L(\bS^{(t)},\bL^{(t-1)})}[F] \geq \norm{\bG_L(\bS^{(t)},\bL^{(t-1)})}[op]$, and the inequalities $R^{(t-1)}-\tilde{R}^{(t)}\leq \bar{R}^{(t)}$ and
$$\pscal{\bG_L(\bS^{(t)},\bL^{(t-1)})}{\bL^{(t-1)}-\tilde{\bL}^{(t-1)}}\leq 2M^{(t)}\bar{R}^{(t)}.$$
\paragraph{Case 2:} $\pscal{\bG_L(\bS^{(t)},\bL^{(t-1)})}{\bL^{(t-1)}-\tilde{\bL}^{(t)}} 
+ \lambda_1(R^{(t-1)} - \tilde{R}^{(t)})< (1+\epsilon)\norm{\bL^{(t-1)}-\tilde\bL^{(t)}}[F]^2$. Then, $\beta_t = g_L(\Sprm^{(t)}, \bL^{(t-1)}, R^{(t-1)})/((1+\epsilon)\norm{\bL^{(t-1)}-\tilde\bL^{(t)}}[F]^2)$, and we obtain:
\begin{eqnarray*}
\label{eq:L-decrease-2}
\mathcal{F}(\Sprm^{(t)}, \bL^{(t)}, R^{(t)})-\mathcal{F}(\Sprm^{(t)}, \bL^{(t-1)}, R^{(t-1)})
 &\leq&  -\frac{1}{2}\frac{(g_L(\Sprm^{(t)}, \bL^{(t-1)}, R^{(t-1)}))^2}{(1+\epsilon)\norm{\bL^{(t-1)}-\tilde\bL^{(t)}}[F]^2}\nonumber\\
&\leq& -\frac{1}{2}\frac{(g_L(\Sprm^{(t)}, \bL^{(t-1)}, R^{(t-1)}))^2}{(1+\epsilon)(2\bar{R}^{(t)})^2},
\end{eqnarray*}
where, to obtain the last inequality, we used that $\norm{\bL^{(t-1)}-\tilde\bL^{(t)}}[F]^2\leq \norm{\bL^{(t-1)}-\tilde\bL^{(t)}}[*]^2\leq(2\bar{R}^{(t)})^2$. \newline

We finally prove \eqref{eq:L-decrease-v2} as follows. We start by noticing that $\norm{\tilde{\bL}^{(t-1)}-\bL^{(t-1)}}[F]^2 = \beta_t^2\norm{\bL^{(t)}-\bL^{(t-1)}}[F]^2$. If $\beta_t = 1$, then by definition of $\beta_t$: $$g_L(\Sprm^{(t)}, \bL^{(t-1)}, R^{(t-1)})\geq (1+\epsilon) \norm{\tilde{\bL}^{(t-1)}-\bL^{(t-1)}}[F]^2 = (1+\epsilon) \norm{\bL^{(t)}-\bL^{(t-1)}}[F]^2.$$ Inequality \eqref{eq:L-decrease-1} then implies that:
\begin{equation*}
\mathcal{F}(\Sprm^{(t)}, \bL^{(t)}, R^{(t)})-\mathcal{F}(\Sprm^{(t)}, \bL^{(t-1)}, R^{(t-1)})\leq -\frac{(1+\epsilon)}{2} \norm{\bL^{(t)}-\bL^{(t-1)}}[F]^2.
\end{equation*}
If $\beta_t = g_L(\Sprm^{(t)}, \bL^{(t-1)}, R^{(t-1)})/((1+\epsilon)\norm{\bL^{(t-1)}-\tilde\bL^{(t)}}[F]^2)$, then:
\begin{eqnarray*}
\norm{\tilde{\bL}^{(t-1)}-\bL^{(t-1)}}[F]^2 &=& \beta_t^2\norm{\bL^{(t)}-\bL^{(t-1)}}[F]^2 
= \frac{(g_L(\Sprm^{(t)}, \bL^{(t-1)}, R^{(t-1)}))^2}{(1+\epsilon) \norm{\tilde{\bL}^{(t-1)}-\bL^{(t-1)}}[F]^2}\nonumber\\
&\leq& \frac{2}{1+\epsilon}\left(\mathcal{F}(\Sprm^{(t)}, \bL^{(t-1)}, R^{(t-1)})- \mathcal{F}(\Sprm^{(t)}, \bL^{(t)}, R^{(t)})\right),
\end{eqnarray*}
which proves the result.
\subsubsection{Proof of Lemma \ref{lem:bound_Omega_A_out}}
\label{subsubsect:proofLemOmega}
To prove equation \eqref{eq:lem_noise_i} in Lemma \ref{lem:bound_Omega_A_out}, recall that for $j \in \cI$, $\underset{i \in \cI}{\sum}\mathbb{E}\left[\bOmega_{ij}\bSigma_{ij}^2 \right] \leq n\nu_n \rho_n$, that $\underset{i \in \cI}{\sum}\mathbb{V}ar\left[\bOmega_{ij}\bSigma_{ij}^2 \right] \leq n\nu_n \rho_n$, and that $\left \Vert \bOmega \odot \bSigma\odot \bSigma \right \Vert_{\infty} \leq 1$. Applying Bernstein's inequality \eqref{eq:Bernstein}, we obtain that for any $j \in \cI$ and $t>0$,
\begin{eqnarray*}
\mathbb{P} \left( \underset{i \in \cI}{\sum} \bOmega_{ij}\bSigma_{ij}^2 \geq \nu_n \rho_n n + \sqrt{2t\nu_n \rho_n n} + \frac{3}{2}t \right) & \leq & 2e^{-t}
\end{eqnarray*}
Choosing $t = 2\nu_n \rho_n n$, we find that
\begin{eqnarray*}
\mathbb{P} \left( \underset{j \in \cI}{\max}\sqrt{\underset{i \in \cI}{\sum} \bOmega_{ij}\bSigma_{ij}^2} \geq \sqrt{6\nu_n \rho_n n}\right) & \leq & 2ne^{ - 2\nu_n\rho_n n}\\
\mathbb{P} \left( \left \Vert \bOmega\odot\bSigma_{\vert I} \right \Vert_{2, \infty} \geq \sqrt{6\nu_n \rho_n n}\right) & \leq & 2e^{ - \nu_n \rho_n n}
\end{eqnarray*}
where we have used the union bound and $\nu_n\gamma_n n \geq \log(n)$. This proves equation \eqref{eq:lem_noise_i} in Lemma \ref{lem:bound_lambda}. 

 To prove equation \eqref{eq:lem_noise_ii} in Lemma \ref{lem:bound_Omega_A_out}, note that $\left \Vert \bOmega_{\vert I} \right \Vert_{2, \infty} \leq \left \Vert \bPi_{\vert I} - \bOmega_{\vert I} \right \Vert_{2, \infty} + \left \Vert \bPi_{\vert I} \right \Vert_{2, \infty}$ and $\left \Vert \bPi_{\vert I} \right \Vert_{2, \infty} \leq \sqrt{\nu_n n}$. Moreover, for $j \in \cI$, $\underset{i \in \cI}{\sum}\mathbb{E}\left[\left( \bPi_{ij} - \bOmega_{ij}\right)^2 \right] \leq \nu_n n$, $\underset{i \in \cI}{\sum}\mathbb{V}ar\left[\left( \bPi - \bOmega_{ij}\right)^2\right] \leq \nu_nn$, and $\left \Vert \bPi_{\vert I} - \bOmega_{\vert I} \right \Vert_{\infty} \leq 1$. We apply Bernstein's inequality and find that for any $j \in \cI$ and $t>0$,
\begin{eqnarray*}
\mathbb{P} \left( \underset{i \in \cI}{\sum} \left( \bPi_{ij} - \bOmega_{ij}\right)^2 \geq \nu_n n + \sqrt{2t\nu_n n} + \frac{3}{2}t \right) & \leq & 2e^{-t}
\end{eqnarray*}
Choosing $t = 2\nu_n n$ and using an union bound, we find that 
\begin{eqnarray*}
\mathbb{P} \left( \underset{j \in \cI}{\sup} \sqrt{\underset{i \in \cI}{\sum} \left( \bPi_{ij} - \bOmega_{ij}\right)^2} \geq \sqrt{6\nu_n n}\right) & \leq & 2ne^{ - 2n \nu_n }\\
\mathbb{P} \left( \left \Vert \bPi_{\vert I} - \bOmega_{\vert I} \right \Vert_{2, \infty} \geq \sqrt{6\nu_n n}\right) \leq 2e^{ - \nu_n n}
\end{eqnarray*}
where we have used that $\nu_n n \geq \log(n)$. This proves equation \eqref{eq:lem_noise_ii}.

To prove equation \eqref{eq:lem_noise_iii}, recall that for $(i,j) \in \cO \times \cI$, $\bOmega_{ij}\bA_{ij} \sim$Bernoulli$(\bPi_{ij}\bS^*_{ij})$, and that $\left \Vert \bPi \odot \bS^{\top} \right \Vert_{\infty} \leq \nu_n\gamma_n$. Then, applying Bernstein's inequality \eqref{eq:Bernstein}, we find that for any $j \in \cI$ and any $t>0$,
\begin{eqnarray*}
  \mathbb{P} \left( \underset{i \in \cO}{\sum}\bOmega_{ij}\bA_{ij} \geq s\nu_n\gamma_n + \sqrt{2ts\nu_n\gamma_n} + \frac{3t}{2} \right) \leq 2e^{-t}.
\end{eqnarray*}
Choosing $t = 2 \nu_n\rho_n n$, we find that
\begin{eqnarray*}
  \mathbb{P} \left( \underset{i \in \cO}{\sum}\bOmega_{ij}\bA_{ij} \geq s\nu_n\gamma_n + 2\sqrt{\gamma_n\rho_n n s}\nu_n + 3\nu_n\rho_n n \right) \leq 2e^{-t}.
\end{eqnarray*}

Under Assumption \ref{ass:out_con}, this implies
\begin{eqnarray*}
  \mathbb{P} \left( \underset{i \in \cO}{\sum}\bOmega_{ij}\bA_{ij} \geq 6\nu_n\rho_n n \right) \leq 2e^{-2\nu_n \rho_n n}.
\end{eqnarray*}
Using the union bound, and the bound $\nu_n \rho_n n \geq \log(n)$, we conclude that
\begin{eqnarray*}
  \mathbb{P} \left( \underset{j \in \cI}{\max}\sqrt{\underset{i \in \cO}{\sum}\bOmega_{ij}\bA_{ij}} \geq \sqrt{6\nu_n \rho_n n n} \right) \leq 2ne^{-2\nu_n \rho_n n} \leq 2e^{-\nu_n \rho_n n}.
\end{eqnarray*}
This concludes the proof of Lemma \ref{lem:bound_Omega_A_out}.

\subsubsection{Proof of Lemma \ref{lem:lower_bound_outlier}}
\label{subsubsect:proof_LBO}
Recall that for $j \in \cO$, $\left \{ \left( \left(\bOmega_{ij}\bA_{ij} - \rho_n \right)_+\right)^2 \right\}_{i \in \cI}$  are independent random variables. Moreover, easy calculations yields that $\mathbb{E}\left[\left(\bOmega_{ij}\left(\bA_{ij} - \rho_n \right)_+\right)^2\right] = \bPi_{ij}\bS^*_{ij}(1- \rho_n)^2$, and that $\mathbb{V}ar\left[\left(\bOmega_{ij}\left(\bA_{ij} - \rho_n \right)_+\right)^2\right]\leq \bPi_{ij}\bS^*_{ij}(1-\rho_n)^2$. Applying Bernstein's inequality \eqref{eq:Bernstein}, we see that for any $t>0$,
\begin{eqnarray*}
  \mathbb{P}\left( \left \vert \underset{i \in \cI}{\sum} \mathbb{E}\left[ \left(\bOmega_{ij} \left(\bA_{ij} - \rho_n \right)_+ \right)^2\right] - \underset{i \in \cI}{\sum} \left(\bOmega_{ij} \left(\bA_{ij} - \rho_n \right)_+ \right)^2 \right \vert \geq \sqrt{2t\underset{i \in \cI}{\sum}\bPi_{ij}\bS^*_{ij}(1-\rho_n)^2} + \frac{3t}{2}\right) \leq 2e^{-t}.
\end{eqnarray*}
Choosing $t =\frac{1}{80} \underset{i \in \cI}{\sum}\bPi_{ij}\bS^*_{ij}(1-\rho_n)^2$, we find that
\begin{eqnarray*}
  \mathbb{P}\left( \underset{i \in \cI}{\sum} \left(\bOmega_{ij} \left(\bA_{ij} - \rho_n \right) \right)_+^2\leq \underset{i \in \cI}{\sum}\bPi_{ij}\bS^*_{ij}(1-\rho_n)^2 - \frac{1}{2}\underset{i \in \cI}{\sum}\bPi_{ij}\bS^*_{ij}(1-\rho_n)^2\right) \leq 2e^{-\frac{1}{80} \underset{i \in \cI}{\sum}\bPi_{ij}\bS^*_{ij}(1-\rho_n)^2}.\nonumber
  \end{eqnarray*}
  When $\underset{j \in \cO}{\min} \underset{i \in \cI}{\sum}\bPi_{ij}\bS^*_{ij}(1-\rho_n)^2 \geq \nu_n \rho_n n$ and $\rho_n \leq \frac{1}{2}$, this implies that
\begin{eqnarray*}
  \mathbb{P}\left(\underset{j \in \cO}{\min} \sqrt{\underset{i \in \cI}{\sum} \left(\bOmega_{ij} \left(\bA_{ij} - \rho_n \right)_+ \right)^2}\leq \frac{1}{4}\underset{j \in \cO}{\min} \sqrt{\underset{i \in \cI}{\sum} \bPi_{ij}\bS^*_{ij}} \right) \leq 2se^{- \frac{-\nu_n \rho_n n}{80}}\nonumber.
\end{eqnarray*}

\subsubsection{Proof of Lemma \ref{lem:erreur_L}}
\label{subsubsect:prooferreurL}
Let $\partial\left \Vert \cdot \right \Vert_*$ and $\partial\left \Vert \cdot \right \Vert_{2,1}$ denote respectively the sub-differentials of $\left \Vert \cdot \right \Vert_*$ and $\left \Vert \cdot \right \Vert_{2,1}$ norms. Recall that $\left( \widehat{\bS}, \widehat{\bL}\right)$ minimizes $\mathcal{F}$. The standard optimality condition over a convex set states that for any admissible matrix $\left(\bS, \bL \right)$, there exists $\widehat{\bV} \in \partial\left \Vert \widehat{\bS} \right \Vert_{2,1}$ and $\widehat{\bW} \in \partial\left \Vert \widehat{\bL} \right \Vert_*$ such that 

\begin{eqnarray}
- \left \langle \bOmega\odot \left(\bA - \widehat{\bS}- \widehat{\bS}^{\top} - \widehat{\bL}\right) \Big \vert \bS - \widehat{\bS} + \bS^{\top} - \widehat{\bS}^{\top}  +  \bL -  \widehat{\bL}\right \rangle&&  \nonumber\\
+ \lambda_1 \left \langle \widehat{\bW} \big \vert \bL - \widehat{\bL}\right \rangle   + \lambda_2 \left \langle \widehat{\bV} \big \vert \bS - \widehat{\bS} \right \rangle &\geq 0& \label{eq:subdiff}
\end{eqnarray}
Applying equation \eqref{eq:subdiff} for the admissible matrices $\left(\widehat{\bS}, \bL^* \right)$, we find that there exists $\widehat{\bW} \in \partial\left \Vert \widehat{\bL} \right \Vert_*$ such that
\begin{eqnarray}
&& - \left \langle \bOmega\odot \left(\bA - \widehat{\bS}- \widehat{\bS}^{\top} - \widehat{\bL}\right) \Big \vert \Delta\bL\right \rangle + \lambda_1  \left \langle \widehat{\bW} \big \vert \Delta\bL \right \rangle\geq 0 \label{eq:grad_L-1}.
\end{eqnarray}
Recall that $\bSigma_{\vert I} \triangleq \bA_{\vert I} + \text{diag}(\bL^*) - \bL^*$, that $\Delta \bL \triangleq \bL^* - \widehat{\bL}$, and that $\bOmega\odot\text{diag}(\bM) = 0$ for any matrix $\bM$. Thus, equation \eqref{eq:grad_L-1} becomes 
\begin{eqnarray}
&& - \left \langle \bOmega\odot \left(\bSigma_{I} + \Delta \bL + \bA_{\vert O}  - \widehat{\bS}- \widehat{\bS}^{\top} \right) \Big \vert \Delta\bL\right \rangle + \lambda_1  \left \langle \widehat{\bW} \big \vert \Delta\bL \right \rangle\geq 0 \label{eq:grad_L}.
\end{eqnarray}
 Developing equation \eqref{eq:grad_L}, we find that

\begin{eqnarray}\label{eq:dev_L_I_O}
 - \left \langle \bOmega\odot\bSigma_{\vert I} \Big \vert \Delta \bL \right \rangle  - \left \langle \bOmega\odot\Delta \bL \Big \vert \Delta \bL \right \rangle - \left \langle \bOmega\odot\left(\bA - \widehat{\bS} - \widehat{\bS}^{\top}\right)_{\vert O} \Big \vert \Delta \bL \right \rangle &&\nonumber \\
 + \left \langle \bOmega\odot\left( \widehat{\bS} + \widehat{\bS}^{\top}\right)_{\vert I} \Big \vert \Delta \bL \right \rangle   + \lambda_1  \left \langle \widehat{\bW} \big \vert \Delta\bL \right \rangle &\geq & 0. \nonumber
\end{eqnarray}
We have proved in Theorem \ref{thm:outliers_detection} that $\widehat{\bS}_{\vert I} = \widehat{\bS}_{\vert I}^{\top} = \bold{0}$ with probability at least $1 - 6e^{-\nu_n \rho_n n}$ . Therefore, when equation \eqref{eq:detectedOutliers} holds,
\begin{eqnarray}
\left \Vert \bOmega\odot\Delta\bL\right \Vert_F^2 &\leq& \left \vert \left \langle \bOmega\odot\bSigma_{\vert I} \Big \vert \Delta \bL \right \rangle\right \vert + \left \vert \left \langle \bOmega\odot\left(\bA  -\widehat{\bS} - \widehat{\bS}^{\top}\right)_{\vert O} \Big \vert \Delta \bL \right \rangle\right \vert + \lambda_1  \left \langle \widehat{\bW} \big \vert \Delta\bL \right \rangle.\nonumber
\end{eqnarray}
Using the duality of the $\Vert \cdot \Vert_{*}$-norm and the $\Vert \cdot \Vert_{op}$-norm, we find that
\begin{eqnarray}
\left \Vert \bOmega\odot\Delta\bL\right \Vert_F^2 &\leq& \left \Vert  \bOmega\odot\bSigma_{\vert I} \right \Vert_{op} \left \Vert \Delta \bL \right \Vert_{*} + \left \vert \left \langle \bOmega\odot\left(\bA - \widehat{\bS} - \widehat{\bS}^{\top}\right)_{\vert O} \Big \vert \Delta \bL \right \rangle\right \vert + \lambda_1  \left \langle \widehat{\bW} \big \vert \Delta\bL \right \rangle.\nonumber
\end{eqnarray}
Next, we bound the term $\left \vert \left \langle \bOmega\odot\left(\bA- \widehat{\bS} - \widehat{\bS}^{\top}\right)_{\vert O} \Big \vert \Delta \bL \right \rangle\right \vert$ using the following Lemma.

\begin{lem}\label{lem:interplay_S_L}
With probability at least $1 - 2e^{-\tilde{\nu}_n\gamma_nsn}$,
$$\left \vert \left \langle \bOmega\odot\left(\bA - \widehat{\bS} - \widehat{\bS}^{\top}\right)_{\vert O} \Big \vert \Delta \bL \right \rangle\right \vert \leq 16 \tilde{\nu}_n \gamma_n \rho_n ns.$$
\end{lem}
\begin{proof}
See Section \ref{subsubsect:proofInterplayS}.
\end{proof}
Finally, we bound $\left \langle \widehat{\bW} \big \vert \Delta\bL \right \rangle$. Note that by definition of the subgradient, $\left \langle \widehat{\bW} \big \vert \bL^* - \widehat{\bL} \right \rangle \leq \left \Vert \bL^* \right \Vert_* - \left\Vert \widehat{\bL}\right \Vert_*$. Using the separability of the spectral norm on orthogonal subspaces and the identity $\cP_{\bL^*}\left(  \bL^*\right) = \bL^*$, we find that
\begin{eqnarray*}
\left \Vert \widehat{\bL} \right \Vert_{*} &=& \left \Vert \cP_{\bL^*}^{\perp}\left(  \Delta\bL\right)+ \cP_{\bL^*}\left(  \Delta\bL\right) - \bL^*\right \Vert_{*}\\
&=& \left \Vert \cP_{\bL^*}^{\perp}\left(  \Delta\bL\right) \right \Vert_{*} + \left \Vert\cP_{\bL^*}\left(  \Delta\bL\right) - \bL^*\right \Vert_{*}\\
&\geq& \left \Vert \cP_{\bL^*}^{\perp}\left(  \Delta\bL\right) \right \Vert_{*} + \left \Vert \bL^*\right \Vert_{*} - \left \Vert\cP_{\bL^*}\left( \Delta\bL\right) \right\Vert_*.
\end{eqnarray*}
Combining this result with Lemma \ref{lem:interplay_S_L}, we find that with probability at least $1 - 6 e^{-\nu_n \rho_n n} - 2e^{-\tilde{\nu}_n\gamma_nsn}$,
\begin{eqnarray}
\left \Vert \bOmega\odot\Delta \bL \right \Vert_F^2 
&\leq& \left \Vert \bOmega\odot\bSigma_{\vert I}\right \Vert_{op} \left(\left \Vert\cP_{\bL^*}\left( \Delta\bL\right) \right\Vert_* + \left \Vert \cP_{\bL^*}^{\perp}\left( \Delta \bL\right) \right \Vert_{*} \right) + 16 \tilde{\nu}_n\gamma_n\rho_nsn + \lambda_1 \left(\left \Vert\cP_{\bL^*}\left( \Delta\bL\right) \right\Vert_* - \left \Vert \cP_{\bL^*}^{\perp}\left( \Delta \bL\right) \right \Vert_{*} \right) \nonumber.
\end{eqnarray}
Recall that by definition, $\Psi \geq 16\tilde{\nu}_n \gamma_n \rho_n n s$. Thus, when $\lambda_1 \geq 3\left \Vert \bOmega\odot\bSigma_{\vert I}\right \Vert_{op}$,
\begin{eqnarray}
\left \Vert \bOmega\odot\Delta \bL \right \Vert_F^2 
&\leq& \frac{\lambda_1}{3}  \left(5\left \Vert\cP_{\bL^*}\left( \Delta \bL\right) \right\Vert_* - \left \Vert \cP_{\bL^*}^{\perp}\left( \Delta \bL\right) \right \Vert_{*} \right) + \Psi \nonumber.
\end{eqnarray}
This proves equation \eqref{eq:erreur_l_i} in Lemma \ref{lem:erreur_L}. This result also implies that
\begin{eqnarray}
\left \Vert \cP_{\bL^*}^{\perp}\left( \Delta \bL\right) \right \Vert_{*}&\leq& 5\left \Vert\cP_{\bL^*}\left( \Delta \bL\right) \right\Vert_* + \frac{3\Psi}{\lambda_1}. \nonumber
\end{eqnarray}
Recall that $\bL^*$ is of rank $k$ and so $\cP_{\bL^*}\left( \Delta \bL\right)$ is of rank at most k. Therefore, 
\begin{eqnarray}
\left \Vert \Delta \bL \right \Vert_{*}&\leq& 6\left \Vert\cP_{\bL^*}\left( \Delta \bL\right) \right\Vert_* + \frac{3\Psi}{\lambda_1}\leq 6\sqrt{k}\left \Vert\cP_{\bL^*}\left( \Delta \bL\right) \right\Vert_F + \frac{3\Psi}{\lambda_1} \nonumber\\
&\leq& 6\sqrt{k}\left \Vert \Delta \bL\right\Vert_F + \frac{3\Psi}{\lambda_1} \leq 6\sqrt{k}\left \Vert \Delta \bL_{\vert I}\right\Vert_F + 6\sqrt{k(s n + s^2)}\rho_n + \frac{3\Psi}{\lambda_1} \nonumber\\
&\leq& 6\sqrt{k}\left \Vert \Delta \bL_{\vert I}\right\Vert_F + 6\sqrt{3ksn}\rho_n + \frac{3\Psi}{\lambda_1}\nonumber.
\end{eqnarray}
where we have used that $\left \Vert \Delta \bL_{\vert O}\right\Vert_{F} \leq \sqrt{\left \vert O \right \vert}\left \Vert \Delta \bL_{\vert O}\right\Vert_{\infty} \leq \sqrt{s^2 + 2sn}\rho_n$. This completes the proof of Lemma \ref{lem:erreur_L}.


\subsubsection{Proof of Lemma \ref{lem:error_X_L}}
\label{subsubsect:prooferrorX}
For ease of notations, let $\alpha = 36^2\frac{\nu_n\rho_n^2kn}{\mu_n} $. To prove Lemma \ref{lem:erreur_L}, we consider the following two cases.

\textbf{Case 1}: $\left\Vert \Delta\bL_{\vert I}\right \Vert_{L_2(\bPi)}^2 \leq \alpha$. Then the result is immediate.

\textbf{Case 2}: $\left\Vert \Delta\bL_{\vert I}\right \Vert_{L_2(\bPi)}^2 > \alpha$. Let $r>0$ a constant to be specified later. We consider the following sets
$$\cS^r = \left\{\bM \in \mathbb{R}^{n \times n}_{sym}: \left \Vert \bM \right\Vert_{\infty} \leq \rho_n, \left \Vert \bM_{\vert I} \right\Vert_{L_2(\bPi)}^2 \geq \alpha, \left\Vert \bM \right\Vert_{*} \leq \sqrt{r}\left\Vert \bM_{\vert I} \right\Vert_F + \sqrt{3rsn}\rho_n + \frac{3\Psi}{\lambda_1}\right\}.$$
Recall that the random noise matrix $\bGamma$ is defined as follows: for any $(i,j) \in [n] \times [n]$, $i<j$, $\bGamma_{ij} = \bGamma_{ji}= \bOmega_{ij}\epsilon_{ij}$ where $(\epsilon_{ij})_{1 \leq i<j \leq n}$ is a Rademacher sequence. Now, we define $\beta_r$ as follows :
$$ \beta_r \triangleq \mathbb{E}\left[ \left\Vert\bGamma_{\vert I} \right \Vert_{op} \right]\left(\frac{64 r\rho_n^2}{\mu_n}\mathbb{E}\left[ \left\Vert\bGamma_{\vert I} \right \Vert_{op} \right] +15 \sqrt{srn}\rho_n^2+ \frac{32\Psi\rho_n}{\lambda_1} \right) .$$

\begin{lem}\label{lem:peeling_L}
With probability larger than $1 - e^{-\nu_n\rho_n n }$, simultaneously for any $\bM \in \cS^r$,
$$\frac{1}{2}\left \Vert\bM \right \Vert_{L_2(\bPi)}^2 \leq \left \Vert\bOmega\odot\bM_{\vert I}\right \Vert_F^2 +  \beta_r $$
\end{lem}
\begin{proof}
See Section \ref{subsubsect:proofPeeling}.
\end{proof}
Recall that $\beta$ was defined in equation \eqref{eq:def_beta}, and note that $\beta = \beta_{36k}$. Then, equation \eqref{eq:erreur_l_ii} in Lemma \ref{lem:erreur_L} implies that $\Delta \bL \in \cS^{36k}$ with probability at least $1 - 6e^{-\nu_n \rho_n n} - 2e^{-\tilde{\nu}_n \gamma_n s n}$. Combining equation \eqref{eq:erreur_l_i} in Lemma \ref{lem:erreur_L} and Lemma \ref{lem:peeling_L}, we find that with probability at least $1 - 7e^{-n \nu_n\rho_n} - 2e^{-\tilde{\nu}_n \gamma_n s n}$,
\begin{eqnarray*}
\frac{1}{2}\left \Vert\Delta \bL_{\vert I} \right \Vert_{L_2(\bPi)}^2 &\leq&  \frac{5\lambda_1}{3} \left \Vert\cP_{\bL^*}\left( \Delta \bL\right)\right \Vert_{*} + \Psi +  \beta.
\end{eqnarray*}
The matrix $\bL^*$ is of rank at most $k$. Therefore,
\begin{eqnarray*}
\left \Vert\Delta \bL_{\vert I} \right \Vert_{L_2(\bPi)}^2 &\leq&  \frac{10\lambda_1\sqrt{k}}{3} \left \Vert \Delta \bL\right \Vert_F + 2\Psi +  2\beta \leq \frac{50\lambda_1^2k}{9\mu_n} + \frac{\mu_n}{2}\left \Vert \Delta \bL\right \Vert_F^2 + \Psi +  \beta\\
&\leq &\frac{50\lambda_1^2k}{9\mu_n} + \frac{\mu_n}{2}\left \Vert \Delta \bL_{\vert I}\right \Vert_F^2 + \frac{3}{2}\mu_n \rho_n^2 s n + \Psi + \beta
\end{eqnarray*}
where we have used that $\left \Vert \Delta \bL_{\vert O}\right \Vert_{F}^2 \leq 3\rho_n^2ns$. Using equation \eqref{eq:eqnorms}, we find that
\begin{eqnarray*}
\left \Vert\Delta \bL_{\vert I} \right \Vert_{L_2(\bPi)}^2&\leq& \frac{1}{2}\left \Vert \Delta \bL_{\vert I}\right \Vert_{L_2(\bPi)}^2 + \frac{\mu_n}{2} \rho_n^2 n + \frac{3}{2}\mu_n \rho_n^2 s n+ \frac{50\lambda_1^2k}{9\mu_n}+ \Psi +  \beta.
\end{eqnarray*}
Thus 
\begin{eqnarray*}
\left \Vert\Delta \bL_{\vert I} \right \Vert_{L_2(\bPi)}^2 &\leq& 8\mu_n \rho_n^2 s n + \frac{100\lambda_1^2k}{9\mu_n}+ 2\Psi +  2\beta.
\end{eqnarray*}
We conclude the proof of Lemma \ref{lem:error_X_L} by recalling that $\mu_n \leq \nu_n$.

\subsubsection{Proof of Lemma \ref{lem:bound_Gamma_op}}
\label{subsubsect:proofGammaOp}
To prove Lemma \ref{lem:bound_Gamma_op}, we use Proposition \ref{thm:OperatorNorm}. For $(i,j) \in I$, set $b_{ij} = \sqrt{\bPi_{ij}}$, and $\xi_{ij} = \frac{\epsilon_{ij}\bOmega_{ij}}{b_{ij}}$, and for $i \in \cI$ set $b_{ii} = 0$. Note that for any $(i,j) \in \cI$, $\bGamma_{ij} = b_{ij}\xi_{ij}$, and that $\{\xi_{ij}\}_{i\leq j}$ is a sequence of independent symmetric random variables with unit variance. Moreover, for any $(i,j) \in  I$, $\left \vert b_{ij}\xi_{ij} \right \vert \leq 1$, so for any $\alpha \geq 3$, $\left(\mathbb{E}\left[\left(\xi_{ij}b_{ij} \right)^{2 \alpha  }\right]\right)^{\frac{1}{2 \alpha  }}\leq 1$. Finally, note that for any $i \in \cI$, $$\sqrt{\sum_{j\in \cI} b_{ij}^2} = \sqrt{\sum_{j\in \cI} \bPi_{ij}}\leq \sqrt{\nu_n n}.$$ 
 Applying Proposition \ref{thm:OperatorNorm}, we find that
\begin{equation*}
    \mathbb{E} \left[ \left \Vert \bGamma_{\vert I} \right \Vert_{op} \right] \leq e^{\frac{2}{3}}\left(\sqrt{\nu_n n} + 42\sqrt{\log(n) }\right)
\end{equation*}
We conclude this proof by recalling that $\nu_n n \geq \log(n)$.

\subsubsection{Proof of Lemma \ref{lem:interplay_S_L}}
\label{subsubsect:proofInterplayS}
To prove Lemma \ref{lem:interplay_S_L}, note that $\left \Vert \Delta \bL \right \Vert_{\infty} \leq \rho_n$, and therefore
\begin{eqnarray}
\left \vert \left \langle \bOmega\odot\left(\bA - \widehat{\bS} - \widehat{\bS}^{\top}\right)_{\vert O} \Big \vert \Delta \bL \right \rangle\right \vert &\leq& 2\rho_n\underset{(i,j) \in O}{\sum} \left \vert \bOmega_{ij}\left(\bA_{ij} - \widehat{\bS}_{ij} - \widehat{\bS}_{ji} \right)\right \vert \label{eq:presque_bernstein}.
\end{eqnarray}
Recall that $\widehat{\bL}$ and $\widehat{\bS}$ have non-negative entries, and that $\widehat{\bL}$ and $\bA$ are symmetric. Therefore, equation \eqref{eq:def_S} implies that $\left\{\widehat{\bS}_{ij} = 0 \text{ or } \widehat{\bS}_{ji} = 0\right\}  \Rightarrow \bA_{ij} = 0$, and that $\widehat{\bS}_{ij} + \widehat{\bS}_{ji} \leq \bA_{ij}$. Thus, equation \eqref{eq:presque_bernstein} implies
\begin{eqnarray}
\left \vert \left \langle \bOmega\odot\left(\bA - \widehat{\bS} - \widehat{\bS}^{\top}\right)_{\vert O} \Big \vert \Delta \bL \right \rangle\right \vert &\leq& 2\rho_n\underset{(i,j) \in O}{\sum}  \bOmega_{ij}\bA_{ij}.\label{eq:pre_bernstein}
\end{eqnarray}
To conclude the proof of Lemma \ref{lem:interplay_S_L}, we first prove the following result:
\begin{eqnarray}
\mathbb{P}\left(\underset{(i,j) \in O}{\sum} \bOmega_{ij} \bA_{ij}  \geq 8\tilde{\nu}_n\gamma_n sn \right)&\leq& \exp(-\tilde{\nu}_n\gamma_n sn) \label{eq:Bernst_S_Outliers}.
\end{eqnarray}
We use Bernstein's inequality to obtain equation \eqref{eq:Bernst_S_Outliers}. Note that $\left\{\bOmega_{ij} \bA_{ij}\right\}_{(i,j) \in O, i<j}$ is a sequence of independent Bernoulli random variables such that for any $i \in [n]$, $\underset{j \in \cO}{\sum}\mathbb{E}\left[\bOmega_{ij}\bA_{ij}\right] \leq  \tilde{\nu}_n \gamma_ns$, $\underset{j \in \cO}{\sum}\mathbb{V}ar\left[\bOmega_{ij}\bA_{ij}\right]\leq \tilde{\nu}_n\gamma_ns$, and 
$\left(\bOmega_{ij}\bA_{ij} - \mathbb{E}\left[\bOmega_{ij}\bA_{ij}\right]\right) \in \left[-1, 1\right]$. Hence, applying Bernstein's inequality \eqref{eq:Bernstein}, we find that for any $t>0$, 
\begin{eqnarray*}
\mathbb{P}\left(\underset{(i,j) \in O}{\sum} \bOmega_{ij}\bA_{ij} \geq 2\tilde{\nu}_n\gamma_n sn + \sqrt{2t\times\tilde{\nu}_n \gamma_nsn} + \frac{3t}{2}\right)&\leq& 2\exp(-t).
\end{eqnarray*}
Choosing $t= 2\tilde{\nu}_n\gamma_nsn$, we obtain equation  \eqref{eq:Bernst_S_Outliers}.
We conclude the proof of Lemma \ref{lem:interplay_S_L} by combining equations \eqref{eq:pre_bernstein} and \eqref{eq:Bernst_S_Outliers}.

\subsubsection{Proof of Lemma \ref{lem:peeling_L}}
\label{subsubsect:proofPeeling}
To prove Lemma \ref{lem:peeling_L}, we show that the probability of the following "bad" event is small :
$$\cB \triangleq \{ \exists \bM \in \cS^r\text{ such that } \left \vert \left\Vert \bOmega\odot \bM_{\vert I} \right \Vert_F^2 - \left\Vert  \bM_{\vert I}  \right \Vert_{L_2(\bPi)}^2 \right\vert \geq \frac{1}{2}\left\Vert  \bM_{\vert I}  \right \Vert_{L_2(\bPi)}^2 +  \beta_r \}.$$
We use a standard peeling argument to control the probability of the event $\cB$. 
For $T > \alpha$, define $$\cS(T) \triangleq \left\{\bM \in \cS^r: \left\Vert  \bM_{\vert I}  \right \Vert_{L_2(\bPi)}^2 \leq T\right\}, \ Z(T) = \underset{\bM \in \cS(T)}{\sup} \left \vert \left\Vert \bOmega\odot \bM_{\vert I} \right \Vert_F^2 - \left\Vert  \bM_{\vert I}  \right \Vert_{L_2(\bPi)}^2 \right\vert, \text{ and } $$ $$\cB(T) \triangleq \left\{ \exists \bM \in \cS(T): \left \vert \left\Vert \bOmega\odot \bM_{\vert I} \right \Vert_F^2 - \left\Vert  \bM_{\vert I}  \right \Vert_{L_2(\bPi)}^2 \right\vert \geq \frac{T}{4} +  \beta_r \right\} = \left\{Z(T) \geq \frac{T}{4} +  \beta\right\}.$$
For $l \geq 1$, define also
$\cS_l \triangleq \left\{\bM \in \cS^r: 2^{l-1}\alpha <\left\Vert  \bM_{\vert I}  \right \Vert_{L_2(\bPi)}^2 \leq 2^{l}\alpha \right\} \subset \cS\left(2^l\alpha\right)$ and 
\begin{eqnarray*}
  \cB_l &\triangleq&\left\{ \exists \bM \in \cS_l: \left \vert \left\Vert \bOmega\odot \bM_{\vert I}  \right \Vert_F^2 - \left\Vert  \bM_{\vert I}  \right \Vert_{L_2(\bPi)}^2 \right\vert \geq \frac{\left\Vert  \bM_{\vert I}  \right \Vert_{L_2(\bPi)}^2}{2} +  \beta_r \right\} \\
  &\subset& \left\{ \exists \bM \in \cS_l: \left \vert \left\Vert \bOmega\odot \bM_{\vert I} \right \Vert_F^2 - \left\Vert  \bM_{\vert I}  \right \Vert_{L_2(\bPi)}^2 \right\vert \geq \frac{2^{l-1}}{2}\alpha +  \beta_r \right\}  \subset \cB\left(2^l\alpha\right).
\end{eqnarray*}
Since $\cS^r \subset \underset{l\geq 1}{\cup} \cS_l$, it is easy to see that $\cB \subset \underset{l \geq 1}{\cup} \cB_l$. To control the probability of the events $\cB_l$, it is enough to control the probability of the events $\cB(T)$, which is done in the following lemma.
\begin{lem}\label{lem:PBT} For any $T \geq \alpha$, we have $\mathbb{P}\left(\cB(T) \right)\leq \exp(-\frac{T}{36^2\rho_n}).$
\end{lem}
\begin{proof}
See Section \ref{subsubsect:proofPBT}.
\end{proof}
We apply Lemma \ref{lem:PBT} to find
\begin{eqnarray*}
\mathbb{P}\left(\cB\right) &\leq& \underset{l \geq 1}{\sum}\mathbb{P}\left(\cB_l\right) \leq \underset{l \geq 1}{\sum}\exp\left(-\frac{2^l \alpha}{36^2\rho_n}\right)\\
&\leq& \underset{l \geq 1}{\sum}\exp\left(-\frac{2l \alpha}{36^2\rho_n}  \alpha\right) = \frac{\exp\left(-\frac{2\alpha}{36^2\rho_n}\right)}{1-\exp\left(-\frac{2\alpha}{36^2\rho_n}  \right)} = \frac{\exp\left(-2\frac{\nu_n\rho_nkn}{\mu_n}   \right)}{1-\exp\left(-2\frac{\nu_n\rho_nkn}{\mu_n}\right)}
\end{eqnarray*}
Note that $\frac{\nu_n\rho_nkn}{\mu_n} \geq \nu_n\rho_n n  \geq \log(n) \geq 1$, and so $\mathbb{P}\left[\cB\right] \leq \frac{1}{2} \exp\left(-2\nu_n \rho_n n  \right) \leq  \exp\left(-\nu_n\rho_n n \right).$ This concludes the proof of Lemma \ref{lem:peeling_L}.


\subsubsection{Proof of Lemma \ref{lem:PBT}}
\label{subsubsect:proofPBT}
Recall that $Z(T) = 2\underset{\bM \in \cS(T)}{\sup} \left \vert \underset{(i,j) \in I}{\sum} \bM_{ij}^2\left(\bOmega_{ij} - \bPi_{ij}\right) \right\vert$, since all matrices in $\cS$ are symmetric.  In order to bound $Z(T)$, we begin by controlling the deviation of $Z(T)$ from its expectation. To do this, we apply Bousquet's Theorem \ref{thm:Bousquet} to the random variable $Z(T) = 2\rho_n \underset{\bM \in \cS(T)}{\sup} \left \vert \underset{(i,j) \in I}{\sum} f^{\bM}_{ij}(\bOmega_{ij}) \right \vert$ where we set $f^{\bM}_{ij}(\bOmega_{ij}) \triangleq \frac{\left(\bOmega_{ij} - \bPi_{ij}\right)\bM_{ij}^2}{\rho_n}$. The set of functions $\left\{f^{\bM}_{ij}, \bM \in \mathcal{S}(T) \right\}$ is separable and we can apply Theorem \ref{thm:Bousquet} (see, e.g., \cite{GineNickl}, Section 2.1). Note that for any $(i,j) \in I$, $\mathbb{E}\left[f^{\bM}_{ij}(\bOmega_{ij})\right] = 0$, $\left\vert f^{\bM}_{ij}(\bOmega_{ij})\right\vert \leq 1$, $\mathbb{E} \left[\left(\bOmega_{ij} - \bPi_{ij}\right)^2 \right] \leq \bPi_{ij}$ and $\left \Vert\bM \right \Vert_{\infty} \leq \rho_n$ so 
\begin{eqnarray*}
v \triangleq 2\underset{\bM \in \mathcal{S}(T)}{\sup} \underset{(i,j) \in I}{\sum} \mathbb{E} \left[f^{\bM}_{ij}(X_{ij})^2\right] \leq 2\underset{(i,j) \in I}{\sum}\bPi_{ij}\frac{\bM_{ij}^4}{\rho_n^2} \leq 2\underset{\bM \in \mathcal{S}(T)}{\sup} \underset{(i,j) \in I}{\sum}\bPi_{ij}\bM_{ij}^2
\leq T.
\end{eqnarray*}
Theorem \ref{thm:Bousquet} implies that 
\begin{eqnarray*}
\mathbb{P} \left(\frac{Z_T}{2\rho_n} > \frac{\mathbb{E}[Z_T] }{2\rho_n}+ \frac{x}{3} + \sqrt{2x\left(\frac{2\mathbb{E}[Z_T]}{2\rho_n} + T\right)} \right) \leq \exp(-x) \\
\mathbb{P} \left(Z_T> \mathbb{E}[Z_T] + \frac{2\rho_n x}{3} + 2\rho_nx + 2\mathbb{E}[Z_T] + 2\rho_n\sqrt{2xT} \right) \leq \exp(-x)
\end{eqnarray*}
where we used $2\sqrt{ab} \leq a + b$. Setting $x = \frac{T}{36^2\rho_n}$ and noticing that $\rho_n \leq 1$ leads to 
\begin{eqnarray}\label{eq:concentration_L}
\mathbb{P} \left(Z_T> 2\mathbb{E}[Z_T] + \frac{T}{8} \right) \leq \exp(-\frac{T}{36^2\rho_n}).
\end{eqnarray}
In a second time, in order to bound $\mathbb{E}\left[Z_T\right]$, we apply a standard symetrization argument (see, e.g., \cite{koltchinskii2011oracle}, Theorem 2.1). We obtain that
\begin{equation}
  \mathbb{E}\left[Z(T) \right] \leq 4\mathbb{E}\left[\underset{\bM \in \cS(T)}{\sup}\left \vert \underset{(i,j) \in I}{\sum} \epsilon_{ij}\bM_{ij}^2\bOmega_{ij}\right \vert \right] \label{eq:sym}
\end{equation}
where $\left(\epsilon_{ij}\right)_{1\leq i<j\leq n}$ is a Rademacher sequence. For $i<j$, define $\phi_{ij}: x \rightarrow \frac{x^2}{2\rho_n}$. Recall that for any $(i,j)$, $\bOmega_{ij} \in \{0,1\}$, and so $\bOmega_{ij} =\bOmega_{ij}^2 $. With these notations, equation \eqref{eq:sym} becomes
\begin{equation*}
  \mathbb{E}\left[Z(T) \right] \leq 8\rho_n \mathbb{E}\left[\underset{\bM \in \cS(T)}{\sup}\left \vert \underset{i<j}{\sum} \epsilon_{ij} \phi_{ij}\left(\bOmega_{ij}\bM_{ij}\right)\right \vert \right].
\end{equation*}
\noindent We note that for $\bM \in \cS(T)$, $\left \Vert\bM \right \Vert_{\infty} \leq \rho_n$. Therefore, the functions $\phi_{ij}$ are 1-Lipschitz functions on $[-\rho_n, \rho_n]$ vanishing at $0$. We apply Talagrand's contraction principle (see, e.g., Theorem 2.2 in \cite{koltchinskii2011oracle}) and find that
\begin{eqnarray}
  \mathbb{E}\left[Z(T) \right] &\leq& 16\rho_n \mathbb{E}\left[\underset{\bM \in \cS(T)}{\sup}\left \vert \underset{(i,j) \in I}{\sum} \epsilon_{ij}\bM_{ij}\bOmega_{ij}\right \vert \right]\nonumber =  8\rho_n\mathbb{E}\left[\underset{\bM \in \cS(T)}{\sup}\left \vert \left \langle \bM \big \vert \bGamma_{\vert I} \right \rangle\right \vert \right]\nonumber
\end{eqnarray}
where for any $(i,j)$, $\bGamma_{ij} = \epsilon_{ij}\bOmega_{ij}$. By the duality of the $\left \Vert \cdot \right \Vert_{*}$-norm and $\left \Vert \cdot \right \Vert_{op}$-norm, and by definition of $\cS^r$, we find that
\begin{eqnarray}
  \mathbb{E}\left[Z(T) \right] &\leq  & 8\rho_n\underset{\bM \in \cS(T)}{\sup}\left \Vert  \bM \right \Vert_{*} \mathbb{E}\left[ \left\Vert\bGamma_{\vert I} \right \Vert_{op} \right]\nonumber\\
   &\leq  & 8\rho_n\left( \sqrt{r}\underset{\bM \in \cS(T)}{\sup}\left \Vert  \bM_{\vert I}\right \Vert_F + \sqrt{3rsn}\rho_n+ \frac{3\Psi}{\lambda_1}\right) \mathbb{E}\left[ \left\Vert\bGamma_{\vert I} \right \Vert_{op} \right]\nonumber.
\end{eqnarray}
Using equation \eqref{eq:eqnorms}, we find that
\begin{eqnarray}
  \mathbb{E}\left[Z(T) \right]&\leq& 8\rho_n\left( \sqrt{r}\left(\frac{1}{\sqrt{\mu_n}}\underset{\bM \in \cS(T)}{\sup}\left \Vert  \bM_{\vert I} \right \Vert_{L_2(\bPi)} + \sqrt{n}\rho_n \right) + \sqrt{3rsn}\rho_n + \frac{3\Psi}{\lambda_1}\right) \mathbb{E}\left[ \left\Vert\bGamma_{\vert I} \right \Vert_{op} \right] \nonumber\\
    &\leq  & \left(\frac{8\rho_n\sqrt{r}}{\sqrt{\mu_n}}\underset{\bM \in \cS(T)}{\sup}\left \Vert  \bM_{\vert I} \right \Vert_{L_2(\bPi)} + 8\sqrt{nr}\rho_n^2 + 8 \sqrt{3srn}\rho_n^2+  \frac{32\Psi\rho_n}{\lambda_1}\right)\mathbb{E}\left[ \left\Vert\bGamma_{\vert I} \right \Vert_{op} \right] \nonumber.
    \nonumber
\end{eqnarray}
Using the definition of $\cS(T)$, we find that
\begin{eqnarray}
     \mathbb{E}\left[Z(T) \right]
    &\leq& \left(\frac{8\rho_n\sqrt{rT}}{\sqrt{\mu_n}}+ 8\sqrt{r n}\rho_n^2 +8 \sqrt{3srn}\rho_n^2 + \frac{32\Psi\rho_n}{\lambda_1}\right) \mathbb{E}\left[ \left\Vert\bGamma_{\vert I} \right \Vert_{op} \right] \nonumber\\
    &\leq& \frac{T}{16} + \mathbb{E}\left[ \left\Vert\bGamma_{\vert I} \right \Vert_{op} \right]\left(\frac{64 r\rho_n^2}{\mu_n}\mathbb{E}\left[ \left\Vert\bGamma_{\vert I} \right \Vert_{op} \right]+ 15 \sqrt{srn}\rho_n^2+ \frac{32\Psi\rho_n}{\lambda_1} \right)\nonumber \\ &=& \frac{T}{16} +  \beta^r \label{eq:esp_L}.
\end{eqnarray}
Combining equation \eqref{eq:concentration_L} and equation \eqref{eq:esp_L} yields the desired result.

\subsection{Proof of Lemma \ref{lem:suite}}
\label{subsec:lem:suite:prf}

Consider the following chain of inequality:
$$\frac{1}{A_{k+1}} - \frac{1}{A_k} = \frac{A_k-A_{k+1}}{A_kA_{k+1}}\geq \gamma_k\frac{A_k}{A_{k+1}}\geq \gamma_k,$$
since $A_{k+1}\leq A_{k}$. Thus, we obtain
$$\frac{1}{A_{k+1}}-\frac{1}{A_1} = \sum_{i=1}^k\left(\frac{1}{A_{i+1}}-\frac{1}{A_i} \right)\geq \sum_{i=1}^k\gamma_i,$$
which gives the result after reshuffling the terms.
\end{document}